\documentclass{article}

 \usepackage[preprint]{neurips_2026}

\usepackage[utf8]{inputenc} 
\usepackage[T1]{fontenc}    
\usepackage{hyperref}       
\usepackage{url}            
\usepackage{booktabs}       
\usepackage{amsfonts}       
\usepackage{nicefrac}       
\usepackage{microtype}      
\usepackage{xcolor}         
\usepackage{colortbl}
\usepackage{tcolorbox}
\definecolor{bgcolor}{rgb}{0.8,1,1}
\definecolor{bgcolor2}{rgb}{0.8,1,0.8}
\definecolor{niceblue}{rgb}{0.0,0.19,0.56}

\usepackage[colorinlistoftodos,bordercolor=orange,backgroundcolor=orange!20,linecolor=orange,textsize=scriptsize]{todonotes}

\def\cX{\mathcal{X}}
\def\cY{\mathcal{Y}}

\def\cE{\mathcal{E}}
\def\P\text{I\kern-0.15em P}
\newcommand{\myparagraph}[1]{\noindent\textbf{#1.}}

\usepackage{amsthm}
\usepackage{mdframed}
\usepackage{thmtools}
\usepackage{times}

\usepackage{amsmath,amssymb,amsthm,mathtools}
\usepackage{tikz}
\usetikzlibrary{fadings,patterns,shadows.blur,shapes,shapes.geometric}
\usepackage{array,booktabs,multirow,tabularx}
\usepackage{yhmath,mathdots,extarrows}
\usepackage{siunitx,gensymb}
\usepackage{physics,cancel,color,xcolor,xspace}
\usepackage{mdframed,thmtools}
\usepackage{pifont}
\usepackage[normalem]{ulem} 
\usepackage{bbold}          
\usepackage{enumitem}

\usepackage{hyperref}
\usepackage{url}

\definecolor{darkgreen}{RGB}{110,192,19}
\definecolor{shadecolor}{gray}{0.9}

\newtheoremstyle{informal}
  {3pt}{3pt}{\itshape}{}{\bfseries}{.}{ }
  {Main Result (Informal Theorem)~\thmnumber{#2}}

\theoremstyle{informal}

\newcommand{\debug}[1]{#1}

\newcommand{\newmacro}[2]{\newcommand{#1}{\debug{#2}}}
\newcommand{\newop}[2]{\DeclareMathOperator{#1}{\debug{#2}}}

\newcommand{\R}{\mathbb{R}}

\renewcommand{\P}{\operatorname{\mathbb{P}}}

\newop{\ex}{\mathbb{E}}     
\newop{\simplex}{\hull}

\newmacro{\noise}{U}
\newmacro{\filter}{\mathcal{F}}
\newmacro{\sgda}{\text{SGDA}}
\newmacro{\seg}{\text{SEG}}

\newmacro{\point}{x}
\newmacro{\pointalt}{\alt\point}
\newmacro{\pointaltalt}{\altalt\point}
\newmacro{\points}{\mathcal{X}}
\newmacro{\intpoints}{\relint\points}
\newmacro{\base}{p}
\newmacro{\basealt}{q}
\newmacro{\basealtalt}{u}
\newmacro{\open}{\mathcal{U}}
\newmacro{\closed}{\mathcal{C}}
\newmacro{\cpt}{\mathcal{K}}
\newmacro{\nhd}{\mathcal{U}}
\newmacro{\gmat}{g}
\newmacro{\gdist}{\dist_{\gmat}}
\newmacro{\mfld}{M}
\newmacro{\form}{\omega}
\newmacro{\tvec}{z}
\newmacro{\uvec}{u}
\newmacro{\basin}{\mathbb{B}}
\newmacro{\ball}{\basin}
\newmacro{\sphere}{\mathbb{S}}

\newmacro{\tstart}{0}		
\newmacro{\timealt}{s}		
\newmacro{\horizon}{T}		

\newmacro{\traj}{x}		
\newmacro{\trajalt}{y}		
\newmacro{\trajaltalt}{z}		

\newmacro{\flowmap}{\Theta}		
\DeclarePairedDelimiterXPP{\flowof}[2]{\flowmap_{#1}}{(}{)}{}{#2}		

\newop{\Opt}{Opt}		
\newop{\Sol}{Sol}		
\newop{\gap}{Gap}	
\newop{\dualitygap}{Duality-Gap}		
\newop{\orcl}{Or}		

\newmacro{\tfun}{f}		
\newmacro{\obj}{f}		
\newmacro{\objalt}{g}		
\newmacro{\sobj}{F}		

\newmacro{\gvec}{g}		
\newmacro{\oper}{A}		
\newmacro{\vecfield}{v}		

\newcommand{\sol}[1][\point]{#1^{\ast}}		
\newmacro{\solvec}{\vecfield(\sol)}		
\newmacro{\solpay}{\eq[\payv]}		

\newmacro{\signal}{V}		
\newmacro{\step}{\gamma}		
\newmacro{\learn}{\eta}		

\newmacro{\vbound}{G}		
\newmacro{\lips}{\ell}		
\newmacro{\strong}{\mu}		
\newmacro{\smooth}{\beta}		

\newop{\tspace}{T}		
\newop{\tcone}{TC}		
\newop{\dcone}{\tcone^{\ast}}		
\newop{\ncone}{NC}		
\newop{\pcone}{PC}		
\newop{\hull}{\Delta}		

\newmacro{\cvx}{\mathcal{C}}		
\newmacro{\subd}{\partial}		
\renewcommand{\eqref}[1]{(\ref{#1})}

\newmacro{\minmax}{\mathcal{L}}		

\newmacro{\minvar}{{\point_{1}}}		
\newmacro{\minvaralt}{\alt\minvar}		
\newmacro{\minvars}{\points_{1}}		
\newmacro{\minsol}{\sol[\minvar]}		

\newmacro{\maxvar}{\point_{2}}		
\newmacro{\maxvaaltr}{\alt\maxvar}		
\newmacro{\maxvars}{\points_{2}}		
\newmacro{\maxsol}{\sol[\maxvar]}		

\newop{\Eucl}{\Pi}		
\newop{\logit}{\Lambda}		
\newop{\dkl}{KL}		

\newmacro{\hreg}{h}		
\newcommand{\onlyappendixintoc}{%
  \let\oldaddcontentsline\addcontentsline
  \renewcommand{\addcontentsline}[3]{}%
}
\newcommand{\restoreaddcontentsline}{%
  \let\addcontentsline\oldaddcontentsline
}
\definecolor{thmsteelbg}{RGB}{235,244,250}         
\definecolor{thmblueborder}{RGB}{45,105,145}      

\definecolor{lemmintborder}{RGB}{238,247,240}          
\definecolor{lemmint}{RGB}{185,220,198}        

\definecolor{propamber}{RGB}{255,244,235}        
\definecolor{propamberborder}{RGB}{188,105,56}     

\definecolor{assumlavborder}{RGB}{246,238,246}       
\definecolor{assumlav}{RGB}{130,82,130}    

\definecolor{exampletealbg}{RGB}{235,248,247}      
\definecolor{exampletealborder}{RGB}{45,140,135}   
\declaretheoremstyle[
  headfont=\normalfont\bfseries\color{thmblueborder},
  notefont=\normalfont\mdseries\color{thmblueborder},
  notebraces={(}{)},
  bodyfont=\normalfont,
  postheadspace=0.5em,
  spaceabove=6pt, spacebelow=4pt,
  mdframed={
    skipabove=6pt, skipbelow=6pt,
    linewidth=1.5pt,
    linecolor=thmblueborder,
    backgroundcolor=thmsteelbg,
    innerleftmargin=8pt, innerrightmargin=8pt,
    innertopmargin=5pt, innerbottommargin=5pt,
    roundcorner=4pt}
]{thmbox}

\declaretheoremstyle[
  headfont=\normalfont\bfseries\color{lemmintborder},
  notefont=\normalfont\mdseries\color{lemmintborder},
  notebraces={(}{)},
  bodyfont=\normalfont,
  postheadspace=0.5em,
  spaceabove=6pt, spacebelow=4pt,
  mdframed={
    skipabove=6pt, skipbelow=6pt,
    linewidth=1.5pt,
    linecolor=lemmintborder,
    backgroundcolor=lemmint,
    innerleftmargin=8pt, innerrightmargin=8pt,
    innertopmargin=5pt, innerbottommargin=5pt,
    roundcorner=4pt}
]{lembox}

\declaretheoremstyle[
  headfont=\normalfont\bfseries\color{propamberborder},
  notefont=\normalfont\mdseries\color{propamberborder},
  notebraces={(}{)},
  bodyfont=\normalfont,
  postheadspace=0.5em,
  spaceabove=6pt, spacebelow=4pt,
  mdframed={
    skipabove=6pt, skipbelow=6pt,
    linewidth=1.5pt,
    linecolor=propamberborder,
    backgroundcolor=propamber,
    innerleftmargin=8pt, innerrightmargin=8pt,
    innertopmargin=5pt, innerbottommargin=5pt,
    roundcorner=4pt}
]{propbox}

\declaretheoremstyle[
  headfont=\normalfont\bfseries\color{assumlavborder},
  notefont=\normalfont\mdseries\color{assumlavborder},
  notebraces={(}{)},
  bodyfont=\normalfont,
  postheadspace=0.5em,
  spaceabove=6pt, spacebelow=4pt,
  mdframed={
    skipabove=6pt, skipbelow=6pt,
    linewidth=1.5pt,
    linecolor=assumlavborder,
    backgroundcolor=assumlav,
    innerleftmargin=8pt, innerrightmargin=8pt,
    innertopmargin=5pt, innerbottommargin=5pt,
    roundcorner=4pt}
]{assumbox}

\declaretheorem[style=propbox,within=section]{definition}
\declaretheorem[style=thmbox,sibling=definition]{theorem}

\declaretheorem[style=lembox,sibling=definition]{lemma}

\title{
Certified Robustness from Approximate Gaussian Mixture Structures in Pretrained Latent Spaces
}

\author{Konstantinos Emmanouilidis \\
  CS \& MINDS\\
  Johns Hopkins University\\
  \And
   Tianjiao Ding  \\
  CIS \\
  University of Pennsylvania
  \And Nghia  Nguyen \\
  CIS \\
  University of Pennsylvania\\
  \And
   Nicolas Loizou \\
  AMS \& MINDS \\
  Johns Hopkins University\\
  \And
  René Vidal \\
  ESE, Radiology \& IDEAS \\
  University of Pennsylvania \\
}

\begin{document}
\onlyappendixintoc        

\maketitle

\begin{abstract}
Deep learning models are vulnerable to adversarial perturbations, raising important concerns for safety-critical deployment. Empirical defenses can achieve strong robustness in practice, but lack formal guarantees, motivating the need for certifiably robust classifiers. While certified methods provide formal guarantees, they often yield overly conservative bounds due to their inability to exploit structure in complex data distributions. In this work, we propose a framework for designing certifiably robust classifiers that leverages latent structure in data representations. We first analyze the Gaussian mixture setting, deriving necessary and sufficient conditions for the existence of robust classifiers and constructing a classifier with a closed-form robustness certificate and generalization guarantees. Our main contribution is to show that exact structure is not required: we prove that if a pretrained encoder maps inputs to a latent distribution that is $\varepsilon$-close (in KL divergence) to a Gaussian mixture, then certified accuracy degrades gracefully, with an explicit bound relating robustness under the true and approximate distributions. This result enables the direct use of pretrained models without requiring exact distributional assumptions. Empirically, our method achieves state-of-the-art or competitive certified accuracy on CIFAR-10 and ImageNet, while maintaining strong clean performance and low computational overhead. Overall, our work establishes approximate latent structure as a practical and principled route to certifiable robustness.
\end{abstract}

\section{Introduction}

Deep learning models are vulnerable to adversarial perturbations, raising fundamental concerns for safety-critical deployment~\citep{szegedy2013intriguing}. A large body of work has sought to address this challenge through two complementary approaches. On the one hand, empirical defenses such as adversarial training, preprocessing, and denoising, can achieve strong robustness in practice~\citep{madry2017towards, shafahi2019adversarial, wong2020fast}, but offer no formal guarantees and are often circumvented by adaptive attacks~\citep{athalye2018obfuscated, carlini2019evaluating}. On the other hand, certified defenses provide provable guarantees of robustness, via randomized smoothing, convex relaxations, or interval bound propagation~\citep{cohen2019certified, wong2018provable, gowal2018effectiveness}. However, these methods typically yield conservative bounds that fail to capture the structure of real-world data distributions.

This gap between empirical robustness and certified guarantees is not merely algorithmic, but fundamentally statistical. Recent impossibility results show that, without assumptions on the data distribution, certifying robustness is inherently difficult~\citep{dohmatob2019generalized, shafahi2018adversarial}. This has motivated a shift toward \emph{structure-aware} approaches, where robustness is tied to properties of the underlying distribution. In particular, \citet{paladversarial, palcertified} establish necessary and sufficient conditions for the existence of robust classifiers based on distributional localization. While conceptually powerful, these results remain difficult to instantiate in practice: the localization sets are implicit, the required parameters are hard to estimate, and the resulting classifiers are not readily scalable to modern datasets. These challenges, motivate the following central question:
\begin{tcolorbox}[colback=darkgreen]
\begin{center}
    \textit{How can we leverage the structure of the data distribution to design classifiers \\that are certifiably robust and efficient in practice?}
\end{center}
\end{tcolorbox}

\textbf{Key idea.} In this work, we show that certifiable robustness can be achieved by exploiting \emph{approximate structure} in pretrained latent data distribitions. Specifically, we demonstrate that it is sufficient for the pretrained latent data distribution to be \emph{approximately} a Gaussian mixture. This perspective yields a practical and theoretically grounded route to certification: rather than requiring exact distributional assumptions, we can leverage pretrained models whose latent spaces exhibit approximate structure.

\textbf{Our approach.} We develop this idea in two stages. First, we analyze the Gaussian mixture setting, deriving necessary and sufficient conditions under which robust classifiers exist, and constructing a classifier with a closed-form certificate of robustness. Crucially, these conditions are expressed in terms of explicit geometric quantities, such as means, covariances, and separations, making them practically verifiable. We then extend this framework to arbitrary data distributions by composing it with a locally Lipschitz pretrained encoder. Our central theoretical result shows that exact Gaussianity is not required: if the latent distribution is $\varepsilon$-close (in KL divergence) to a Gaussian mixture, then the certified accuracy degrades gracefully, with an explicit bound relating robustness under the true and approximate distributions. This result enables the direct use of pretrained models, such as CLIP-like encoders, without requiring exact control over the data distribution.

\textbf{Contributions.} Our contributions are summarized as follows:
\begin{enumerate}[leftmargin=*]
    \item \textit{Verifiable conditions for robustness in Gaussian mixtures.} We derive necessary and sufficient conditions for the existence of robust classifiers in Gaussian mixture models, expressed in closed form in terms of the geometry of the distribution (e.g., Gaussian mixture means and covariances).
    
    \item \textit{A certifiably robust classifier with closed-form guarantees.} 
    We construct a classifier that leverages the geometry of the underlying data distribution and admits an explicit certificate of robustness, along with corresponding generalization bounds.
    
    \item \textit{Extension to real-world data via approximate latent structure.} We prove that if a locally Lipchitz pretrained encoder maps inputs to a latent distribution that is $\varepsilon$-close to a Gaussian mixture, then certified accuracy degrades gracefully, providing a practical pathway to certification.
    
    \item \textit{Empirical validation.} We demonstrate competitive or superior certified accuracy on CIFAR-10 and ImageNet, while maintaining strong clean performance.
\end{enumerate}



\begin{figure}[t]
    \centering
    \includegraphics[width=.6\linewidth]{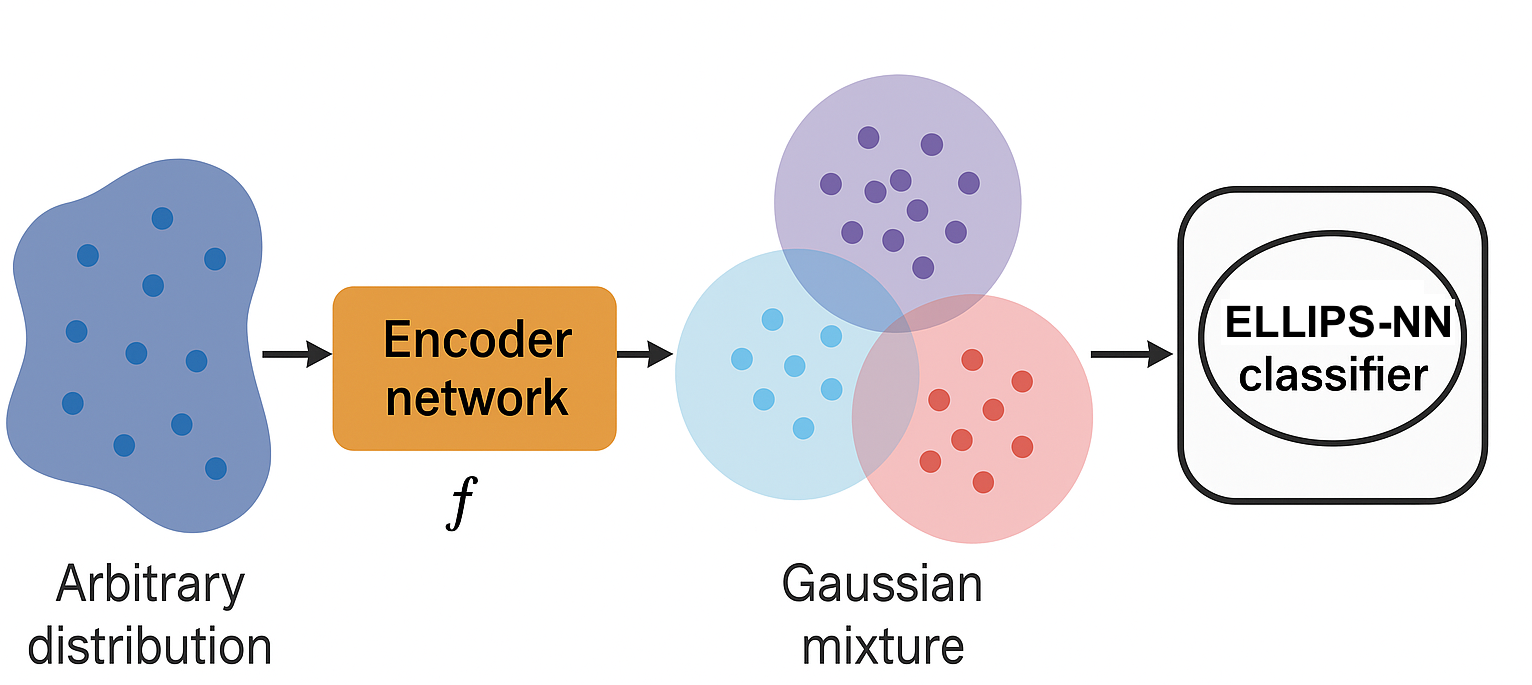}
    \caption{\small\textbf{Pipeline of the proposed certifiably robust classifier.} 
    A pretrained encoder maps inputs to a latent representation whose distribution is well approximated by a Gaussian mixture. We first derive verifiable conditions for robustness in the Gaussian mixture setting and use them to construct the \textsc{ELLIPS} classifier with a closed-form certificate. The resulting pipeline is certifiably robust by combining the local Lipschitzness of the encoder, the certified robustness of \textsc{ELLIPS}, and a graceful degradation guarantee when the latent distribution is only approximately Gaussian-mixture structured.}
    \label{fig: pipeline}
\end{figure}

\section{Preliminaries and Background}

In this section we will introduce the setting as well as the required background that will be necessary for presenting our results in the rest of the paper. 

\subsection{Setup \& Preliminaries}
Consider the canonical setting of a classification problem over the input space $\cX \subset \R^d$ and label space $\mathcal{Y} = \{1, ..., K\}$. Let $\mathcal{D}$ be the data distribution over $\cX \times \cY$ with joint probability density function $p(x, y)$ and marginal $p_i(x) = p(x|y=i)$\footnote{For notational convenience, we refer in the rest of the paper to $\mathcal{D}$ and $p(x, y)$ interchangeably.} for class $i \in \cY$. Let $[n]$ denote the set $\{1, \dots, n\}$ and $\mathbb{B}_2(x, r)$ the $\ell_2-$ball centered at $x$ with radius $r$. Given a classifier $f: \cX \rightarrow \cY$, the \textit{robust risk} of $f$ under any perturbation bounded in $\ell_2$-norm by $\epsilon$ is defined as  
\begin{equation}
    R_f(\epsilon) = \P\limits_{(x, y)\sim p}\left(\exists x'\in\mathbb{B}_{2}(x, \epsilon): f(x') \neq y\right).
\end{equation}

Equipped with the notion of robust risk, we can now provide the definition of a robust classifier. 
\begin{definition} [($\epsilon, \delta$)-robust classifier, \citep{palcertified}]
    A classifier $f$ is said to be $(\epsilon, \delta)$-robust if $R_f(\epsilon) \leq \delta$, i.e., the robust risk against perturbations with $\ell_2$-norm $\epsilon$ is at most $\delta$. 
\end{definition}

Based on the above definition, \citet{palcertified} examined the conditions that the data distribution should satisfy in order for an $(\epsilon, \delta)$-robust classifier to exist. The conditions depend on two notions of localization and strong localization of the data distribution, which we state next.

\begin{definition}[Localized Distribution, \citep{palcertified}]
\label{l0concdefn}
A probability distribution $p$ over a domain $\cX \subseteq \R^n$ is said to be $(C, \epsilon, \delta)$-localized if there exists a subset $S \subseteq \cX$ such that $p(S) \geq 1 - \delta$ and ${\rm Vol}(S) \leq C \exp(-\epsilon)$. Here, ${\rm Vol}$ denotes the standard Lebesgue measure on $\R^n$, and $p(S)$ denotes the measure of $S$ under $p$.
\end{definition}

\begin{definition}[Strongly Localized Distribution, \citep{palcertified}] \label{def:strongconc}
A probability distribution $p$ over a domain $\cX \subseteq \R^n$ is said to be $(C, \epsilon, \delta, \gamma)$-strongly localized with respect to a distance $d$, if each class conditional distribution $p_i$ is $\left(C, \epsilon, \delta\right)$-localized on the set $S_i \subseteq \cX$ and $p_i\left(\bigcup_{i' \neq i} S_{i'}^{+2\epsilon}\right) \leq \gamma$, where $S^{+\epsilon} = \{ x \in \mathcal{X}  \colon \exists \bar x \in S ~ \text{ such that } d(x, \bar x) \leq \epsilon\}$ is the $\epsilon$-expansion of the set $S$ in $d$.
\end{definition} 

Note we stated the notion of strongly localized distribution with a slight modification in the notation relative to \cite{palcertified} in order to make explicit the dependence of the defined notion in the constant $C$. With these concepts established, we, next, provide a summary of the results of \cite{palcertified} as well as other related works before stating our main results in the next section. 

\subsection{Prior Art \& Closely Related Works}
In a series of works, \citet{palcertified, paladversarial} investigate the necessary and sufficient conditions that enable the existence of a robust classifier under different $\ell_p$-norm bounded attacks. For the case of $\ell_2$-adversarial perturbations, they provide the following theorem, establishing the existence of a robust classifier for general data distributions.

\begin{theorem}[\citet{palcertified}]
    A necessary condition for the existence of an $(\epsilon, \delta)$-robust classifier for a data distribution $\mathcal{D}$ with balanced classes is that each class-conditional $\mathcal{D}_i$ be $(C, \epsilon, \delta)$-localized. Conversely, if the data distribution $\mathcal{D}$ is $(C, \epsilon, \delta, \gamma)$-strongly localized, then an $(\epsilon, \delta+\gamma)$-robust classifier exists. 
\end{theorem}


The aforementioned theoretical result indicates that if one can determine whether the underlying distribution is strongly localized, and compute the sets where the distribution localizes over, then the existence of a robust classifier is guaranteed. Intuitively, the proposed $(\epsilon, \delta+\gamma)$-robust classifier in \citet{palcertified} is constructed as a nearest neighbour classifier, assigning each data point to the class corresponding to the nearest localization set. 
Additional related works on certified robustness approaches are provided in Appendix~\ref{app: related_work}. 

The above results shed light on the role of the data distribution in the existence of a robust classifier for a specific classification task. However, they do not specify how one can verify in practice whether a real-world data distribution is strongly localized, find over which sets it localizes and compute its localization parameters in order to construct the proposed robust classifier. This highlights the need for practical conditions that can be utilized to answer the above questions, motivating our theoretical investigation in the next section. 

\section{Sufficient Conditions for the Existence of a Robust Classifier for GMMs}
\label{sec: necessary_and_sufficient_cond}

We first consider a setting where the data distribution $\mathcal{D}$ is a Gaussian Mixture Model (GMM) with $K$ components corresponding to $K$ classes: $\mathcal{D}_i = \mathcal{N}\left(\mu_i, \Sigma_i\right)$, $\forall i \in [K]$. We examine the necessary conditions under which each class conditional in the given mixture of Gaussians is $(C, \epsilon, \delta)-$localized with respect to the $\ell_2$ distance. Ensuring that each Gaussian marginal $\mathcal{D}_i$ is localized will satisfy the requirement according to \cite{paladversarial} in showing the existence of a robust classifier. 


\begin{theorem}\label{thm: necessary-condition}
     Assume that the data distribution $\mathcal{D}$ is a $d$-dimensional GMM and $\mathcal{D}_i = \mathcal{N}\left(\mu_i, \Sigma_i\right)$ corresponds to the class conditional of the $i-$th class. Let $S_i, \forall i \in [K],$ be the ellipsoid set 
    \begin{eqnarray}
        S_i = \{\left(x-\mu_i\right)^T \Sigma^{-1}_i \left(x-\mu_i\right) \leq r_i^2\}. 
    \end{eqnarray}
    Then, each Gaussian marginal $\mathcal{D}_i$ is $(C, \epsilon, \delta)$-localized on the corresponding set $S_i, \forall i \in [K],$ if and only if the parameters $\epsilon, \delta$ satisfy 
    \begin{eqnarray}
        \delta \leq 1 - F_{\chi^2_d}(r_i^2) ,\quad \epsilon \leq \ln{\left(\frac{\Gamma\left(\frac{d}{2}+1\right)C}{\pi^{d/2}r_i^d\sqrt{\text{det}(\Sigma_i)}}\right)},
    \end{eqnarray}
    where $F_{\chi^2_d}(\cdot)$ is the CDF of the $\chi^2_d-$distribution and $\Gamma(\cdot)$ is the Gamma function.
\end{theorem}

Theorem \ref{thm: necessary-condition} provides the necessary conditions for the class conditionals $\mathcal{D}_i$, $\forall i\in[K]$, to be $\left(C, \epsilon, \delta\right)$-localized. More specifically, the established conditions are provided for each localization parameter that if satisfied ensure that the Gaussian marginals localize over the aforementioned sets. Importantly, the localization sets resemble the intuition that most points in a Gaussian marginal concentrate around the mean and the shape of the localization set is dictated by the shape of the corresponding covariance matrix $\Sigma_i, \forall i \in [K]$.

The following theorem provides sufficient conditions under which the data distribution $\mathcal{D}$ is $\left(C, \epsilon, \delta, \gamma\right)$-strongly localized, which consists of a stronger notion of localization.
\begin{theorem}\label{thm: sufficient-condition-strongly-localized}
    The data distribution $\mathcal{D}$ is $\left(C, \epsilon, \delta, \gamma\right)$-strongly localized with respect to the $\ell_2$ distance, if each class conditional $\mathcal{D}_i = \mathcal{N}\left(\mu_i, \Sigma_i\right)$ is $\left(C, \epsilon, \delta\right)$-localized on an ellipsoid set $S_i = \left\{(x-\mu_i\right)^T \Sigma^{-1}_i (x-\mu_i) \leq r_i^2\}$ with parameters 
    $$\delta \leq 1 - F_{\chi^2_d(0)}(r_i^2) ,\quad \epsilon \leq \ln{\left(\frac{\Gamma\left(\frac{d}{2}+1\right)C}{\pi^{d/2}r_i^d\sqrt{\text{det}(\Sigma_i)}}\right)}, \nonumber \\$$
    $$
         \sum\limits_{j\neq i} F_{\chi^2_d(w_{ij})}(R_j^2)\leq\gamma, \nonumber $$
      where $F_{\chi^2_d(w)}$ is the CDF of the $\chi^2_d(w)$ distribution with $d$ degrees of freedom and centrality parameter $w$,  $\Gamma(\cdot)$ is the Gamma function, $\lambda_{\\min}(\Sigma)$ is the smallest eigenvalue of $\Sigma$, $w_{ij} = \|\Sigma_j^{-1/2} (\mu_i - \mu_j)\|^2_2$ and
        $R_j = \frac{2\epsilon}{\sqrt{\lambda_{\min}(\Sigma_j)}} + r_j$.
  
\end{theorem}


Let us pause to elaborate on the implications of the theoretical results established so far. Based on the recent work of \cite{palcertified}, if the data distribution is $\left(C, \epsilon, \delta, \gamma\right)$-strongly localized, then there exists an $\left(\epsilon, \delta\right)$-robust classifier. However, given that the aforementioned result holds for any distribution $\mathcal{D}$, the previous work of \citet{palcertified} does not characterize the localization sets nor provides closed-form expressions for each localization parameter. Instead, by focusing on a structured setting instead we are able to define the localization sets $S_i$ and provide in Theorem~\ref{thm: sufficient-condition-strongly-localized} practical sufficient conditions for the existence of a provably robust classifier.
On the other hand, \citep{palcertified} proves that when classes are balanced, a robust classifier exists only if all class-conditionals are localized. 
Thus, using Theorem~\ref{thm: necessary-condition} we can provide a way of testing whether there is an $\left(\epsilon, \delta\right)$-robust classifier for the underlying data distribution. 

\section{A Provably Robust Classifier for $\ell_2$ Attacks} 
\label{sec: ELLIPS_NN_classifier}
In this section, we show how to construct a provably robust classifier against $\ell_2$ attacks for a Gaussian mixture model utilizing the intuition developed in the previous theoretical results. According to the established results of Section~\ref{sec: necessary_and_sufficient_cond}, if the underlying Gaussian mixture is strongly localized over the ellipsoid sets $S_i$, then a robust classifier is guaranteed to exist. Intuitively the robust classifier should depend on and leverage the structure of the localization sets in order to classify the inputs correctly.

The proposed \emph{nearest ellipsoid} (\textsc{ELLIPS}) classifier operates on ellipsoids $\cE \!\!=\!\! \{E_1, E_2, \ldots, E_K\}$, each one having an associated label $y_i \in \{1, 2, \ldots, K\}$ corresponding to the class $i$, and can be seen as an instantiation of the Bayes classifier for that setting. The ellipsoid $E_i$ is defined by the tuple $(\mu_i, \Sigma_i)_{i\in[K]},$ where $\mu_i, \Sigma_i$ denote the center and covariance matrix of the given ellipsoid. We, also, let $\Pi = (\pi_i)_{i\in[K]}$ be the set of priors of the marginal Gaussian distributions.
Having access to the sets $\cE, \Pi$, the proposed classifier is given by 
\begin{eqnarray}
   \textsc{ELLIPS}(x, \cE, \Pi) &=& y_{i^\star},
\end{eqnarray}
where
\begin{eqnarray}
   i^\star = \text{argmax}_{i\in[K]} \{\text{score}(x, E_i, \pi_i)\}, \nonumber 
\end{eqnarray}
\begin{equation*}
    \text{score}(x, E_i, \pi_i) = - d_M^2(x, E_i) - \log\left(\text{det}(\Sigma_i)\right) + 2 \log(\pi_i) \nonumber
\end{equation*}
and $d_M(x, E_i)$ denotes the Mahalanobis distance of the input sample $x$ from the $i$-th ellipsoid.

The classifier \textsc{ELLIPS} is a nearest ellipsoid classifier with respect to the $d_M$ distance that takes into account two additional terms regarding the shape of the ellipsoid defined by $\Sigma_i$ and the prior $\pi_i$. In this way, the proposed classifier can leverage the geometry of the underlying distribution in order to effectively classify the corresponding input $x \in \cX$. 
Notably, the \textsc{ELLIPS} classifier coincides with the well-known in the literature Quadratic Discriminant Analysis (QDA), whose properties in classification are widely analyzed.
\subsection{Certificate of Robustness}
In this section, we provide a certificate of robustness against $\ell_2$ adversarial attacks for the \textsc{ELLIPS} classifier. In order to establish theoretical guarantees for the certificate, we first need to define the notion of margin. Formally, the margin of the \textsc{ELLIPS} classifier at a point $x \in \cX$ is defined as:
\begin{eqnarray}
    m(x) = \text{score}(x, E_{i_*}, \pi_{i_*}) - \text{score}(x, E_{i_2}, \pi_{i_2}), \nonumber
\end{eqnarray}
where $i_* = \text{argmax}_{i\in[K]} \{\text{score}(x, E_i, \pi_i)\}$ and $i_2 = \text{argmax}_{i\neq i_*} \{\text{score}(x, E_i, \pi_i)\}$ are the classes with the highest and second highest score, respectively. The following theorem provides a certificate of robustness for the \textsc{ELLIPS} classifier based on the margin of each point $x\in\cX$. 

\begin{theorem}[Robustness Certificate] 
\label{thm: certificate_robustness2}It holds that 
$$\textsc{ELLIPS}(x, \cE, \Pi) = \textsc{ELLIPS}(x', \cE, \Pi)$$
whenever 
\begin{eqnarray}
   \|x' - x\|_2 &\leq& 
   \frac{m(x)}{\sqrt{c_M^2 + (-\lambda_{\min}^{W_{i}})_+ m(x)} + c_M} \label{certified_radius}
\end{eqnarray}
where $\lambda_{\min}^{W_{i}}$ is the minimum among all eigenvalues of the matrices $W_i = \Sigma^{-1}_i - \Sigma^{-1}_{i_*}, \forall i \neq i_*$, $(-\lambda_{\min}^{W_{i}})_+=\max(-\lambda_{\min}^{W_{i}},0)$, and $c_M = \max\limits_{i \neq {i_*}} \|\Sigma_{i_*}^{-1} (x - \mu_{i_*})^T - \Sigma_i^{-1} (x-\mu_i)^T\|_2$.
\end{theorem}

Theorem~\ref{thm: certificate_robustness2} provides a certificate of robustness for the \textsc{ELLIPS} classifier. More specifically, it establishes that the proposed classifier remains robust for all perturbations $\|x - x'\|_2$ that satisfy \eqref{certified_radius}. Importantly, the maximum allowed perturbation depends on the margin $m(x)$ and the geometry around the current point $x \in \R^d$. Leveraging information about the classifier's landscape allows for tighter certification based on the local curvature controlling the $\lambda_{\min}^{W_{i}} \in \R$. Specifically, if $\lambda_{\min}^{W_{i}} < 0$, the certified radius in \eqref{certified_radius} corresponds to a second-order certificate, while if $\lambda_{\min}^{W_{i}} \geq 0$ it resembles a first-order formula for certified radius. 


\subsection{Robust Generalization Bound}
In this section, we consider the practical implementation of the \textsc{ELLIPS} classifier and provide a generalization bound for the certificate of robustness of the learnt classifier. So far, we have assumed that the parameters $(\mu_i, \Sigma_i, \pi_i)_{i\in[K]}$ of the ellipsoids are known to analyze the robustness of the proposed classifier. Hereinafter, we consider the learnt classifier \textsc{ELLIPS} which uses the sample mean, sample covariance and class proportions for estimating the true parameters of the underlying distribution. Specifically, if $\{x_j\}_{j=1}^{n_i} \sim \mathcal{D}_i$ are $n_i$ samples from the class $i\in[K],$ the algorithm uses the following estimates 
\begin{equation*}
    \hat{\mu}_i = \frac{1}{n_i} \sum_{j=1}^{n_i} x_j, \quad
    \hat{\Sigma}_i = \frac{1}{n_i} \sum_{j=1}^{n_i} (x_j - \hat{\mu}_i) (x_j - \hat{\mu}_i)^T,
\quad \hat{\pi}_i = \frac{n_i}{n},
\end{equation*}
where $n = \sum_{i=1}^K n_i$ is the total number of samples. 
Based on the above estimates, we derive generalization bounds that establish with high probability the robustness of the learnt classifier. More specifically, the following theorem indicates that with high probability the learnt certificate of robustness is close to the true certificate of robustness, thus ensuring the robustness of the learnt classifier ELLIPS. 
\begin{theorem}\label{thm: generalization_bound}
    For a sample $(x, y),$ let $\mathcal{R}(x), \hat{\mathcal{R}}(x)$ denote the true and learnt radius of robustness respectively. 
    If the number of samples observed from each Gaussian distribution $\mathcal{D}_i, i \in [K]$ is $n = \mathcal{O}\left(\frac{d^{1/4}\log(\frac{1}{\delta})}{\epsilon^{1/4}}\right)$, 
    then for any $0 < \epsilon < \epsilon_{\min}$ it holds with probability at least $1 - \delta$ that
    \begin{eqnarray}
        |\hat{\mathcal{R}} (x) - \mathcal{R}(x)| &\leq& \mathcal{O}\left(\epsilon\right) \nonumber
    \end{eqnarray}
    where $\epsilon_{\min} = \min\{\lambda_{\min}^{W_{i}}, \lambda_{\min}^{\Sigma_{i}}, c_M\}, \lambda_{\min}^{W_{i}}$ is the minimum over all eigenvalues of the matrices $W_i = \Sigma^{-1}_j - \Sigma^{-1}_{j_*}, \forall j \neq j_*$ and $\lambda_{\min}^{\Sigma_i},$ denotes the minimum eigenvalue value of the covariance matrices $\Sigma_i, \forall i \in [K]$. 
\end{theorem}
Theorem~\ref{thm: generalization_bound} provides a generalization bound for the certificate of robustness of the \textsc{ELLIPS} classifier. More specifically, it establishes the required number of samples such that the learnt certified radius of robustness is $\epsilon$-close to the true certified radius of robustness. Interestingly, the provided bound accommondates the change in the expression of \eqref{certified_radius} based on the local geometry induced to take into account both of the closed-form expressions from Theorem~\ref{thm: certificate_robustness2}, thus fully characterizing the generalization of the combined formula of certified radius. We note that the $\mathcal{O}\left(\cdot\right)$ notation hides any dependence on constants, such as the sample difference from the mean $\|x - \mu_i\|_2$, with respect to the parameters $n, d, \delta$.
\section{Generalizing to Complex Real-data Distributions}\label{sec: generalize-to-realworld}

In this section, we focus on constructing certifiably robust classifiers for arbitrary input distributions.
 The previous analysis of robustness in the Gaussian mixture case, as we shall see, is not merely a special case of the more general construction, but rather serves as the fundamental block en route to treating arbitrary, complex data distributions. More specifically, we leverage existing pretrained encoders to map the input distribution approximately to a Gaussian mixture and then perform a fine-grained analysis of robustness in the structured setting of the latent space. 



The intuition behind the aforementioned design paradigm is that one can leverage existing pretrained encoders and then finetune them appropriately in order to be used in the proposed pipeline. Actually, there is a wide range of Visual Language Models, such as CLIP, FARE-4 \citep{schlarmann2024robustclip}, and variants thereof, that can be incorporated in the proposed framework to obtain a certifiably robust classifier on an arbitrary data distribution. We provide the implementation details regarding fine-tuning in Section~\ref{sec: experiments}.

We, next, provide a certificate of robustness for the proposed generalized classifier (\textsc{GenELLIPS}) that acts on any arbitrary input distribution and uses an encoder network $f$ that is locally Lipschitz continuous and maps approximately to an Gaussian mixture distribution. Formally, the \textsc{GenELLIPS} classifier is defined as $\textsc{GenELLIPS}(x, f, \cE, \Pi) = \textsc{ELLIPS}(f(x), \cE, \Pi)$, where the parameters $\cE, \Pi$ are obtained by fitting Gaussians in the latent representations of the data.
Denote, also, with $i_* = \text{argmax}_{i\in[K]} \{\text{score}(f(x), E_i, \pi_i)\}\nonumber$ the class with the highest score for a point $x\in \cX$. We this notation at hand, we are ready to present the theorem establishing the certificate of robustness for the \textsc{GenELLIPS} classifier. 
\begin{theorem}\label{thm: gen_ELLIPS}
    Let $f$ be an locally Lipschitz encoder mapping the input distribution $\mathcal{D}_x$ to a latent distribution
$P_z$ and $Q_z$ be a GMM fitted to the latent distribution $P_z$. If $\mathrm{KL}(P_z\|Q_z)\leq \epsilon$, then
\begin{enumerate}[leftmargin=*]
    \item For $\epsilon = 0$, the certified radius of $\textsc{GenELLIPS}(x, f, \cE, \Pi)$ at any point $x\in\cX$ is
\begin{equation}
    R(x) \leq 
\frac{m(f(x))}{L_x \sqrt{c_M^2 + (-\lambda_{\min}^{W_{i}})_+ m(x)} + c_M} \nonumber
\end{equation}
where $\lambda_{\min}^{W_{i}}$ is the minimum among all eigenvalues of the matrices $W_i = \Sigma^{-1}_i - \Sigma^{-1}_{i_*}, \forall i \neq i_*$, $(-\lambda_{\min}^{W_{i}})_+=\max(-\lambda_{\min}^{W_{i}},0)$, and $c_M = \max\limits_{i \neq {i_*}} \|\Sigma_{i_*}^{-1} (x - \mu_{i_*})^T - \Sigma_i^{-1} (x-\mu_i)^T\|_2$.
\item For $\epsilon > 0$, the certified accuracy of \textsc{GenELLIPS} under the true encoded distribution $P_z$ satisfies 
\begin{eqnarray}
  \mathrm{CertAcc}(P_z)
\geq
\mathrm{CertAcc}(Q_z)
-
\sqrt{\frac{\epsilon}{2}} \label{eq: certified_acc_between_}.  
\end{eqnarray}
\end{enumerate}
\end{theorem}

Theorem~\ref{thm: gen_ELLIPS} provides the certified accuracy of the generalized classifier (\textsc{GenELLIPS}). More specifically, if the classifier uses an encoder that maps the input distribution to an approximate GMM, then the theorem guarantees that the certified accuracy under the approximate distribution $P_z$ degrades only by $\sqrt{\frac{\epsilon}{2}}$ from the certified accuracy of the encoder that maps exactly to a GMM distribution $Q_z$. On the other hand, in the case that the encoder is exact, Theorem~\ref{thm: gen_ELLIPS} provides additionally the closed-form expression of the certified radius of the classifier. In each case, by knowing the distance of the two distributions and the local Lipschitz constant of the encoder, the certified accuracy of the \textsc{GenELLIPS} classifier can be computed.

It suffices, now, to select an encoder $f$ and utilize a method for estimating the local Lipschitz constant in practice. 
To instantiate the \textsc{GenELLIPS} classifier we leverage a FARE-4 encoder \citep{schlarmann2024robustclip}, and finetune it using an objective promoting isotropy and Gaussianity of the latent distribution. For a detailed description of the loss used we refer the interested reader to Appendix~\ref{app: experiments_details}. For estimating the local Lipschitz constant at a sample $x$, we utilize the CLEVER method \citep{weng2018evaluatingrobustnessneuralnetworks} with .
\section{Experimental Evaluation}
We conduct experiments on both synthetic data and benchmark datasets validating our theoretical results and evaluating the robustness of our proposed classifier in practice. 

\subsection{Synthetic Experiments}
\label{sec: experiments}

\myparagraph{Setup} We conduct experiments in the Gaussian mixture setting, where the input distribution is comprised of $K$ classes and each class is distributed according to $\mathcal{N}(\mu_i, \Sigma_i), \forall i \in [K]$. We run experiments for multiple setups testing for different number of classes $K=\{2, 3, 5, 10\}$ with different distances $R = \{2, 4, 6\}$ between them, as well as isotropic and non-isotropic covariances matrices $\Sigma$. We provide further details on our synthetic experiments in Appendix~\ref{app: experiments_details}.  

\textbf{Comparison of Our Method with \citet{paladversarial}.} We empirically validate the robustness certificate for the \textsc{ELLIPS} classifier from Theorem~\ref{thm: certificate_robustness2} and compare the certified accuracy of our method with the one of \citet{palcertified}, where the certified accuracy equals to $1 - \delta - \gamma$. As shown in Figure~\ref{fig: comparison_with_Pal}, our approach consistently provides tighter certified robustness guarantees across all experimental settings, significantly outperforming the method of \citet{paladversarial}.

\begin{figure}[ht]
    \centering
    \includegraphics[width=0.34\linewidth]{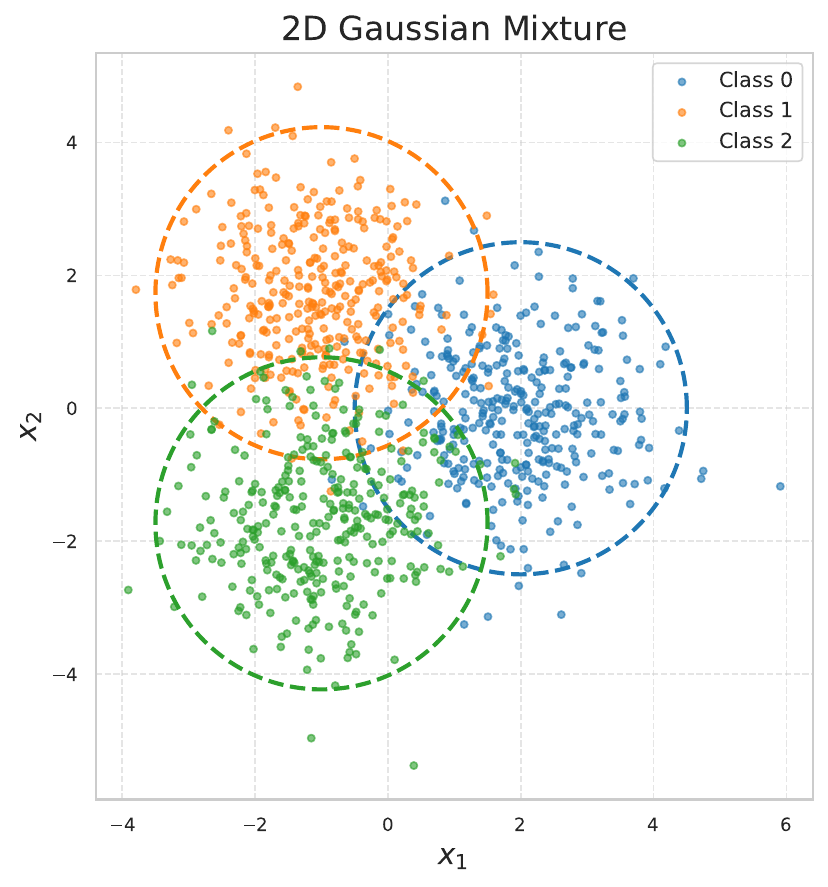}
    \includegraphics[width=0.41\linewidth]{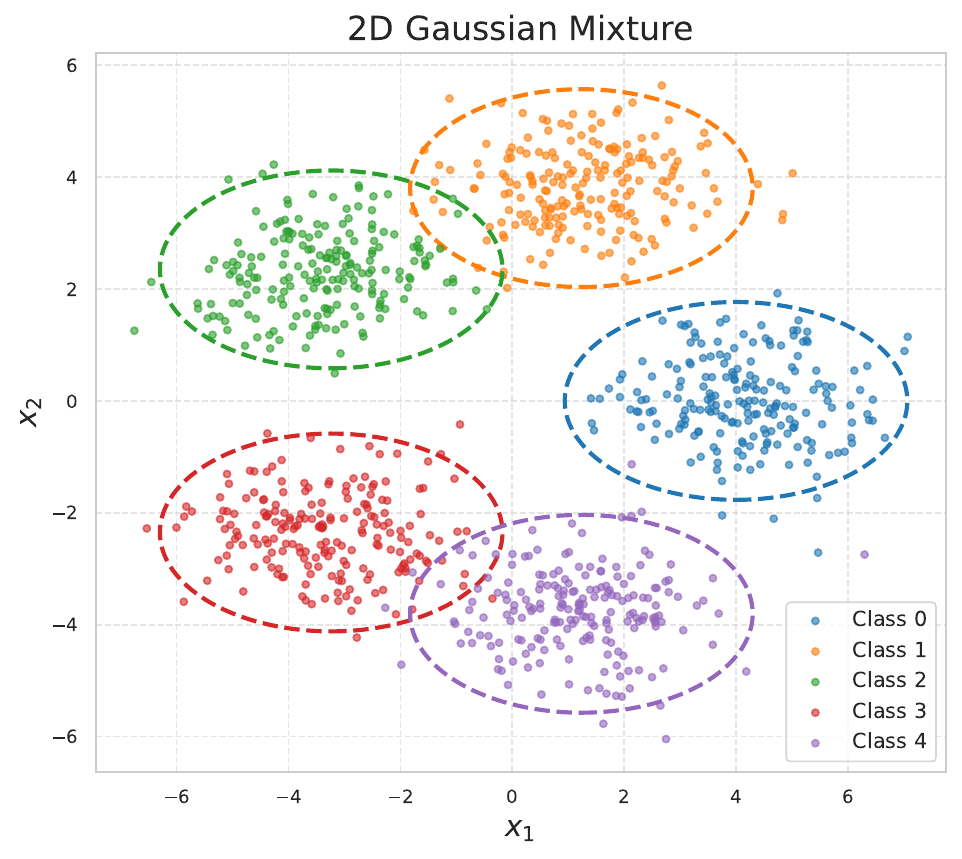}\\
    \includegraphics[width=0.37\linewidth]{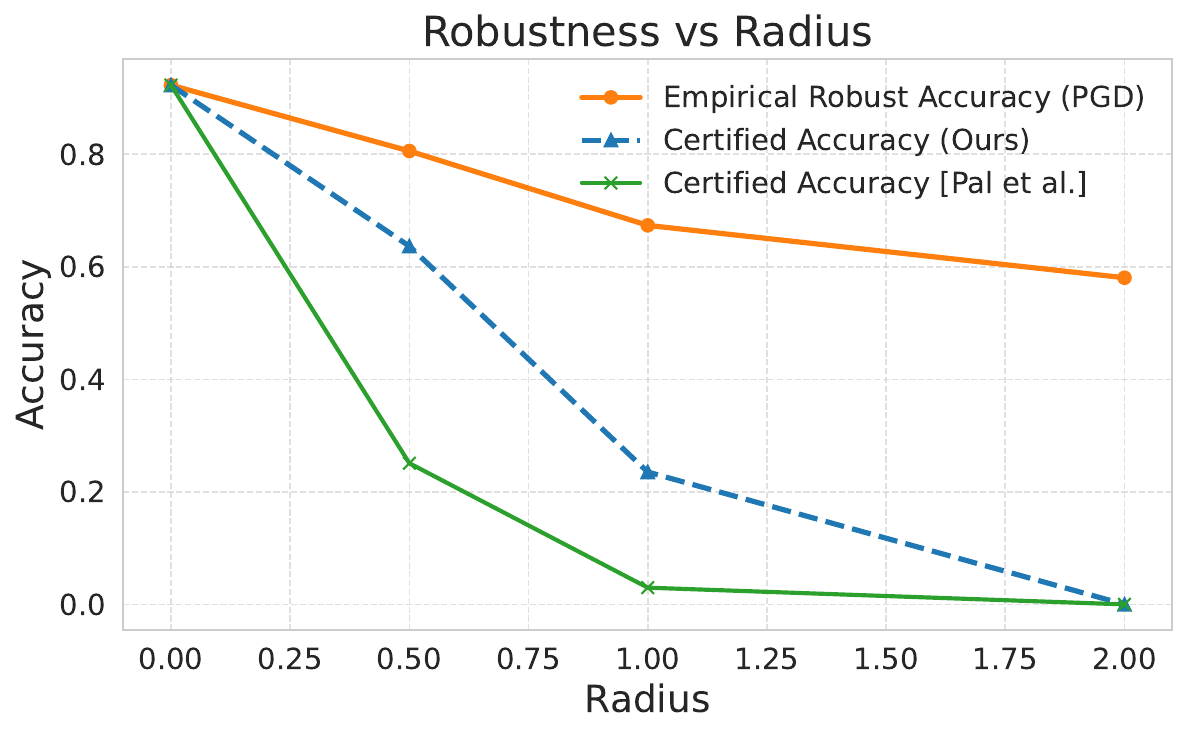}
     \includegraphics[width=0.37\linewidth]{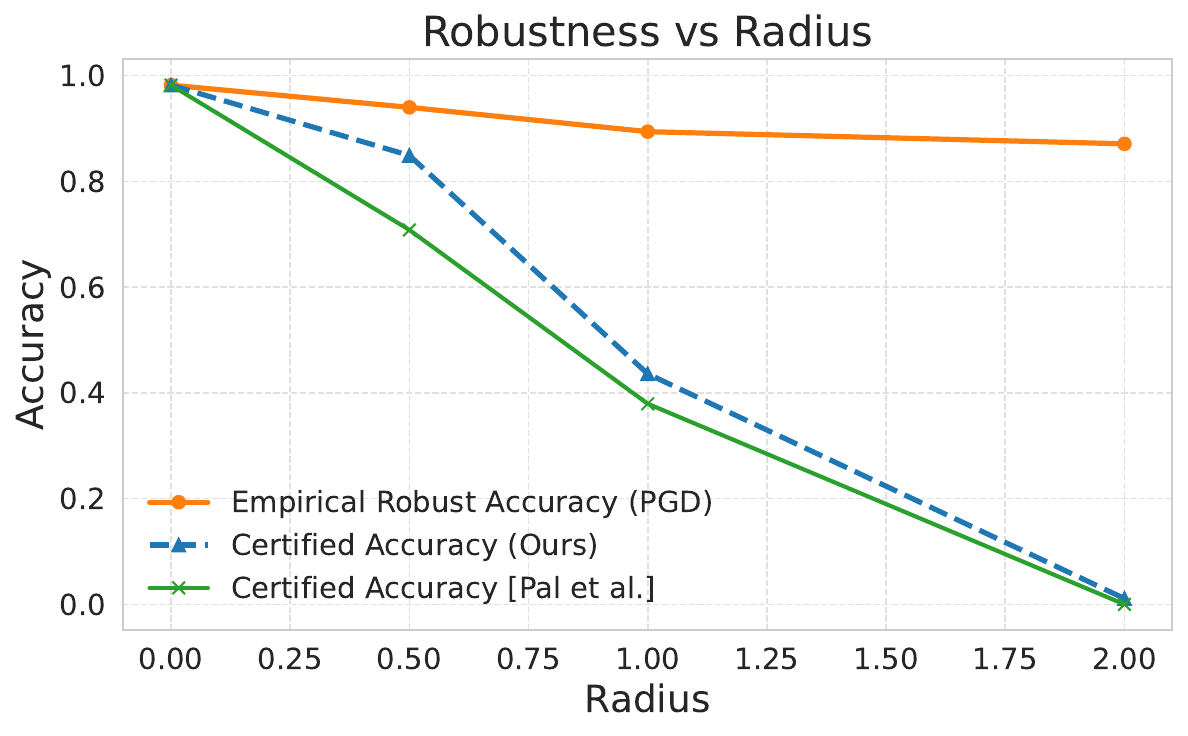}
    \caption{Comparison of different certification methods in Gaussian Mixture distributions. The proposed method outperforms the prior certification scheme of \citet{paladversarial}, achieving higher robust accuracy against $\ell_2$ attacks.}
    \label{fig: comparison_with_Pal}
\end{figure}
    

\subsection{Experiments on Benchmark Datasets}
\begin{table*}[ht]
   \centering
  \begin{tabular}{@{}l c ccc@{}}
    \toprule
    \multirow{2}{*}{\textbf{Model}} 
      & \multirow{2}{*}{\shortstack{\textbf{Clean} \\  \textbf{Accuracy}}} 
      & \multicolumn{3}{c}{\textbf{Certified Accuracy (\%)}} \\
    \cmidrule(lr){3-5}
      & 
      & $\epsilon = 0.25$ 
      & $\epsilon = 0.5$ 
      & $\epsilon = 1.0$ \\
    \midrule
    SmoothAdv \citep{salman2019provably} 
      & \underline{86.2\%} 
      & \underline{81.0\%} 
      & 54.4\% 
      & 34.8\% \\
    DRT + MME (Gaussian) \citep{yang2022certifiedrobustnessensemblemodels} 
      & 81.4\% 
      & 70.4\% 
      & 57.8\% 
      & 34.4\% \\
    DRT + MME (SmoothAdv) \citep{yang2022certifiedrobustnessensemblemodels} 
      & 72.6\% 
      & 67.2\% 
      & \underline{60.2\%} 
      & 39.4\% \\
    DRT + WE (SmoothAdv) \citep{yang2022certifiedrobustnessensemblemodels} 
      & 72.6\% 
      & 67.0\% 
      & \underline{60.2\%} 
      & \underline{39.5\%} \\
    \rowcolor{green!30}
    \textbf{\textsc{GenELLIPS} (Ours)} 
      & \textbf{90.14\%} 
      & \textbf{84.5\%} 
      & \textbf{78.2\%} 
      & \textbf{40.5\%} \\
    \bottomrule 
  \end{tabular}
  \caption{Certified accuracy on CIFAR-10 dataset. The proposed method outperforms the state-of-the-art models in the SoK benchmark, achieving higher robust accuracy without compromising the clean accuracy.} 
    \label{table: cifar_10_results}
\end{table*}

\begin{table*}[ht]
  \centering
  \begin{tabular}{@{}lcc@{}}
    \toprule
    \textbf{Model} & $\boldsymbol{\epsilon = 1.0}$ & $\boldsymbol{\epsilon = 2.0}$ \\
    \midrule
    \rowcolor{red!10} DensePure \citep{xiao2022densepureunderstandingdiffusionmodels} & \textbf{67.0\%} & \textbf{42.2\%} \\
    \rowcolor{red!10} Denoising with Pre-trained Diffusion Models \citep{carlini2023certifiedadversarialrobustnessfree} & \underline{54.3\%} & 29.5\% \\
    Randomized Smoothing and Adversarial Training \citep{salman2020denoised} & 45.0\% & 28.0\% \\
    Ensemble Models and Variance Reduction \citep{horváth2022boostingrandomizedsmoothingvariance} & 44.6\% & 28.6\% \\
    Ensemble Models \citep{yang2022certifiedrobustnessensemblemodels} & 44.4\% & 30.4\% \\
    \rowcolor{green!30}
    \textbf{\textsc{GenELLIPS} [Ours]} & 45.7\% & \underline{31.1\%} \\
    \bottomrule
  \end{tabular}
  \caption{Certified accuracy on ImageNet dataset. Our approach performs competitively with the models in the SoK benchmark. Demonstrably, it outperforms all state-of-the-art models apart from the ones that use diffusion models, which might be computationally expensive in practice.}
  \label{table:imagenet_results}
\end{table*}



\textbf{Results.} We compare our method against state-of-the-art certified robustness approaches reported in the SoK benchmark by \citet{li2023sok} in the CIFAR-10 and ImageNet dataset. As shown in Table~\ref{table: cifar_10_results}, \textsc{GenELLIPS} consistently outperforms prior methods in the CIFAR-10 dataset, achieving higher certified accuracy across all perturbation levels $\epsilon = \{0.25, 0.5, 1.0\}$. Notably, the classifier simultaneously maintains superior clean accuracy compared to the reported baselines. The presented results highlight the robustness of the proposed pipeline as well as the practical effectiveness of our certification framework. 

\begin{table*}[ht]
  \centering
  \begin{tabular}{@{}lc@{}}
    \toprule
    \textbf{Model} & \textbf{Wall Time/Image (min)} \\
    \midrule
    \rowcolor{red!10} DensePure \citep{xiao2022densepureunderstandingdiffusionmodels} & 36.47 \\
    \rowcolor{red!10} Denoising with Pre-trained Diffusion Models \citep{carlini2023certifiedadversarialrobustnessfree} & 25.599 \\
    Randomized Smoothing & 3.015 \\
    \rowcolor{green!30}
    \textbf{\textsc{GenELLIPS} [Ours]} & 2.729 \\
    \bottomrule
  \end{tabular}
  \caption{Wall-clock Time on ImageNet dataset. Our approach is 9x-12x faster than the diffusion-based pipelines, offering a computationally light alternative with competitive certified accuracy.}
  \label{table: wall_clock}
\end{table*}
On the ImageNet dataset our method performs competitively against the top baselines for certified accuracy reported in the SoK benchmark. In Table~\ref{table:imagenet_results}, the proposed method outperforms all prior baselines with the only exception the ones that utilize diffusion models and thus incur a significantly higher computation cost for certification. For reference, Table~\ref{table: wall_clock} shows that \textsc{GenELLIPS} is 9x - 12x times faster than the top diffusion baselines, indicating that our method is a light-weight approach to certified robustness with competitive performance. 
\section{Conclusion}
\label{sec: conclusion}
We have proposed a principled framework for leveraging the structure of the data distribution to design classifiers that are both certifiably robust and achieve strong empirical performance. Our theoretical contributions extend prior localization results by providing practical and verifiable conditions for computing localization parameters in Gaussian mixture models, thus ensuring the existence of a robust classifier. Building on the aforementioned results, we introduced a robust classifier that exploits the geometric structure of the underlying distribution and is provably robust against $\ell_2$-adversarial attacks. To handle complex real-world distributions, we generalized our approach using an encoder network that maps inputs to a structured Gaussian mixture, and established a certifiably robust pipeline for any underlying data distribution. Empirical evaluations demonstrated that our method outperforms state-of-the-art robust pipelines, achieving high certified robustness and simultaneously maintaining strong clean accuracy.


\newpage
\bibliographystyle{plainnat}
\bibliography{main}

\newpage

\restoreaddcontentsline   
\appendix
\newpage
\section*{Supplemental Material}
\tableofcontents

\newpage
\section{Additional Related Work on Certified Robustness}
\label{app: related_work}
There has been a great line of work on methods establishing theoretical guarantees in the field of certified robustness. The closely related ones to our theoretical investigation aim to correlate the properties of the underlying data distribution with the existence of a robust classifier. 
In \citet{dohmatob2019generalized, pydi2020adversarial}, the authors focus on a binary classification setting and provide a lower bound on the robust classification risk that can be attained. The established bound depends on the Wasserstein distance between the two class conditional distributions, showing that the robust risk increases as the class conditional become closer. This intuition is further extended in the general multi-class classification setting in \citet{palcertified, paladversarial} by considering the sets where each marginal distribution localizes and measuring their overlap to estimate the robust risk. However, \citet{palcertified} do not provide a practical method for computing the associated localization sets, thus constraining the applicability of the established method in practice. Our work instead expands the previous results by providing concrete expressions for the localization sets and the associated parameters and proposing a classifier that utilizes the localization sets in order to robustly classify the input points. 

Given that the proposed \textsc{ELLIPS} classifier consists an instantiation of the Bayes classifier for the GMM setting, we provide additional theoretical studies on the optimal Bayes classifier for the clean and adversarial classifier on that setting. Recent work of \citet{10197639} uses robust isoperimetry and establishes the closed form expressions of the Bayes optimal classifier for the adversarial classification task for two or three classes. The general case even though a fundamental question to the best of our knowledge remains open. Lastly, a specific examination of the classifier for $\ell_0$ attacks is provided in \citet{10.1145/3417994} establishing an asymptotically optimal robust classifier for the GMM setting. We leave as future work examining whether our approach, that uses an encoder and then a classifier for the GMM setting, can be combined with the robust classifier of \citet{10.1145/3417994} to establish robust high certified accuracy results against $\ell_0$ adversarial attacks. 
\section{Limitations}
\label{sec: limitations}
Since this work is one of first in leveraging the data structure in order to provide classifiers that are efficient and certifiably robust, there is a number of limitations as well as avenues for future research directions. \\
\textbf{Beyond $\ell_2$ Attacks.} One limitation of the current results is that they apply to $\ell_2$ adversarial attacks. An interesting open problem is to investigate how our framework can be adapted to other adversarial threat models, providing certified classifiers under different $\ell_p$-attacks or even semantic adversarial attacks. \\
\textbf{Alternative Fine-tuning.}
Another limitation of our work is that the encoder should be fine-tuned in order to map the input distribution to an approximately latent GMM. The level of approximation defines the degradation of the certified accuracy from the certified accuracy of the exact GMM latent distribution. Thus, investigating methods for better fine-tuning the encoder and achieving better approximation in the latent space will provide even better certified accuracy results. 

\newpage
\section{Proof of Theorem \ref{thm: necessary-condition}}
\label{app: thm: necessary-condition}
\begin{proof}
    In order to prove that each $\mathcal{D}_i, \forall i \in [K],$ is $\left(C, \epsilon, \delta\right)-$localized, we need to show that there is a set $S_i\subseteq \mathcal{X}$ such that the following hold 
    \begin{eqnarray}
        p_i(S_i) &\geq& 1 - \delta \label{t1_0a} \\
        \text{Vol}(S_i) &\leq& C e^{-\epsilon} \label{t1_0b}
    \end{eqnarray}
    where $p_i$ is the density function of $\mathcal{D}_i$.  
    We, first, define the set $S_i$ on which each Gaussian distribution $\mathcal{D}_i$ localizes. To do so, consider the probability density function 
    \begin{eqnarray}
        p_i(x) = \frac{1}{\sqrt{(2\pi)^{d} \text{det}(\Sigma_i)}} e^{-\frac{1}{2} (x-\mu_i)^T \Sigma_i^{-1} (x - \mu_i)} \nonumber
    \end{eqnarray}
    and the $c$ level-set 
    \begin{eqnarray}
        A_c = \left\{x\in \mathcal{X}: p_i(x) \geq c\right\}\label{t1_e}
    \end{eqnarray}
    for some fixed $0 < c < p_i(\mu_i)$. We want to select $c$ such that at least $1 - \delta$ of the mass is included in this level set, so that inequality \eqref{t1_0a} holds. Note that for a fixed $c$ the level-set is an ellipsoid, as it holds that
    \begin{eqnarray}
        p_i(x) &=& c \nonumber \\
        \iff \ln p_i(x) &=& \ln c \nonumber \\
        \iff (x-\mu_i)^T \Sigma_i^{-1} (x - \mu_i) &=& -\left[2\ln(c) +d \ln(2\pi) + \ln(\text{det}(\Sigma_i))\right] \nonumber 
    \end{eqnarray}
    Letting $r_i^2 = -\left[2\ln(c) +d \ln(2\pi) + \ln(\text{det}(\Sigma_i))\right]$ for any $0 < c < p_i(\mu_i)$, the level set in \eqref{t1_e} can be equivalently written as
    \begin{eqnarray}
         S_i =\left\{x\in \mathcal{X}: (x-\mu_i)^T \Sigma_i^{-1} (x - \mu_i) \leq r_i^2\right\} \nonumber
    \end{eqnarray}
    which is the set of points with Mahalanobis distance $d_M(x, \mu_i) \leq r_i$. 
    
    In order for \eqref{t1_0a} to hold, we want to find the level set $c$ of $\mathcal{A}_c$ or equivalently the radius $r_i$ of the set $ S_i $ 
    such that at least $1 - \delta$ of the mass of the Gaussian distribution $\mathcal{N}\left(\mu_i, \Sigma_i\right)$ is included in $S_i$  
    \begin{eqnarray}
        \int_{S_i} p_i(x) dx \geq 1 - \delta \label{t1_d}
    \end{eqnarray}
    The integral in \eqref{t1_d} is the probability that a sample $x\sim\mathcal{D}_i$ lies inside the set $S_i$ and thus we get equivalently that the following should hold 
    \begin{eqnarray}
        \mathbb{P}_{x \sim \mathcal{D}_i} \left(x\in S_i\right) &=& \int_{S_i} p_i(x) dx \geq 1 - \delta \label{t1_f}
    \end{eqnarray}
    By a change of variables $y = \Sigma_i^{-1/2} (x - \mu_i),$ we can transform the density $p_i(x)$ inside the integral to the density of the standard $\mathcal{N}\left(0, I\right)$ Gaussian $f(x) = \frac{1}{\sqrt{(2\pi)^{d}}} e^{-\frac{1}{2} \|y\|^2_2}$ and thus the set $S_i$ can be equivalently written as
    \begin{eqnarray}
         \hat{S}_i = \left\{x\in \mathcal{X}: \|y\|_2^2 \leq r_i^2\right\} .\nonumber
    \end{eqnarray}
    Hence, inequality \eqref{t1_f} after the change of variables $y = \Sigma_i^{-1/2} (x - \mu_i)$ requires
    \begin{eqnarray}
        \mathbb{P}_{x \sim \mathcal{D}_i} \left(x\in \mathcal{X}: \|y\|_2^2 \leq r_i^2\right) \leq 1 - \delta \label{t1_1}.
    \end{eqnarray}
    Note, now, that since $y = \Sigma_i^{-1/2} (x - \mu_i)$ follows the standard Gaussian distribution $\mathcal{N}\left(0, I\right)$, the random variable $\|y\|_2^2$ follows the chi-squared distribution with $d$ degrees of freedom. 
    Hence, the left hand-side of \eqref{t1_1} is exactly the cumulative probability distribution of the $\chi_d^2$ distribution up to $r_i^2$.
    Thus, in order for \eqref{t1_1} to hold, the $r_i^2$ should be the $(1-\delta)$-quantile of $\chi_d^2$, i.e. 
    \begin{eqnarray}
      F_{\chi_{d}^2}(r_i^2) &\leq& 1 - \delta \nonumber \\
      \delta &\leq& 1 - F_{\chi_d^2}(r_i^2) \nonumber
    \end{eqnarray}
    where $F_{\chi_d^2}$ is the cumulative distribution function of the $\chi_d^2$. 

    In order for inequality \eqref{t1_0b} to hold, we have that
    \begin{eqnarray}
        \text{Vol}(S_i) &\leq& C e^{-\epsilon} \nonumber \\
        \iff \frac{\pi^{d/2}r_i^d}{\Gamma\left(\frac{d}{2}+1\right)} \sqrt{\text{det}(\Sigma_i)} &\leq& C e^{-\epsilon} \nonumber \\
        \iff \epsilon &\leq& \ln{\left(\frac{\Gamma\left(\frac{d}{2}+1\right)C}{\pi^{d/2}r_i^d\sqrt{\text{det}(\Sigma_i)}}\right)} \label{t1_2}
    \end{eqnarray}
    where $\Gamma(\cdot)$ is the Gamma function. 
\end{proof}
\newpage
\section{Proof of Theorem~\ref{thm: sufficient-condition-strongly-localized}}
\label{app: thm: sufficient-condition-strongly-localized}
\begin{proof}
    In order to show that $\mathcal{D}$ is $\left(C, \epsilon, \delta, \gamma\right)-$strongly localized, we need to show that for each class conditional $\mathcal{D}_i,\forall i \in [K],$ there is a set $S_i\subseteq \mathcal{X}$ such that the following hold 
    \begin{eqnarray}
        p_i(S_i) &\geq& 1 - \delta \label{t1_0a2} \\
        \text{Vol}(S_i) &\leq& C e^{-\epsilon} \label{t1_0b2} \\
        p_i\left(\bigcup_{j\neq i} S_{j}^{+2\epsilon}\right) &\leq& \gamma \label{t1_0c}
    \end{eqnarray}
    where the set $S_i^{+\epsilon} = \left\{x\in\mathcal{X}: \exists \hat{x} \in S_i \text{ with } \|x - \hat{x}\|_2 \leq \epsilon\right\}$ is the $\epsilon-$expansion of the set $S_i$ with respect to the $\ell_2-$distance. By assumption, we have that the conditions \eqref{t1_0a2}, \eqref{t1_0b2} hold. The last condition (inequality \eqref{t1_0c}) requires that the class conditionals are well-separated in the sense that $p_i\left(\bigcup_{j\neq i} S_{j}^{+2\epsilon}\right), \forall i \in [K],$ is upper bounded. To ensure that, we can apply the union bound to get
    \begin{eqnarray}
        p_i\left(\bigcup_{j\neq i} S_{j}^{+2\epsilon}\right) = \mathbb{P}_{x\sim D_i}\left(\bigcup_{j\neq i} S_j^{+2\epsilon}\right) \leq \sum_{j\neq i}\mathbb{P}_{x\sim D_i}\left(S_j^{+2\epsilon}\right) \label{t1_3}
    \end{eqnarray}
    We, next, bound the probability that a sample $x\sim D_i$ belongs to the set $S_j^{+2\epsilon}$
    \begin{eqnarray}
        \mathbb{P}_{x\sim D_i}\left(S_j^{+2\epsilon}\right) &=& \mathbb{P}_{x\sim D_i}\left(\exists s \in S_j: \|x - s\|_2 \leq 2 \epsilon \right)\nonumber\\
        &=& \mathbb{P}_{x\sim D_i}\left(\exists s \in \mathcal{X}: d_M(s, \mu_j) \leq r_j \text{ and } \|x - s\|_2 \leq 2 \epsilon \right)\label{t1_4}
    \end{eqnarray}
    Since the expression in \eqref{t1_4} involves both the Mahalanobis distance and the $\ell_2-$distance, we will express both conditions $d_M(s, \mu_j) \leq r_j$, $\|x - s\|_2 \leq 2 \epsilon$ in terms of the Mahalanobis distance $d_M(x, \mu_j)$. Using the triangle inequality for the Mahalanobis distance, we get
    \begin{eqnarray}
        d_M(x, \mu_j) &\leq& d_M(x, s) + d_M(s, \mu_j) \nonumber \\
        &\leq& d_M(x, s) + r_j \nonumber \\
        &=& \|\Sigma_j^{-1/2} (x-s)\|_2 + r_j \nonumber \\
        &\leq& \|\Sigma_j^{-1/2}\|_2 \|x - s\|_2 + r_j \nonumber \\
        &\leq& \frac{2\epsilon}{\sqrt{\lambda_{\min}(\Sigma_j)}} + r_j \label{t1_5}
    \end{eqnarray}
    where $\lambda_{\min}(\Sigma_j)$ is the smallest eigenvalue of $\Sigma_j$. 
    Hence, for $x \sim D_i$ the event $\mathcal{E}_j = \{\exists s \in \mathcal{X}: d_M(s, D_j) \leq r_j \text{ and } \|x - s\|_2 \leq 2 \epsilon \}$ is contained in the event $\{d_M^2(x, \mu_j) \leq R^2_j\},$ where $R_j = \frac{2\epsilon}{\sqrt{\lambda_{\min}(\Sigma_j)}} + r_j$. 
    Thus, we can bound the probability in inequality \eqref{t1_4} by
    \begin{eqnarray}
        \mathop{\mathbb{P}}\limits_{x\sim D_i}\left(S_j^{+2\epsilon}\right) &\leq& \mathop{\mathbb{P}}\limits_{x\sim D_i}\left(d_M^2(x, \mu_j) \leq R_j^2\right) \nonumber \\
        &=& \mathop{\mathbb{P}}\limits_{x\sim D_i}\left((x -\mu_j)^T \Sigma_j^{-1} (x-\mu_j) \leq R_j^2\right) \label{t1_5} 
    \end{eqnarray}
    Letting $y = \Sigma^{-1/2}_j (x - \mu_j)$, we get from \eqref{t1_5} that 
    \begin{eqnarray}
        \mathop{\mathbb{P}}\limits_{x\sim D_i}\left(S_j^{+2\epsilon}\right) &\leq& \mathop{\mathbb{P}}\limits_{x\sim D_i}\left(\|y\|^2_2 \leq R_j^2\right) \label{t1_6}
    \end{eqnarray}
    Notice that $x \sim \mathcal{D}_i$ and thus the distribution of the random variable $y$ will have mean $\Sigma_j^{-1/2} (\mu_i - \mu_j)$. Thus, the distribution of $\|y\|^2$ will be a non-central $\chi^2_d-$distribution with $d$ degrees of freedom and centrality parameter $w_{ij} = \|\Sigma_j^{-1/2} (\mu_i - \mu_j)\|^2_2$. 
    From inequality \eqref{t1_6}, we get that 
    \begin{eqnarray}
        \mathop{\mathbb{P}}\limits_{x\sim D_i}\left(S_j^{+2\epsilon}\right) \leq \mathop{\mathbb{P}}\limits_{x\sim D_i}\left(\|y\|^2_2 \leq R_j^2\right) = F_{\chi^2_d(w_{ij})}(R_j^2) \label{t1_7} 
    \end{eqnarray}
    where $F_{\chi^2_d(w_{ij})}(\cdot)$ is the cumulative density function of the $\chi^2_d(w_{ij})$ distribution. 
    
    Substituting inequality \eqref{t1_7} into \eqref{t1_3}, we obtain the following bound
    \begin{eqnarray}
        p_i\left(\bigcup_{j\neq i} S_{j}^{+2\epsilon}\right) &\leq& \sum\limits_{j\neq i}\mathop{\mathbb{P}}\limits_{x\sim D_i}\left(S_j^{+2\epsilon}\right)\nonumber \\
        &\leq& \sum\limits_{j\neq i} F_{\chi^2_d(w_{ij})}(R_j^2)\nonumber
    \end{eqnarray}
    and thus letting $\gamma = \sum\limits_{j\neq i} F_{\chi^2_d(w_{ij})}(R_j^2)$ finishes the proof.  
\end{proof}
\newpage
\section{Proof of Theorem~\ref{thm: certificate_robustness2}}
\label{app: thm: certificate_robustness2}
\begin{proof}
Consider a sample $(x, y)$ with positive margin
\begin{eqnarray}
    m(x) = \text{score}(x, E_y, \pi_y) - \max_{i \neq i_*} \text{score}(x, E_i, \pi_i) > 0\label{eq0_cert_rad} 
\end{eqnarray}
where $i_* = \max_{i\in[K]} \text{score}(x, E_i, \pi_i) = y.$

We want to show that the perturbed sample $x'$ has also positive margin 
\begin{eqnarray}
    m(x') = \text{score}(x', E_{i_*}, \pi_{i_*}) - \max_{i \neq i_*}\text{score}(x', E_i, \pi_{i}) 
\end{eqnarray}
Substituting the definition of $\text{score}$ and rearranging the terms, we have 
\begin{eqnarray}
    m(x') &=& - (x' - \mu_{i_*}) \Sigma_{i_*}^{-1} (x'-\mu_{i_*})^T - \log\left(\text{det}(\Sigma_{i_*})\right) + 2 \log(\pi_{i_*}) \nonumber \\
    && - \max_{i \neq {i_*}} \{- (x' - \mu_i) \Sigma_i^{-1} (x'-\mu_i)^T - \log\left(\text{det}(\Sigma_i)\right) + 2 \log(\pi_i)\} \nonumber \\
    &=&  \min_{i \neq {i_*}} \Big\{- (x' - \mu_{i_*}) \Sigma_{i_*}^{-1} (x' -\mu_{i_*})^T + (x' - \mu_i) \Sigma_i^{-1} (x'-\mu_i)^T \nonumber \\
    && \quad \quad \quad - \log\left(\frac{\text{det}(\Sigma_{i_*})}{\text{det}(\Sigma_i)}\right) + 2 \log\left(\frac{\pi_{i_*}}{\pi_i}\right)\Big\} \quad \quad\label{eq1_cert_rad}
\end{eqnarray}
For any $x, x' \in \cX$ and $\forall i \in [K],$ we have that 
\begin{eqnarray}
    (x' - \mu_i) \Sigma_i^{-1} (x' - \mu_i)^T &=& (x' - x) \Sigma^{-1}_i (x'-\mu_i)^T + (x - \mu_i) \Sigma_i^{-1} (x' - \mu_i)^T \nonumber \\
    &=& (x' - x) \Sigma^{-1}_i (x' - x)^T + 2 (x' - x)\Sigma_i^{-1} (x - \mu_i)^T \nonumber \\ 
    && + (x - \mu_i)\Sigma_i^{-1} (x-\mu_i)^T \label{eq2_cert_rad}
\end{eqnarray}
Using \eqref{eq2_cert_rad} into \eqref{eq1_cert_rad} for the terms $- (x' - \mu_{i_*}) \Sigma_{i_*}^{-1} (x'-\mu_{i_*})^T$ and $(x' - \mu_i) \Sigma_i^{-1} (x'-\mu_{i})^T$, we have that 
\begin{eqnarray}
    m(x') &=& \min_{i \neq {i_*}} \Big\{- (x' - x) \Sigma^{-1}_{i_*} (x' - x)^T - 2 (x' - x)\Sigma_{i_*}^{-1} (x - \mu_{i_*})^T - (x - \mu_{i_*})\Sigma_{i_*}^{-1} (x-\mu_{i_*})^T\nonumber \\
    &&\quad \quad \quad + (x' - x) \Sigma^{-1}_i (x' - x)^T + 2 (x' - x)\Sigma_i^{-1} (x - \mu_i)^T + (x - \mu_i)\Sigma_i^{-1} (x-\mu_i)^T \nonumber \\
    &&\quad \quad \quad - \log\left(\frac{\text{det}(\Sigma_{i_*})}{\text{det}(\Sigma_i)}\right) + 2 \log\left(\frac{\pi_{i_*}}{\pi_i}\right)\Big\} \nonumber \\
    &=& \min_{i \neq {i_*}} \Big\{(x' - x) (\Sigma^{-1}_i - \Sigma^{-1}_{i_*}) (x' - x)^T + 2 (x' - x) [\Sigma_i^{-1} (x - \mu_i)^T - \Sigma_{i_*}^{-1} (x - \mu_{i_*})^T] \nonumber \\
    && \quad \quad \quad - (x - \mu_{i_*})\Sigma_{i_*}^{-1} (x-\mu_{i_*})^T + (x - \mu_i)\Sigma_i^{-1} (x-\mu_i)^T \nonumber \\
    && \quad \quad \quad - \log\left(\frac{\text{det}(\Sigma_{i_*})}{\text{det}(\Sigma_i)}\right) + 2 \log\left(\frac{\pi_{i_*}}{\pi_i}\right)\Big\}  \label{eq2b_cert_rad} 
\end{eqnarray}
Given that the matrix $W_i = \Sigma^{-1}_i - \Sigma^{-1}_{i_*}$ is symmetric, as the difference of inverses of symmetric matrices, it holds that $(x' - x) W_i (x' - x)^T \geq \lambda_{\min}(W_i) \|x' - x\|^2_2$. Using that and Cauchy-Schwarz inequality, we get from \eqref{eq2b_cert_rad} 
\begin{eqnarray}
   m(x') &\geq& \min_{i \neq {i_*}} \Big\{\lambda_{\min}(W_i) \|x' - x\|_2^2 - 2 \|x' - x\|_2 \|\Sigma_i^{-1} (x - \mu_i)^T - \Sigma_{i_*}^{-1} (x - \mu_{i_*})^T\|_2 \nonumber \\
    && \quad \quad \quad - (x - \mu_{i_*})\Sigma_{i_*}^{-1} (x-\mu_{i_*})^T + (x - \mu_i)\Sigma_i^{-1} (x-\mu_i)^T \nonumber \\ 
    && \quad \quad \quad + \log\left(\frac{\text{det}(\Sigma_i)}{\text{det}(\Sigma_{i_*})}\right) + 2 \log\left(\frac{\pi_{i_*}}{\pi_i}\right)\Big\}\quad \quad \label{eq2c_cert_rad} 
\end{eqnarray}
Using the subadditivity of the min operator and the definition of margin $m(x)$ from \eqref{eq0_cert_rad}, we get
\begin{eqnarray}
    m(x') \geq m(x) + \min_{i \neq {i_*}}\lambda_{\min}(W_i) \|x' - x\|_2^2 -2 \|x' - x\|_2 \max_{i \neq {i_*}} \|\Sigma_{i_*}^{-1} (x - \mu_{i_*})^T - \Sigma_i^{-1} (x-\mu_i)^T\|_2 \nonumber
\end{eqnarray}
Letting for brevity $\lambda_{\min} = \min_{i \neq {i_*}}\{\lambda_{\min}(W_i)\}$ and $c_M = \max\limits_{i \neq {i_*}} \|\Sigma_{i_*}^{-1} (x - \mu_{i_*})^T - \Sigma_i^{-1} (x-\mu_i)^T\|_2,$ in order for $m(x')$ to be non-negative, it suffices that
\begin{eqnarray}
    m(x) + \lambda_{\min} \|x' - x\|_2^2 -2c_M \|x' - x\|_2 &>& 0 \nonumber
\end{eqnarray}
If $\lambda_{\min} < 0,$ then we have that
\begin{eqnarray}
    \|x' - x\|_2 &\leq& \frac{c_M - \sqrt{c_M^2 - m(x) \lambda_{\min}}}{\lambda_{\min}} = \frac{m(x)}{c_M +\sqrt{c_M^2 - m(x) \lambda_{\min}}}\label{eq_cert_radius_case_a_app} 
\end{eqnarray}
If $\lambda_{\min} \geq 0,$ then it suffices that 
\begin{eqnarray}
    \|x' - x\|_2 &\leq& \frac{m(x)}{2c_M} \label{eq_cert_radius_case_b_app}.
\end{eqnarray}
Combining the expressions in \eqref{eq_cert_radius_case_a_app} and \eqref{eq_cert_radius_case_b_app}, we get the final result.
\end{proof}
\newpage
\section{Proof of Generalization Bound}
\label{app: generalization_bound}
We, first, provide some necessary notation and preparatory Lemmas for bounding the associated quantities appearing in the generalization bound and then we provide the proof of Theorem~\ref{thm: generalization_bound} in Section~\ref{app: thm: generalization_bound}.
\subsection{Notation}
For a real matrix $A \in \mathbb{R}^{d\times d}$, we denote the minimum eigenvalue of $A$ with $\lambda_{\min}(A)$. We let $\lambda_{\min}^{\Sigma_i}$ denote the minimum over all eigenvalues of the covariance matrices $\Sigma_i, \forall i \in [K]$ and $\lambda_{\min}^{W_{i}}$ or for brevity $\lambda_{\min}$ the minimum over all eigenvalues of the matrices $W_i = \Sigma^{-1}_j - \Sigma^{-1}_{j_*}, \forall j \neq j_*$.
\subsection{Preparatory Lemmas}
\label{app: prep}
\begin{lemma}\label{lemma: mahalanobis_dist_bound}
    For a Gaussian marginal with true mean $\mu$ and covariance matrix $\Sigma$ and empirical mean and covariance $\hat{\mu}, \hat{\Sigma}$ satisfying $\|\hat{\mu} - \mu\|_{\Sigma^{-1}} \leq \mathcal{\Tilde{O}}\left(\sqrt{\frac{d}{n}}\right), \|\hat{\Sigma} - \Sigma\|_{op} \leq \mathcal{\Tilde{O}}\left(\sqrt{\frac{d}{n}}\right),$ we have that for any point $x\in\cX$ it holds that
    \begin{eqnarray}
        |d_{M}^2(x, \mu) - d_M^2(x, \hat{\mu})| &\leq& \mathcal{\Tilde{O}}\left(\sqrt{\frac{d}{n}}\right) \nonumber
    \end{eqnarray}
    where $d_M(x, \mu)$ corresponds to the Mahalanobis distance. 
\end{lemma}
\begin{proof}
We have that
\begin{eqnarray}
    |d_{M}^2(x, \mu) - d_M^2(x, \hat{\mu})| &=& \left|(x-\mu)^T \Sigma^{-1} (x-\mu) - (x-\hat{\mu})^T \hat{\Sigma}^{-1} (x-\hat{\mu})\right| \nonumber 
\end{eqnarray}
Adding and subtracting the term $(x-\hat{\mu})^T \Sigma^{-1} (x-\hat{\mu})$ and applying the triangle inequality, we get
\begin{eqnarray}
   |d_{M}^2(x, \mu) - d_M^2(x, \hat{\mu})| &=& \Big|(x-\mu)^T \Sigma^{-1} (x-\mu) - (x-\hat{\mu})^T \Sigma^{-1} (x-\hat{\mu}) \nonumber \\
   && \quad + (x-\hat{\mu})^T \Sigma^{-1} (x-\hat{\mu}) - (x-\hat{\mu})^T \hat{\Sigma}^{-1} (x-\hat{\mu})\Big| \nonumber\\
   &\leq& \underbrace{\left|(x-\mu)^T \Sigma^{-1} (x-\mu) - (x-\hat{\mu})^T \Sigma^{-1} (x-\hat{\mu})\right|}_{T_1} \nonumber \\
   && + \underbrace{\left|(x-\hat{\mu})^T \left(\hat{\Sigma}^{-1} - \Sigma^{-1}\right) (x-\hat{\mu})\right|}_{T_2} \label{lem_eq0}
\end{eqnarray}
We, next, bound the two terms $T_1, T_2$. By rearranging the terms in $T_1$, we get
\begin{eqnarray}
    T_1 &=& \left|(x-\mu)^T \Sigma^{-1} (x-\mu) - (x-\hat{\mu})^T \Sigma^{-1} (x-\hat{\mu})\right| \nonumber \\
    &=& \left|(\hat{\mu} - \mu)^T \Sigma^{-1} x -(x-\mu)^T \Sigma^{-1} \mu +(x-\hat{\mu})^T \Sigma^{-1} \hat{\mu}\right|\nonumber\\
    &=& \left|(\hat{\mu} - \mu)^T \Sigma^{-1} x-(x-\mu)^T \Sigma^{-1} \mu +(x-\hat{\mu})^T \Sigma^{-1} (\hat{\mu} - \mu) + (x-\hat{\mu})^T \Sigma^{-1} \mu \right| \nonumber\\
    &=& \left|2 \Delta\hat{\mu}^T \Sigma^{-1} (x-\mu) + \Delta\hat{\mu}^T \Sigma^{-1} \Delta\hat{\mu} \right| \nonumber \\
    &\leq& 2\|\Delta\hat{\mu}\|_{\Sigma^{-1}} \|x - \mu\|_{\Sigma^{-1}} + \|\Delta\hat{\mu}\|_{\Sigma^{-1}}^2 \nonumber 
\end{eqnarray}
Using the fact that $\|\Delta\hat{\mu}\|_{\Sigma^{-1}} = \mathcal{\Tilde{O}}\left(\sqrt{\frac{d}{n}}\right)$, we obtain
\begin{eqnarray}
    T_1 &\leq& \mathcal{\Tilde{O}}\left(\sqrt{\frac{d}{n}}\right) \label{lem_eq1}
\end{eqnarray}
We can bound the term $T_2$ using the inequality
\begin{eqnarray}
    T_2 = \left|(x-\hat{\mu})^T \left(\hat{\Sigma}^{-1} - \Sigma^{-1}\right) (x-\hat{\mu})\right| \leq \|x-\hat{\mu}\|^2_2 \|\hat{\Sigma}^{-1} - \Sigma^{-1}\|_{\text{op}} \nonumber
\end{eqnarray}
Adding and subtracting the true mean $\mu$ and utilizing the definition of $\Delta \hat{\mu} = \hat{\mu} - \mu$, we get 
\begin{eqnarray}
   T_2 &\leq&  \|x-\mu + \Delta \hat{\mu}\|^2_2 \|\hat{\Sigma}^{-1} - \Sigma^{-1}\|_{\text{op}} \nonumber \\
   &\stackrel{}{\leq}& 2\left(\|x - \mu\|^2_2 + \|\Delta\hat{\mu}\|_2^2\right) \|\hat{\Sigma}^{-1} - \Sigma^{-1}\|_{\text{op}} \nonumber
\end{eqnarray}
where at the last step we have applied the inequality $\|a+b\|^2 \leq 2\|a\|^2+2\|b\|^2$.
Using the fact that $\|\Delta\hat{\mu}\|_2 = \mathcal{\Tilde{O}}\left(\sqrt{\frac{d}{n}}\right), \|\hat{\Sigma}^{-1} - \Sigma^{-1}\|_{\text{op}} = \mathcal{\Tilde{O}}\left(\sqrt{\frac{d}{n}}\right),$ we obtain that
\begin{eqnarray}
    T_2 &\leq& \mathcal{\Tilde{O}}\left(\frac{d^{3/2}}{n^{3/2}}\right) \label{lem_eq2} 
\end{eqnarray}
Substituting inequalities \eqref{lem_eq1}, \eqref{lem_eq2} into \eqref{lem_eq0}, we have that
\begin{eqnarray}
  |d_{M}^2(x, \mu) - d_M^2(x, \hat{\mu})| &\leq& \mathcal{\Tilde{O}}\left(\sqrt{\frac{d}{n}}\right) + \mathcal{\Tilde{O}}\left(\frac{d^{3/2}}{n^{3/2}}\right) \leq \mathcal{\Tilde{O}}\left(\sqrt{\frac{d}{n}}\right) \nonumber
\end{eqnarray}
\end{proof}
\begin{lemma}\label{lemma: basic_matrix_inequality}
    For any symmetric positive definite matrices $A, B$ with $\lambda_{\min}(A) > \|B-A\|_{\text{op}}$, it holds that
    \begin{eqnarray}
        |\log \text{det}(A) - \log \text{det}(B)| &\leq& \frac{d \|B-A\|_{\text{op}}}{\lambda_{\min}(A)-\|B-A\|_{\text{op}}} \nonumber
    \end{eqnarray}
\end{lemma}
\begin{proof}
   From the trace representation, we have that for any two symmetric positive definite matrices $A, B$ it holds that
    \begin{eqnarray}
        |\log \text{det}(A) - \log \text{det}(B)| &=& \left|\text{Tr}\left(\int_{0}^{1} \left(A + t (B -A)\right)^{-1} (B-A) dt\right)\right| \nonumber \\
        &\leq& \frac{1}{\lambda_{\min}(A)-\|B-A\|_{\text{op}}} \|B-A\|_{\text{op}} \label{lem_logdet_eq1}
    \end{eqnarray}
    Using triangle inequality, we get
    \begin{eqnarray}
        |\log \text{det}(A) - \log \text{det}(B)| &\leq& \int_{0}^{1} \left|\text{Tr}\left(\left(A + t (B -A)\right)^{-1} (B-A)\right)\right|dt \label{log_eq_1}
    \end{eqnarray}
    From Holder's inequality for trace we have that for any two matrices $X, Y$ it holds that $|\text{Tr} (XY)| \leq d \|X\|_{\text{op}} \|Y\|_{\text{op}}$. Applying that for $X = \left(A + t (B -A)\right)^{-1}$ and $Y = B-A,$ we obtain 
    \begin{eqnarray}
       \left|\text{Tr}\left(\left(A + t (B -A)\right)^{-1} (B-A)\right)\right|&\leq& d \|B-A\|_{\text{op}} \left\|\left(A + t (B -A)\right)^{-1}\right\|_{\text{op}} \label{log_eq_2}
    \end{eqnarray}
    Substituting \eqref{log_eq_2} into \eqref{log_eq_1}, we obtain
    \begin{eqnarray}
        |\log \text{det}(A) - \log \text{det}(B)| &\leq& d \|B-A\|_{\text{op}} \int_{0}^{1} \left\|\left(A + t (B -A)\right)^{-1}\right\|_{\text{op}} dt \label{log_eq_3}
    \end{eqnarray}
    We, now, bound the operator norm $\left\|\left(A + t (B -A)\right)^{-1}\right\|_{\text{op}}$. Using the fact that $\left\|\left(A + t (B -A)\right)^{-1}\right\|_{\text{op}} \leq \frac{1}{\lambda_{\min}(A + t (B -A))}$ and Weyl's inequality we have that $\lambda_{\min}(A + t (B -A)) \geq \lambda_{\min}(A) + \|B-A\|_{\text{op}}, \forall t \in [0, 1]$. Thus, we get
    \begin{eqnarray}
        \left\|\left(A + t (B -A)\right)^{-1}\right\|_{\text{op}} &\leq& \frac{1}{\lambda_{\min}(A) - \|B-A\|_{\text{op}}} \label{log_eq_4}
    \end{eqnarray}
    Substituting \eqref{log_eq_4} into \eqref{log_eq_3}, we have
    \begin{eqnarray}
        |\log \text{det}(A) - \log \text{det}(B)|&\leq& d \|B-A\|_{\text{op}} \int_{0}^{1} \frac{dt}{\lambda_{\min}(A) - \|B-A\|_{\text{op}}} = \frac{d \|B-A\|_{\text{op}}}{\lambda_{\min}(A) - \|B-A\|_{\text{op}}} \nonumber
    \end{eqnarray} 
\end{proof}
\begin{lemma}\label{lemma: logdet_bound}
    If the number of samples observed from two Gaussian marginals $\mathcal{D}_{i_*}, \mathcal{D}_{i_2'}$ is at least $n > d$ and $\|\Sigma_{i_*}-\hat{\Sigma}_{i_*}\|_{\text{op}} < \lambda_{\min}(\Sigma_{i_*}), \|\Sigma_{i_*}-\hat{\Sigma}_{i_*}\|_{\text{op}} < \lambda_{\min}(\Sigma_{i_2'})$, the following holds 
    \begin{eqnarray}
        |\log \text{det}(\Sigma_{i_*}) - \log \text{det}(\hat{\Sigma}_{i_*})| + |\log \text{det}(\Sigma_{i_2'}) - \log \text{det}(\hat{\Sigma}_{i_2'})| &\leq& \mathcal{\Tilde{O}}\left(\frac{d^{3/2}}{n^{1/2}}\right)\nonumber
    \end{eqnarray}
\end{lemma}
\begin{proof}
    From Lemma~\ref{lemma: basic_matrix_inequality}, for $A = \Sigma_{i_*}, B = \hat{\Sigma}_{i_*}$ and $\epsilon = \|\Sigma_{i_*}-\hat{\Sigma}_{i_*}\|_{\text{op}}$, we get
    \begin{eqnarray}
        |\log \text{det}(\Sigma_{i_*}) - \log \text{det}(\hat{\Sigma}_{i_*})| &\leq& \frac{d \epsilon}{\lambda_{\min}(\Sigma_{i_*}) (1 - \epsilon/\lambda_{\min}(\Sigma_{i_*}))}\label{lem_logdet_eq2}
    \end{eqnarray}
    Given that $\epsilon = \|\Sigma_{i_*}-A\|_{\text{op}} = \mathcal{\Tilde{O}}\left(\sqrt{\frac{d}{n}}\right)$, we get for $\epsilon < \lambda_{\min}(\Sigma_{i_*})$ from the Taylor expansion of the function $f(x) = \frac{1}{1 - x} = \sum_{n=0}^{+\infty} x^n$ that
    \begin{eqnarray}
        \frac{1}{1 - \frac{\epsilon}{\lambda_{\min}(\Sigma_{i_*})}} \leq \sum_{n=0}^{+\infty} \left(\frac{\epsilon}{\lambda_{\min}(\Sigma_{i_*})}\right)^n = \mathcal{\Tilde{O}}\left(1\right) \label{eq_2a}
    \end{eqnarray}
    Substituting \eqref{eq_2a} into \eqref{lem_logdet_eq2}, we get that
    \begin{eqnarray}
        |\log \text{det}(\Sigma_{i_*}) - \log \text{det}(\hat{\Sigma}_{i_*})| &\leq& \mathcal{\Tilde{O}}\left(d\epsilon\right) = \mathcal{\Tilde{O}}\left(\frac{d^{3/2}}{n^{1/2}}\right) \label{log_eq_last1}
    \end{eqnarray}
    where at the last step we have used the fact that $\epsilon = \|\Sigma_{i_*}-A\|_{\text{op}} = \mathcal{\Tilde{O}}\left(\sqrt{\frac{d}{n}}\right)$.
    By applying the above steps similarly for $A = \Sigma_{i_2'}$ and $B = \hat{\Sigma}_{i_2'}$ and using the fact that $\|\Sigma_{i_2'} - \hat{\Sigma}_{i_2'}\|_{\text{op}} = \mathcal{\Tilde{O}}\left(\sqrt{\frac{d}{n}}\right)$, we obtain
    \begin{eqnarray}
      |\log \text{det}(\Sigma_{i_2'}) - \log \text{det}(\hat{\Sigma}_{i_2'})| &\leq& \mathcal{\Tilde{O}}\left(\frac{d^{3/2}}{n^{1/2}}\right) \label{log_eq_last2}  
    \end{eqnarray}
    Adding inequalities \eqref{log_eq_last1}, \eqref{log_eq_last2}, we get the final result
    \begin{eqnarray}
        |\log \text{det}(\Sigma_{i_*}) - \log \text{det}(\hat{\Sigma}_{i_*})| + |\log \text{det}(\Sigma_{i_2'}) - \log \text{det}(\hat{\Sigma}_{i_2'})| &\leq& \mathcal{\Tilde{O}}\left(\frac{d^{3/2}}{n^{1/2}}\right)\nonumber
    \end{eqnarray}
    \end{proof}

\begin{lemma}\label{lemma: logpi_bound}
    If the number of samples observed from two Gaussian marginals $\mathcal{D}_{i_*}, \mathcal{D}_{i_2'}$ is at least $n > d$ and the true priors satisfy $\pi_{i_*} ,\pi_{i_2'} > \pi_{\min} > 0$, then it holds that
    \begin{eqnarray}
        |\log\left(\hat{\pi}_{i_2'}\right) - \log\left(\pi_{i_2'}\right)| + |\log\left(\hat{\pi}_{i_*}\right) - \log\left(\pi_{i_*}\right)| &\leq& \mathcal{\Tilde{O}}\left(\sqrt{\frac{d}{n}}\right) \nonumber
    \end{eqnarray}
\end{lemma}
\begin{proof}
For a Gaussian component with true prior $\pi$ and empirical prior $\widehat{\pi}$, we write
\[
\log\left(\frac{\widehat{\pi}}{\pi}\right)
=
\log\left(1+\frac{\widehat{\pi}-\pi}{\pi}\right).
\]
Assume that $\pi$ is bounded away from zero, namely
$\pi\geq \pi_{\min}>0$. Since
\[
|\widehat{\pi}-\pi|
\leq
\mathcal{\Tilde{O}}\left(\sqrt{\frac dn}\right),
\]
we have, for $n$ sufficiently large,
\[
\left|\frac{\widehat{\pi}-\pi}{\pi}\right|
\leq
\frac{|\widehat{\pi}-\pi|}{\pi_{\min}}
<1.
\]
Therefore, using the Taylor expansion of $\log(1+x)$
\[
\log(1+x)
=
\sum_{i=1}^{+\infty}
\frac{(-1)^{i-1}}{i}x^i,
\qquad |x|<1,
\]
with
\[
x=\frac{\widehat{\pi}-\pi}{\pi},
\]
we obtain
\[
\left|
\log\left(\frac{\widehat{\pi}}{\pi}\right)
\right|
=
\left|
\log\left(1+\frac{\widehat{\pi}-\pi}{\pi}\right)
\right|
=
\mathcal{\Tilde{O}}\left(\sqrt{\frac dn}\right).
\]
Applying the aforementioned statement twice, for $\pi_{i_2'}$ and $\pi_{i_*}$, and summing we obtain the final result.
\end{proof}

\begin{lemma}\label{lemma: margin_difference}
Let the functions 
\begin{eqnarray}
   m(x) = d_M^2(x, \mu_{i_2'}) - d_M^2(x, \mu_{i_*}) +\log\left(\frac{\text{det}\left(\Sigma_{i_2'}\right)}{\text{det}\left(\Sigma_{i_*}\right)}\right) -2\log\left(\frac{\pi_{i_2'}}{\pi_{i_*}}\right), \nonumber \\
   \hat{m}(x) = d_M^2(x, \hat{\mu}_{i_2'}) - d_M^2(x, \hat{\mu}_{i_*}) +\log\left(\frac{\text{det}\left(\hat{\Sigma}_{i_2'}\right)}{ \text{det}\left(\hat{\Sigma}_{i_*}\right)}\right) -2\log\left(\frac{\hat{\pi_{i_2'}}}{\hat{\pi_{i_*}}}\right) \nonumber.
\end{eqnarray}
where $\hat{\mu}_j, \hat{\Sigma}_j$ is the empirical mean and covariance of the Gaussian marginal $\mathcal{D}_j, \forall j\in [K]$. 
If the number of samples observed from each marginal $\mathcal{D}_j, \forall j \in [K]$ is $n > d,$ then we have that
\begin{eqnarray}
    |\hat{m}(x) - m(x)| &\leq& \mathcal{\Tilde{O}}\left(\sqrt{\frac{d}{n}}\right)\nonumber  
\end{eqnarray}    
\end{lemma}
\begin{proof}
    Using the expressions of margins and applying the triangle inequality, we have that 
\begin{eqnarray}
    |\hat{m}(x) - m(x)| &=& \Big|d_M^2(x, \hat{\mu}_{i_2'}) - d_M^2(x, \mu_{i_2'}) - d_M^2(x, \hat{\mu}_{i_*}) + d_M^2(x, \mu_{i_*}) \nonumber \\
    &\quad& + \log\left(\frac{\text{det}\left(\hat{\Sigma}_{i_2'}\right)}{ \text{det}\left(\hat{\Sigma}_{i_*}\right)}\right) - \log\left(\frac{\text{det}\left(\Sigma_{i_2'}\right)}{\text{det}\left(\Sigma_{i_*}\right)}\right) - 2\log\left(\frac{\hat{\pi_{i_2'}}}{\hat{\pi_{i_*}}}\right) + 2\log\left(\frac{\pi_{i_2'}}{\pi_{i_*}}\right)\Big| \nonumber\\
    &\leq& \underbrace{\left|d_M^2(x, \hat{\mu}_{i_2'}) - d_M^2(x, \mu_{i_2'})\right| +\left|d_M^2(x, \hat{\mu}_{i_*}) - d_M^2(x, \mu_{i_*})\right|}_{T_1} \nonumber \\
    && + \underbrace{\left|\log\left(\text{det}\left(\hat{\Sigma}_{i_2'}\right)\right) - \log\left(\text{det}\left(\Sigma_{i_2'}\right)\right)\right| + \left|\log\left(\text{det}\left(\hat{\Sigma}_{i_*}\right)\right)- \log\left(\text{det}\left(\Sigma_{i_*}\right)\right)\right|}_{T_2} \nonumber \\
    && +2 [\underbrace{\left|\log\left(\hat{\pi_{i_2'}}\right) - \log\left(\pi_{i_2'}\right)\right|+ \left|\log\left(\hat{\pi_{i_*}}\right) - \log\left(\pi_{i_*}\right)\right|}_{T_3}] \label{lemma_margin_eq1}
\end{eqnarray}
We, next, bound the three terms $T_1, T_2, T_3$. Applying Lemma~\ref{lemma: mahalanobis_dist_bound} for the Gaussian marginals $\mathcal{D}_{i_2}, \mathcal{D}_{i_*}$, we have that
\begin{eqnarray}
    \left|d_M^2(x, \hat{\mu}_{i_2'}) - d_M^2(x, \mu_{i_2'})\right| &\leq& \mathcal{\Tilde{O}}\left(\sqrt{\frac{d}{n}}\right) \label{lemma_margin_eq2} \\
    \left|d_M^2(x, \hat{\mu}_{i_*}) - d_M^2(x, \mu_{i_*})\right| &\leq& \mathcal{\Tilde{O}}\left(\sqrt{\frac{d}{n}}\right) \label{lemma_margin_eq3}
\end{eqnarray}
Adding inequalities \eqref{lemma_margin_eq2}, \eqref{lemma_margin_eq3}, we get that
\begin{eqnarray}
    T_1 &\leq& \mathcal{\Tilde{O}}\left(\sqrt{\frac{d}{n}}\right) \label{lemma_margin_eq4}
\end{eqnarray}
We, next, bound the term $T_2$ by using Lemma \ref{lemma: logdet_bound}
\begin{eqnarray}
    T_2 &\leq& \mathcal{\Tilde{O}}\left(\frac{d^{3/2}}{n^{1/2}}\right) \label{lemma_margin_eq5}
\end{eqnarray}
For the term $T_3$, we use Lemma \ref{lemma: logpi_bound} and get 
\begin{eqnarray}
    T_3 &\leq& \mathcal{\Tilde{O}}\left(\sqrt{\frac{d}{n}}\right) \label{lemma_margin_eq6} 
\end{eqnarray}
Substituting inequalities \eqref{lemma_margin_eq4}, \eqref{lemma_margin_eq5}, \eqref{lemma_margin_eq6} into \eqref{lemma_margin_eq1}, we obtain
\begin{eqnarray}
   |\hat{m}(x) - m(x)| \leq \mathcal{\Tilde{O}}\left(\sqrt{\frac{d}{n}}\right) + \mathcal{\Tilde{O}}\left(\frac{d^{3/2}}{n^{3/2}}\right) + 2 \mathcal{\Tilde{O}}\left(\sqrt{\frac{d}{n}}\right)\leq \mathcal{\Tilde{O}}\left(\sqrt{\frac{d}{n}}\right) \nonumber
\end{eqnarray}
\end{proof}

\begin{lemma}\label{lemma: lambda_difference}
If the number of samples observed from each marginal $\mathcal{D}_j, \forall j \in [K],$ is at least $n  \geq d,$ then we have that
\begin{eqnarray}
    \left|\lambda_{\min}-\hat\lambda_{\min}\right| &\leq& \mathcal{\Tilde{O}}\left(\sqrt{\frac{d}{n}}\right) \nonumber 
\end{eqnarray} 
where $\lambda_{\min}$ is the minimum over all eigenvalues of $W_i = \Sigma^{-1}_i - \Sigma^{-1}_{i_*}, \forall i \neq i_*$. 
\end{lemma}
\begin{proof}
    Using the fact that for any two real sequences $a_i, b_i$ it holds that $|\min_i a_i - \min_i b_i| \leq \max_i |a_i - b_i|, $ we have that
    \begin{eqnarray}
        \bigl|\lambda_{\min}-\hat\lambda_{\min}\bigr| &=& \Bigl|\min_{i\in[K]}\;\lambda_{\min}\bigl(W_i\bigr)\;-\;\min_{i\in[K]}\;\lambda_{\min}\bigl(\hat W_i\bigr)\Bigr|\nonumber \\
        &\leq&\;\max_{i\in[K]}\;\Bigl|\lambda_{\min}\bigl(W_i\bigr)-\lambda_{\min}\bigl(\hat W_i\bigr)\Bigr|\\
        \end{eqnarray}
    Using Weyl's and triangle inequality, we get
        \begin{eqnarray}
            \bigl|\lambda_{\min}-\hat\lambda_{\min}\bigr| &\overset{}{\leq}&
        \max_i\;\bigl\|\,W_i-\hat W_i\bigr\|_{\mathrm{op}}\\
        \;&=&\;
        \max_i\;\bigl\|\,(\Sigma_i^{-1}-\Sigma_*^{-1})
                -(\,\hat\Sigma_i^{-1}-\hat\Sigma_*^{-1})\bigr\|_{\mathrm{op}}
        \\
        &\leq&
        \max_i\;\Bigl(\,
          \|\Sigma_i^{-1}-\hat\Sigma_i^{-1}\|_{\mathrm{op}}
         +\|\Sigma_*^{-1}-\hat\Sigma_*^{-1}\|_{\mathrm{op}}
        \Bigr)
        \\
        &\leq&
        \max_i\;\Bigl(\,
          \|\Sigma_i^{-1}\|_{\mathrm{op}}^2\,\|\Sigma_i-\hat\Sigma_i\|_{\mathrm{op}}
         +\|\Sigma_*^{-1}\|_{\mathrm{op}}^2\,\|\Sigma_*-\hat\Sigma_*\|_{\mathrm{op}}
        \Bigr)
        \\
        &=\;&
        \mathcal{\Tilde{O}}\!\Bigl(\sqrt{d/n}\Bigr)
        \end{eqnarray}
        where at the last step we used the fact that $\|\Sigma_i - \hat\Sigma_i\|_{\mathrm{op}} = \mathcal{\Tilde{O}}\left(\sqrt{\frac{d}{n}}\right)$.
\end{proof}

\begin{lemma}\label{lemma: lambda_inverse_difference}
    If the number of samples observed from each marginal $\mathcal{D}_j, \forall j \in [K],$ is at least $n > d$ and $|\hat{\lambda}_{\min} - \lambda_{\min}| < \lambda_{\min},$ then it holds that
    \begin{eqnarray}
        \left|\frac{1}{\hat{\lambda}_{\min}} - \frac{1}{\lambda_{\min}}\right| &\leq& \mathcal{\Tilde{O}}\left(\sqrt{\frac{d}{n}}\right) \nonumber
    \end{eqnarray}
    where $\lambda_{\min}$ is the minimum over all eigenvalues of $W_i = \Sigma^{-1}_i - \Sigma^{-1}_{i_*}, \forall i \neq i_*$.
\end{lemma}
\begin{proof}
    Let $\epsilon_{\lambda} = \hat{\lambda}_{\min} - \lambda_{\min}$ denote the error between the estimate and the true minimum eigenvalue. We have that
    \begin{eqnarray}
        \left|\frac{1}{\hat{\lambda}_{\min}} - \frac{1}{\lambda_{\min}}\right| &=& \left|\frac{\lambda_{\min} - \hat{\lambda}_{\min}}{\lambda_{\min}\hat{\lambda}_{\min}}\right| 
        = \left|\frac{\epsilon_{\lambda}}{\lambda_{\min}\hat{\lambda}_{\min}}\right| 
        = \frac{\left|\epsilon_{\lambda}\right|}{\lambda_{\min}\left|\hat{\lambda}_{\min} - \lambda_{\min} + \lambda_{\min}\right|}\nonumber\\
        &\leq& \frac{\left|\epsilon_{\lambda}\right|}{\lambda_{\min}\left(\lambda_{\min} - |\epsilon_{\lambda}\right|)} \nonumber\\
        &\leq& \frac{\left|\epsilon_{\lambda}\right|}{\lambda^2_{\min}(1- \frac{\epsilon_{\lambda}}{\lambda_{\min}})} \label{eq_1_eps_lambda}
    \end{eqnarray}
    In order to bound the term $\frac{1}{1 - \frac{\epsilon_{\lambda}}{\lambda_{\min}}}$, we use the Taylor expansion of the function $f(x) = \frac{1}{1-x} = \sum_{n=0}^{+\infty} x^n$ and get for $\epsilon_{\lambda} < \lambda_{\min}$ that
    \begin{eqnarray}
        \frac{1}{1 - \frac{\epsilon_{\lambda}}{\lambda_{\min}}} \leq \sum_{n=0}^{+\infty} \left(\frac{\epsilon_\lambda}{\lambda_{\min}}\right)^n = \mathcal{\Tilde{O}}\left(1\right)\label{eq_2_eps_lambda}
    \end{eqnarray}
    Substituting inequality \eqref{eq_2_eps_lambda} into \eqref{eq_1_eps_lambda} and using from Lemma~\ref{lemma: lambda_difference} the fact that $\epsilon_\lambda = \mathcal{\Tilde{O}}\left(\sqrt{\frac{d}{n}}\right)$, we obtain
    \begin{eqnarray}
        \left|\frac{1}{\hat{\lambda}_{\min}} - \frac{1}{\lambda_{\min}}\right| &\leq& \mathcal{\Tilde{O}}\left(\sqrt{\frac{d}{n}}\right) \nonumber
    \end{eqnarray}
\end{proof}
\begin{lemma}\label{lemma: c_M_dif}
For a sample $(x, y),$ let 
\begin{eqnarray}
   c_M(x) = \max\limits_{i \neq y} \|\Sigma_y^{-1} (x - \mu_y)^T - \Sigma_i^{-1} (x-\mu_i)^T\|_2 \nonumber \\
   \hat{c}_M(x) = \max\limits_{i \neq y} \|\hat{\Sigma}_y^{-1} (x - \hat{\mu}_y)^T - \hat{\Sigma}_i^{-1} (x-\hat{\mu}_i)^T\|_2 \nonumber
\end{eqnarray}
and $\hat{\mu}_j, \hat{\Sigma}_j$ the empirical mean and covariance of the Gaussian marginal $\mathcal{D}_j, \forall j\in [K]$. 
If the number of samples observed from each marginal $\mathcal{D}_j, \forall j \in [K]$ is $n > d$ and $\|\Sigma_j - \hat{\Sigma}_j\|_{\text{op}} < \frac{1}{\|\Sigma_i^{-1} \|_{\text{op}}},$ then we have that
\begin{eqnarray}
    |c_M - \hat{c}_M| &\leq& \mathcal{\Tilde{O}}\left(\sqrt{\frac{d}{n}}\right) \nonumber 
\end{eqnarray}    
\end{lemma}
\begin{proof}
    Let $\alpha_i =  \Sigma_i^{-1} (x-\mu_i)^T, \forall i \in [K]$ and $\hat{\alpha}_i =  \hat{\Sigma}_i^{-1} (x-\hat{\mu}_i)^T, \forall i \in [K]$. Then, we have that 
    \begin{eqnarray}
        c_M(x) = \max\limits_{i\neq y} \|\alpha_y - \alpha_i\|_2, \quad \text{ and} \quad \hat{c}_M(x) = \max\limits_{i \neq y} \|\hat{\alpha}_y - \hat{\alpha}_i\|_2 \nonumber
    \end{eqnarray}
    Using the fact that for any two arbitrary sequences of real numbers $b_i, c_i$ it holds that
    \begin{eqnarray}
        \left|\max_ib_i - \max_i c_i\right| &\leq& \max_{i} |b_i - c_i| \nonumber
    \end{eqnarray}
    we have for $b_i = \|\alpha_y - \alpha_i\|_2$ and $c_i = \|\hat{\alpha}_y - \hat{\alpha}_i\|_2$ that
    \begin{eqnarray}
        \left|c_M - \hat{c}_M\right| = \left|\max_{i\neq y} \|\alpha_y - \alpha_i\|_2 - \max_{i\neq y} \|\hat{\alpha}_y - \hat{\alpha}_i\|_2\right| \leq \max_{i \neq y} \left|\|\alpha_y - \alpha_i\|_2 - \|\hat{\alpha}_y - \hat{\alpha}_i\|_2\right| \label{eq1_c_M_dif}
    \end{eqnarray}
    Applying the triangle inequality on the norm $\|\alpha_y - \alpha_i\|_2 - \|\hat{\alpha}_y - \hat{\alpha}_i\|_2,$ we have that $\|\alpha_y - \alpha_i\|_2 - \|\hat{\alpha}_y - \hat{\alpha}_i\|_2 \leq \|\alpha_y - \hat{{\alpha}}_y + \hat{\alpha}_i - \alpha_i\|_2$ we obtain
    \begin{eqnarray}
        \left|c_M - \hat{c}_M\right| \leq \max_{i \neq y} \left|\|\alpha_y - \alpha_i\|_2 - \|\hat{\alpha}_y - \hat{\alpha}_i\|_2\right| \leq \max_{i \neq y} \|\alpha_y - \hat{\alpha}_y\|_2 + \|\alpha_i - \hat{\alpha}_i\|_2 \label{eq1b_c_M} 
    \end{eqnarray}
    We, next, bound the error on the terms $\|\alpha_y - \hat{\alpha}_y\|_2$ and $\alpha_i - \hat{\alpha}_i$. It holds that
    \begin{eqnarray}
        \alpha_i - \hat{\alpha}_i &=& \Sigma_i^{-1} (x-\mu_i)^T - \hat{\Sigma}_i^{-1} (x-\hat{\mu}_i)^T \nonumber\\
        &=& \Sigma_i^{-1} (x-\mu_i)^T - \Sigma_i^{-1} (x-\hat{\mu}_i)^T + \Sigma_i^{-1} (x-\hat{\mu}_i)^T -\hat{\Sigma}_i^{-1} (x-\hat{\mu}_i)^T \nonumber\\
        &=& \Sigma_i^{-1} (\hat{\mu}_i - \mu_i)^T + (\Sigma_i^{-1} - \hat{\Sigma}_i^{-1}) (x-\hat{\mu}_i)^T \nonumber\\
        &=& \Sigma_i^{-1} (\hat{\mu}_i - \mu_i)^T + (\Sigma_i^{-1} - \hat{\Sigma}_i^{-1}) (\mu_i-\hat{\mu}_i)^T + (\Sigma_i^{-1} - \hat{\Sigma}_i^{-1}) (x-\mu_i)^T \nonumber
    \end{eqnarray}
    Taking the norm and applying the triangle inequality, we get
    \begin{eqnarray}
        \|\alpha_i - \hat{\alpha}_i\|_2 &\leq& \|\Sigma_i^{-1} (\hat{\mu}_i - \mu_i)^T\|_2 + \|(\Sigma_i^{-1} - \hat{\Sigma}_i^{-1}) (\mu_i-\hat{\mu}_i)^T\|_2 + \|(\Sigma_i^{-1} - \hat{\Sigma}_i^{-1}) (x-\mu_i)^T\|_2 \nonumber\\
        &\leq& \|\Sigma_i^{-1}\|_{\text{op}} \|\hat{\mu}_i - \mu_i\|_2 + \|\Sigma_i^{-1} - \hat{\Sigma}_i^{-1}\|_{\text{op}} \|\mu_i-\hat{\mu}_i\|_2 + \|\Sigma_i^{-1} - \hat{\Sigma}_i^{-1}\|_{\text{op}} \|x - \mu_i\|_2 \quad \quad \label{eq2_lemm_c}
    \end{eqnarray}
    In order to bound inequality \eqref{eq2_lemm_c}, we need upper bounds on $\|\mu_i-\hat{\mu}_i\|_2$ and $\|\Sigma_i^{-1} - \hat{\Sigma}_i^{-1}\|_{\text{op}}$. We have that $\|\mu_i-\hat{\mu}_i\|_2 = \mathcal{\Tilde{O}}\left(\sqrt{\frac{d}{n}}\right),$ where $n$ is the minimum number of samples observed from each marginal $\mathcal{D}_i, \forall i\in[K]$. 
    For the term $\|\Sigma_i^{-1} - \hat{\Sigma}_i^{-1}\|_{\text{op}},$ we get for $\|\Sigma_i - \hat{\Sigma}_i\|_{\text{op}} < \frac{1}{\|\Sigma_i^{-1} \|_{\text{op}}}$ that the following inequality holds 
    \begin{eqnarray}
         \|\Sigma_i^{-1} - \hat{\Sigma}_i^{-1}\|_{\text{op}} \leq \|\Sigma_i^{-1} \|_{\text{op}}^2 \|\Sigma_i - \hat{\Sigma}_i\|_{\text{op}}\nonumber
    \end{eqnarray}
    Using the fact that $\|\Sigma_i - \hat{\Sigma}_i\|_{\text{op}} = \mathcal{\Tilde{O}}\left(\sqrt{\frac{d}{n}}\right)$ we obtain
    \begin{eqnarray}
        \|\Sigma_i^{-1} - \hat{\Sigma}_i^{-1}\|_{\text{op}} \leq \mathcal{\Tilde{O}}\left(\sqrt{\frac{d}{n}}\right) \label{eq3_lemm_c}
    \end{eqnarray}
    Substituting inequality \eqref{eq3_lemm_c} and the fact that $\|\mu_i-\hat{\mu}_i\|_2 = \mathcal{\Tilde{O}}\left(\sqrt{\frac{d}{n}}\right)$ into \eqref{eq2_lemm_c}
    \begin{eqnarray}
        \|\alpha_i - \hat{\alpha}_i\|_2 &\leq& \|\Sigma_i^{-1}\|_{\text{op}} \mathcal{\Tilde{O}}\left(\sqrt{\frac{d}{n}}\right) + \mathcal{\Tilde{O}}\left(\frac{d}{n}\right) + \mathcal{\Tilde{O}}\left(\sqrt{\frac{d}{n}}\right) \leq \mathcal{\Tilde{O}}\left(\sqrt{\frac{d}{n}}\right) \label{eq4_lemm_c}
    \end{eqnarray}
    Similarly, for $i = y$ we have that
    \begin{eqnarray}
        \|\alpha_y - \hat{\alpha}_y\|_2 &\leq& \mathcal{\Tilde{O}}\left(\sqrt{\frac{d}{n}}\right) \label{eq5_lemm_c}
    \end{eqnarray}
    From \eqref{eq1b_c_M}, \eqref{eq4_lemm_c}, \eqref{eq5_lemm_c}, we get that
    \begin{eqnarray}
      \left|c_M - \hat{c}_M\right| &\leq&  \|\alpha_y - \hat{\alpha}_y\|_2  + \max_{i \neq y} \|\alpha_i - \hat{\alpha}_i\|_2 \leq \mathcal{\Tilde{O}}\left(\sqrt{\frac{d}{n}}\right) \nonumber
    \end{eqnarray}
\end{proof}

\begin{lemma}\label{lemma: A_dif}
Let $A = \sqrt{c_M^2 - m(x) \lambda_{\min}}$ and $\hat{A} = \sqrt{\hat{c}_M^2 - \hat{m}(x) \hat{\lambda}_{\min}}$. 
If the minimum number of samples observed from each marginal $\mathcal{D}_j, \forall j \in [K]$ is $n > d$ and $\|\Sigma_j - \hat{\Sigma}_j\|_{\text{op}} < \frac{1}{\|\Sigma_j^{-1} \|_{\text{op}}},$ then we have that
\begin{eqnarray}
    |A - \hat{A}| &\leq& \mathcal{\Tilde{O}}\left(\frac{d^{1/4}}{n^{1/4}}\right)\nonumber  
\end{eqnarray}    
\end{lemma}
\begin{proof}
    Using the Holder's inequality for $f(x) = \sqrt{x},$ we have that
    \begin{eqnarray}
        \left|A - \hat{A}\right| &=& \left|\sqrt{c_M^2 - m(x) \lambda_{\min}} - \sqrt{\hat{c}_M^2 - \hat{m}(x) \hat{\lambda}_{\min}}\right| \nonumber \\
        &\leq& \sqrt{\left|c_M^2 - \hat{c}_M^2 - m(x) \lambda_{\min} + \hat{m}(x) \hat{\lambda}_{\min}\right|} \label{eq_1_A_dif}
    \end{eqnarray}
    We, now, express the quantity under the root with respect to $\epsilon_c = \hat{c}_M - c_M$ and $\epsilon_{\lambda} = \hat{\lambda}_{min} - \lambda_{min},$ as follows 
    \begin{eqnarray}
        c_M^2 - \hat{c}_M^2 - m(x) \lambda_{\min} + \hat{m}(x) \hat{\lambda}_{\min} &\leq& (c_M - \hat{c}_M) (c_M + \hat{c}_M) - m(x) \lambda_{\min} + \hat{m}(x) \hat{\lambda}_{\min} \nonumber \\
        &=& \epsilon_c (2c_M + \epsilon) - m(x) (\lambda_{\min} - \hat{\lambda}_{\min}) + \hat{\lambda}_{\min} (m(x) - \hat{m}(x)) \nonumber \\
       &=& \epsilon_c (2c_M + \epsilon) - m(x) (\lambda_{\min} - \hat{\lambda}_{\min}) \nonumber \\
       && + (\hat{\lambda}_{\min}-\lambda_{\min}) (m(x) - \hat{m}(x)) + \lambda_{\min} (m(x) - \hat{m}(x)) \quad \quad \label{eq_2_A_dif}
    \end{eqnarray}
    Combining \eqref{eq_1_A_dif}, \eqref{eq_2_A_dif}, using triangle inequality and the concavity of the square root, we get
    \begin{eqnarray}
        \left|A - \hat{A}\right| &\leq& \sqrt{\epsilon_c (2c_M + \epsilon_c)} + \sqrt{m(x) \left|\lambda_{\min} - \hat{\lambda}_{\min}\right|} + \sqrt{\left|(\hat{\lambda}_{\min}-\lambda_{\min}) (m(x) - \hat{m}(x))\right|} \nonumber \\ 
        && + \sqrt{\left|\lambda_{\min} (m(x) - \hat{m}(x))\right|} \nonumber \\
        &\leq& \sqrt{\epsilon_c (2c_M + \epsilon_c)} + \sqrt{m(x)\epsilon_{\lambda}} + \sqrt{\epsilon_{\lambda} \epsilon_m} + \sqrt{\left|\lambda_{\min} \epsilon_m\right|} \label{eq_3_A_dif}
    \end{eqnarray}
    From Lemma~\ref{lemma: lambda_difference}, \ref{lemma: c_M_dif}, \ref{lemma: margin_difference}, we have that $\epsilon_m = \mathcal{\Tilde{O}}\left(\sqrt{\frac{d}{n}}\right), \epsilon_c = \mathcal{\Tilde{O}}\left(\sqrt{\frac{d}{n}}\right)$ and $\epsilon_{\lambda} = \mathcal{\Tilde{O}}\left(\sqrt{\frac{d}{n}}\right)$ and thus
    \begin{eqnarray}
        \left|A - \hat{A}\right| &\leq& \mathcal{\Tilde{O}}\left(\frac{d^{1/4}}{n^{1/4}}\right) + \mathcal{\Tilde{O}}\left(\frac{d^{1/4}}{n^{1/4}}\right) + \mathcal{\Tilde{O}}\left(\sqrt{\frac{d}{n}}\right)+ \mathcal{\Tilde{O}}\left(\frac{d^{1/4}}{n^{1/4}}\right)= \mathcal{\Tilde{O}}\left(\frac{d^{1/4}}{n^{1/4}}\right) \nonumber
    \end{eqnarray}
\end{proof}

\begin{lemma}\label{lemma: case1}
    For an input sample $(x, y),$ with $|\hat{c}_M - c_M| \leq c_M$ and number of observed samples from each marginal $\mathcal{D}_j, \forall j \in [K]$ at least $n > d,$ it holds that
    \begin{eqnarray}
        \left|\frac{m(x)}{2c_M} - \frac{\hat{m}(x)}{2\hat{c}_M}\right| &\leq& \mathcal{\Tilde{O}}\left(\frac{d^{1/2}}{n^{1/2}}\right) \nonumber
    \end{eqnarray}
\end{lemma}
\begin{proof}
    From triangle inequality, we have that
    \begin{eqnarray}
        \left|\frac{m(x)}{2c_M} - \frac{\hat{m}(x)}{2\hat{c}_M}\right| &=& \left|\frac{m(x)}{2c_M} - \frac{\hat{m}(x)}{2c_M} + \frac{\hat{m}(x)}{2c_M} - \frac{\hat{m}(x)}{2\hat{c}_M}\right|\nonumber\\
        &\leq& \frac{1}{2c_M} \left|m(x)-\hat{m}(x)\right| + \frac{\hat{m}(x)}{2} \left|\frac{1}{c_M} - \frac{1}{\hat{c}_M}\right| \nonumber\\
        &\leq& \frac{\left|m(x)-\hat{m}(x)\right|}{2c_M} + \left(\frac{m(x)}{2} + \frac{|m(x)-\hat{m}(x)|}{2}\right)\left|\frac{1}{c_M} - \frac{1}{\hat{c}_M}\right| \nonumber\\
        &=& \frac{\left|m(x)-\hat{m}(x)\right|}{2c_M} + \left(\frac{m(x)}{2c_M} + \frac{|m(x)-\hat{m}(x)|}{2c_M}\right)\left|\frac{\hat{c}_M - c_M}{\hat{c}_M}\right| \label{eq_0_case_1}
    \end{eqnarray}
    We, next, bound the terms $\left|m(x)-\hat{m}(x)\right|$ and $\left|\frac{\hat{c}_M - c_M}{\hat{c}_M}\right|$. From Lemma~\ref{lemma: margin_difference}, we have that 
    \begin{eqnarray}
        |\hat{m}(x) - m(x)| &\leq& \mathcal{\Tilde{O}}\left(\frac{d^{1/2}}{n^{1/2}}\right) \label{eq_1_case1}
    \end{eqnarray}
    From Lemma~\ref{lemma: c_M_dif}, we have that $|\hat{c}_M - c_M| \leq \epsilon_{c}$ with $\epsilon_c = \mathcal{\Tilde{O}}\left(\frac{d^{1/2}}{n^{1/2}}\right)$ and assuming that $\epsilon_c < c_M,$ we have that 
    \begin{eqnarray}
        \left|\frac{\hat{c}_M - c_M}{\hat{c}_M}\right| \leq \frac{\epsilon_c}{\left|\hat{c}_M\right|} \leq \frac{\epsilon_c}{c_M - \left|\hat{c}_M-c_M\right|} 
        \leq  \frac{\epsilon_c}{c_M - \epsilon_c} = \frac{\epsilon_c}{c_M(1 - \epsilon_c / c_M)}\label{eq_2_case1}
    \end{eqnarray}
    In order to bound the term $\frac{1}{1 - \frac{\epsilon_c}{c_M}}$, we use the Taylor expansion of the function $f(x) = \frac{1}{1 - x} = \sum_{n=0}^{+\infty} x^n$ and get for $\epsilon_c < c_M$ that
    \begin{eqnarray}
        \frac{1}{1 - \frac{\epsilon_c}{c_M}} \leq \sum_{n=0}^{+\infty} \left(\frac{\epsilon_c}{c_M}\right)^n = \mathcal{\Tilde{O}}\left(1\right) \label{eq_3_case1}
    \end{eqnarray}
    Substituting \eqref{eq_3_case1} into \eqref{eq_2_case1}, we get that
    \begin{eqnarray}
       \left|\frac{\hat{c}_M - c_M}{\hat{c}_M}\right| \leq \mathcal{\Tilde{O}}\left(\epsilon_c \right) \leq \mathcal{\Tilde{O}}\left(\frac{d^{1/2}}{n^{1/2}}\right) \label{eq_4_case1}
    \end{eqnarray}
    Combining \eqref{eq_1_case1}, \eqref{eq_4_case1} with \eqref{eq_0_case_1}, we obtain
    \begin{eqnarray}
        \left|\frac{m(x)}{2c_M} - \frac{\hat{m}(x)}{2\hat{c}_M}\right| &\leq& \mathcal{\Tilde{O}}\left(\frac{d^{1/2}}{n^{1/2}}\right) + \left[\mathcal{\Tilde{O}}\left(1\right) + \mathcal{\Tilde{O}}\left(\frac{d^{1/2}}{n^{1/2}}\right)\right]\mathcal{\Tilde{O}}\left(\frac{d^{1/2}}{n^{1/2}}\right) \leq \mathcal{\Tilde{O}}\left(\frac{d^{1/2}}{n^{1/2}}\right) 
    \end{eqnarray}
\end{proof}
\begin{lemma}\label{lemma: cross-term}
    For an input sample $(x, y)$ with $ |\hat{\lambda}_{\min} - \lambda_{\min}| < \lambda_{\min}, \|\Sigma_j - \hat{\Sigma}_j\|_{\text{op}} < \frac{1}{\|\Sigma_j^{-1} \|_{\text{op}}},$ and number of observed samples from each marginal $\mathcal{D}_j, \forall j \in [K]$ at least $n > d,$ then it holds that
    \begin{eqnarray}
        \left|\frac{\hat{c}_M - \sqrt{\hat{c}_M^2 - \hat{m}(x) \hat{\lambda}_{min}}}{\hat{\lambda}_{min}} - \frac{\hat{m}(x)}{2\hat{c}_M}\right| &\leq& \mathcal{\Tilde{O}}\left(\sqrt{\frac{d}{n}}\right)\nonumber
    \end{eqnarray}
\end{lemma}
\begin{proof}
    We use the Taylor expansion of the function $f(x) = \sqrt{\hat{c}_M^2 - x}=\hat{c}_M \sum_{n=0}^\infty \binom{1/2}{n}(-1)^n \left(\frac{x}{\hat c_M^2}\right)^n$ and get
    \begin{eqnarray}
        \sqrt{\hat{c}_M^2 - \hat{m}(x) \hat{\lambda}_{min}} &=& \hat{c}_M \sum_{n=0}^\infty \binom{1/2}{n}(-1)^n \left(\frac{\hat{m}(x) \hat{\lambda}_{min}}{\hat c_M^2}\right)^n \nonumber\\
        \Rightarrow \left|\frac{\hat{c}_M - \sqrt{\hat{c}_M^2 - \hat{m}(x) \hat{\lambda}_{min}}}{\hat{\lambda}_{min}} - \frac{\hat{m}(x)}{2\hat{c}_M}\right| &\leq& \mathcal{\Tilde{O}}\left(\frac{\hat{m}^2(x)\hat{\lambda}_{\min}}{8\hat{c}_M^3}\right) \label{eq_2_case2}
    \end{eqnarray}
    Let $\epsilon_{m} = \hat{m}(x) - m(x), \epsilon_c = \hat{c}_M - c_M$ and $\epsilon_{\lambda} = \hat{\lambda}_{min} - \lambda_{min}$. Then, we have that
    \begin{eqnarray}
        \frac{\hat{m}^2(x)\hat{\lambda}_{\min}}{8\hat{c}_M^3} = \frac{(m(x)+\epsilon_{m})^2 (\lambda_{\min}+\epsilon_{\lambda})}{8(c_M+\epsilon_c)^3} = 
        \frac{(m^2(x)+2 \epsilon_{m} m(x)+ \epsilon_m^2) (\lambda_{\min}+\epsilon_{\lambda})}{8c_M^3(1+\epsilon_c/c_M)^3} \label{eq_3_case2}
    \end{eqnarray}
    Using the Taylor expansion of $f(x) = (1 + \frac{1}{x})^{-3}=\sum_{n=0}^{\infty} (-1)^n\frac{(n+2)(n+1)}{2}\,x^{-n}$, we have for $\epsilon_c < c_M$ that $(1+\epsilon_c/c_M)^{-3} = 1 - 3\frac{\epsilon_c}{c_M} + 6 (\frac{\epsilon_c}{c_M})^2 + ... \leq \mathcal{\Tilde{O}}\left(1\right)$ and thus
    \begin{eqnarray}
        \frac{\hat{m}^2(x)\hat{\lambda}_{\min}}{8\hat{c}_M^3} &\leq&  \frac{(m^2(x)+2 \epsilon_{m} m(x)+ \epsilon_m^2) (\lambda_{\min}+\epsilon_{\lambda})}{8c_M}\mathcal{\Tilde{O}}\left(1\right) \nonumber \\
        &=& \mathcal{\Tilde{O}}\left(\epsilon_{\lambda} +\epsilon_{m}+ \epsilon_{m}^2 + \epsilon_{m}\epsilon_{\lambda}+ \epsilon_{m}^2\epsilon_{\lambda}\right) \nonumber \\
        &=& \mathcal{\Tilde{O}}\left(\sqrt{\frac{d}{n}}\right)\label{eq_4_case_2ab}
    \end{eqnarray}
    where we have used the fact that from Lemmas~\ref{lemma: margin_difference}, \ref{lemma: lambda_difference}, \ref{lemma: c_M_dif}, it holds  
    $\epsilon_m , \epsilon_c , \epsilon_{\lambda} = \mathcal{\Tilde{O}}\left(\sqrt{\frac{d}{n}}\right)$.
    Thus, combining \eqref{eq_2_case2}, \eqref{eq_3_case2}, \eqref{eq_4_case_2ab}, we obtain 
    \begin{eqnarray}
        \left|\frac{\hat{c}_M - \sqrt{\hat{c}_M^2 - \hat{m}(x) \hat{\lambda}_{min}}}{\hat{\lambda}_{min}} - \frac{\hat{m}(x)}{2\hat{c}_M}\right| &\leq& \mathcal{\Tilde{O}}\left(\sqrt{\frac{d}{n}}\right) \label{eq_4_case2}
    \end{eqnarray}
\end{proof}
\begin{lemma}\label{lemma: case2}
    For an input sample $(x, y)$ with $ |\hat{\lambda}_{\min} - \lambda_{\min}| < \lambda_{\min}, \|\Sigma_j - \hat{\Sigma}_j\|_{\text{op}} < \frac{1}{\|\Sigma_j^{-1} \|_{\text{op}}}, |\hat{c}_M - c_M| <c_M$ and number of observed samples from each marginal $\mathcal{D}_j, \forall j \in [K]$ at least $n > d,$ it holds that
    \begin{eqnarray}
        \left|\frac{\hat{c}_M - \sqrt{\hat{c}_M^2 - \hat{m}(x) \hat{\lambda}_{min}}}{\hat{\lambda}_{min}} - \frac{m(x)}{2c_M}\right| &\leq& \mathcal{\Tilde{O}}\left(\sqrt{\frac{d}{n}}\right) \nonumber
    \end{eqnarray}
\end{lemma}
\begin{proof}
    Applying the triangle inequality, we have that
    \begin{eqnarray}
       \left|\frac{\hat{c}_M - \sqrt{\hat{c}_M^2 - \hat{m}(x) \hat{\lambda}_{min}}}{\hat{\lambda}_{min}} - \frac{m(x)}{2c_M}\right| &\leq& \underbrace{\left|\frac{\hat{c}_M - \sqrt{\hat{c}_M^2 - \hat{m}(x) \hat{\lambda}_{min}}}{\hat{\lambda}_{min}} - \frac{\hat{m}(x)}{2\hat{c}_M}\right|}_{T_1} \nonumber \\
       && + \underbrace{\left|\frac{\hat{m}(x)}{2\hat{c}_M} - \frac{m(x)}{2c_M}\right|}_{T_2} \label{eq_1_case2}
    \end{eqnarray}
    We, next, bound the terms $T_1, T_2$ appearing on the right-hand side of \eqref{eq_1_case2}. From Lemma~\ref{lemma: cross-term}, we have for $\epsilon_c < c_M$ that
    \begin{eqnarray}
        T_1 &\leq& \mathcal{\Tilde{O}}\left(\sqrt{\frac{d}{n}}\right) \label{eq_4_case2}
    \end{eqnarray}
    From Lemma~\ref{lemma: case1}, we have that
    \begin{eqnarray}
        T_2 &\leq& \mathcal{\Tilde{O}}\left(\sqrt{\frac{d}{n}}\right) \label{eq_5_case2}
    \end{eqnarray}
    Substituting \eqref{eq_4_case2}, \eqref{eq_5_case2} to \eqref{eq_1_case2}, we get that
    \begin{eqnarray}
        \left|\frac{\hat{c}_M - \sqrt{\hat{c}_M^2 - \hat{m}(x) \hat{\lambda}_{min}}}{\hat{\lambda}_{min}} - \frac{m(x)}{2c_M}\right| &\leq& \mathcal{\Tilde{O}}\left(\sqrt{\frac{d}{n}}\right)
    \end{eqnarray}
\end{proof}

\begin{lemma}\label{lemma: case4}
    For an input sample $(x, y)$ with $\left|\hat{\lambda}_{\min} - \lambda_{\min}\right| < \lambda_{\min}, |\hat{c}_M - c_M| <c_M$ and number of observed samples from each marginal $\mathcal{D}_j, \forall j \in [K]$ at least $n > d,$ it holds that
    \begin{eqnarray}
        \left|\frac{c_M - \sqrt{c_M^2 - m(x) \lambda_{\min}}}{\lambda_{\min}} - \frac{\hat{c}_M - \sqrt{\hat{c}_M^2 - \hat{m}(x) \hat{\lambda}_{\min}}}{\hat{\lambda}_{\min}}\right| &\leq& \mathcal{\Tilde{O}}\left(\frac{d^{1/4}}{n^{1/4}}\right) \nonumber
    \end{eqnarray}
\end{lemma}
\begin{proof}
   Let $\mathcal{R}(x), \hat{\mathcal{R}}(x)$ be the true and learnt certificate of robustness from Theorem~\ref{thm: certificate_robustness2}. We have that 
\begin{eqnarray}
    |\hat{\mathcal{R}}(x) - \mathcal{R}(x)| &=& \left|\frac{c_M - \sqrt{c_M^2 - m(x) \lambda_{\min}}}{\lambda_{\min}} - \frac{\hat{c}_M - \sqrt{\hat{c}_M^2 - \hat{m}(x) \hat{\lambda}_{\min}}}{\hat{\lambda}_{\min}}\right| \nonumber \\
    &=& \Big|\frac{c_M - \sqrt{c_M^2 - m(x) \lambda_{\min}}}{\lambda_{\min}} - \frac{c_M - \sqrt{c_M^2 - m(x) \lambda_{\min}}}{\hat{\lambda}_{\min}} \nonumber \\
    && \quad \quad \quad + \frac{c_M - \sqrt{c_M^2 - m(x) \lambda_{\min}}}{\hat{\lambda}_{\min}} - \frac{\hat{c}_M - \sqrt{\hat{c}_M^2 - \hat{m}(x) \hat{\lambda}_{\min}}}{\hat{\lambda}_{\min}}\Big| \nonumber 
\end{eqnarray}
Let $A = \sqrt{c_M^2 - m(x) \lambda_{\min}}$ and $\hat{A} = \sqrt{\hat{c}_M^2 - \hat{m}(x) \hat{\lambda}_{\min}} $. By applying the triangle inequality, we get
\begin{eqnarray}
    |\hat{\mathcal{R}}(x) - \mathcal{R}(x)| &\leq& \left|\frac{c_M - A}{\lambda_{\min}} - \frac{c_M - A}{\hat{\lambda}_{\min}}\right| + \left|\frac{c_M - A}{\hat{\lambda}_{\min}} - \frac{\hat{c}_M - \hat{A}}{\hat{\lambda}_{\min}}\right| \nonumber \\
    &\leq& \left|\frac{c_M - A}{\lambda_{\min}} - \frac{c_M - A}{\hat{\lambda}_{\min}}\right| + \left|\frac{A - \hat{A}}{\hat{\lambda}_{\min}}\right| + \left|\frac{c_M - \hat{c}_M}{\hat{\lambda}_{\min}}\right| \nonumber \\
    &=& \left|c_M - A\right| \left|\frac{1}{\lambda_{\min}} - \frac{1}{\hat{\lambda}_{\min}}\right| + \left|\frac{1}{\hat{\lambda}_{\min}}\right| \left(\left|A - \hat{A}\right| + \left| c_M - \hat{c}_M\right|\right) \nonumber \\
    &=& \left|c_M - A\right| \left|\frac{1}{\lambda_{\min}} - \frac{1}{\hat{\lambda}_{\min}}\right| + \left|\frac{1}{\hat{\lambda}_{\min}} - \frac{1}{\lambda_{\min}}\right| \left(\left|A - \hat{A}\right| + \left| c_M - \hat{c}_M\right|\right) \nonumber \\
    && + \left|\frac{1}{\lambda_{\min}}\right| \left(\left|A - \hat{A}\right| + \left| c_M - \hat{c}_M\right|\right) \label{eq1_gen_bou}
\end{eqnarray}
Thus, in order to bound inequality \eqref{eq1_gen_bou} we need to bound the terms $\left|\frac{1}{\lambda_{\min}} - \frac{1}{\hat{\lambda}_{\min}}\right|, \left|A - \hat{A}\right|$ and $\left| c_M - \hat{c}_M\right|$. 
From Lemmas~\ref{lemma: lambda_inverse_difference},~\ref{lemma: c_M_dif},~\ref{lemma: A_dif}, we have that for $\epsilon_\lambda = \hat{\lambda}_{\min} - \lambda_{\min} \leq \lambda_{\min}$, $\|\Sigma_i - \hat{\Sigma}_i\|_{\text{op}} < \frac{1}{\|\Sigma_i^{-1} \|_{\text{op}}},$ it holds that
\begin{eqnarray}
   \left|\frac{1}{\lambda_{\min}} - \frac{1}{\hat{\lambda}_{\min}}\right| &\leq& \mathcal{\Tilde{O}}\left(\sqrt{\frac{d}{n}}\right)\label{eq2_tg} \\
   \left|A - \hat{A}\right| &\leq& \mathcal{\Tilde{O}}\left(\frac{d^{1/4}}{n^{1/4}}\right)\label{eq3_tg} \\
   \left|c_M - \hat{c}_M\right| &\leq& \mathcal{\Tilde{O}}\left(\sqrt{\frac{d}{n}}\right) \label{eq4_gen_b}
\end{eqnarray}
Substituting inequalities \eqref{eq2_tg}, \eqref{eq3_tg}, \eqref{eq4_gen_b} into \eqref{eq1_gen_bou}, we get
\begin{eqnarray}
    |\hat{\mathcal{R}}(x) - \mathcal{R}(x)| &\leq& \left|c_M - A\right| \left|\frac{1}{\lambda_{\min}} - \frac{1}{\hat{\lambda}_{\min}}\right| + \left|\frac{1}{\hat{\lambda}_{\min}} - \frac{1}{\lambda_{\min}}\right| \left(\left|A - \hat{A}\right| + \left| c_M - \hat{c}_M\right|\right) \nonumber \\ 
    && + \left|\frac{1}{\lambda_{\min}}\right| \left(\left|A - \hat{A}\right| + \left| c_M - \hat{c}_M\right|\right)\nonumber \\
    &\leq& \mathcal{\Tilde{O}}\left(\sqrt{\frac{d}{n}}\right)+ \mathcal{\Tilde{O}}\left(\sqrt{\frac{d}{n}}\right) \mathcal{\Tilde{O}}\left(\frac{d^{1/4}}{n^{1/4}}\right) +  \mathcal{\Tilde{O}}\left(\frac{d^{1/4}}{n^{1/4}}\right) \nonumber \\
    &\leq& \mathcal{\Tilde{O}}\left(\frac{d^{1/4}}{n^{1/4}}\right) \nonumber
\end{eqnarray}
\end{proof}

\begin{lemma}\label{lemma: case3}
    If $\left|\hat{\lambda}_{\min} - \lambda_{\min}\right|  < \lambda_{\min}, \left|\hat{c}_M - c_M\right| <c_M, \|\Sigma_j - \hat{\Sigma}_j\|_{\text{op}} < \frac{1}{\|\Sigma_j^{-1} \|_{\text{op}}},$, the following bound holds
    \begin{eqnarray}
        \left|\frac{c_M - \sqrt{c_M^2 - m(x) \lambda_{\min}}}{\lambda_{\min}} - \frac{\hat{m}(x)}{2\hat{c}_M}\right| &\leq& \mathcal{\Tilde{O}}\left(\frac{d^{1/4}}{n^{1/4}}\right) \nonumber
    \end{eqnarray}
\end{lemma}
\begin{proof}
    From triangle inequality, we have that 
    \begin{eqnarray}
      \left|\frac{c_M - \sqrt{c_M^2 - m(x) \lambda_{\min}}}{\lambda_{\min}} - \frac{\hat{m}(x)}{2\hat{c}_M}\right| &\leq& \underbrace{\left|\frac{c_M - \sqrt{c_M^2 - m(x) \lambda_{\min}}}{\lambda_{\min}} - \frac{\hat{c}_M - \sqrt{\hat{c}_M^2 - \hat{m}(x) \hat{\lambda}_{min}}}{\hat{\lambda}_{min}}\right|}_{T_1} \nonumber \\
        && + \underbrace{\left|\frac{\hat{c}_M - \sqrt{\hat{c}_M^2 - \hat{m}(x) \hat{\lambda}_{min}}}{\hat{\lambda}_{min}} - \frac{\hat{m}(x)}{2\hat{c}_M}\right|}_{T_2} \label{eq_1_case3}  
    \end{eqnarray}
    From Lemmas~\ref{lemma: cross-term}, \ref{lemma: case4}, we have for $\epsilon_c = |\hat{c}_M - c_M| < c_M$ that 
    \begin{eqnarray}
        T_1 &\leq& \mathcal{\Tilde{O}}\left(\frac{d^{1/4}}{n^{1/4}}\right) \label{eq_2_case3}\\
        T_2 &\leq& \mathcal{\Tilde{O}}\left(\frac{d^{1/2}}{n^{1/2}}\right) \label{eq_3_case3}
    \end{eqnarray}
    Substituting \eqref{eq_2_case3}, \eqref{eq_3_case3} into \eqref{eq_1_case3}, we obtain 
    \begin{eqnarray}
        \left|\frac{c_M - \sqrt{c_M^2 - m(x) \lambda_{\min}}}{\lambda_{\min}} - \frac{\hat{m}(x)}{2\hat{c}_M}\right| &\leq& \mathcal{\Tilde{O}}\left(\frac{d^{1/4}}{n^{1/4}}\right) \nonumber
    \end{eqnarray}
\end{proof}

\subsection{Proof of Theorem~\ref{thm: generalization_bound}.}
\label{app: thm: generalization_bound}
\begin{proof}
Let $\hat{\mu}_i, \hat{\Sigma}_i, \hat{\pi}_i$ be the learnt parameters and $\mu_i, \Sigma_i, \pi_i $ the true parameters of the Gaussian marginal $\mathcal{D}_i, \forall i \in [K]$. Denote with $n_i$ the number of samples observed from $\mathcal{D}_i$ and let $n=\min_{i\in[K]} n_i$ be the minimum number of samples observed from any marginal distribution. 
Using Gaussian concentration results (Theorem 6.1 \citet{Wainwright}), we have that the empirical mean $\hat{\mu}_i$ and empirical covariance $\hat{\Sigma}_i$ of each Gaussian marginal $\mathcal{D}_i, \forall i \in [K],$ satisfy 
\begin{eqnarray}
    \|\hat{\mu}_i - \mu_i\|_{\Sigma^{-1}} \leq \mathcal{\Tilde{O}}\left(\sqrt{\frac{d}{n}}\right), \quad \|\hat{\Sigma} - \Sigma_i\|_{\text{op}} \leq \mathcal{\Tilde{O}}\left(\sqrt{\frac{d}{n}}\right) \nonumber 
\end{eqnarray}
where $\mathcal{\Tilde{O}}(\cdot)$ suppresses logarithmic terms in $\delta$. 
For samples $x_1, x_2, ..., x_n \sim \text{Multinomial}(\pi_1, ..., \pi_{K})$ and $\hat{\pi}_i = \frac{1}{n} \sum_{j=1}^n \mathbb{I}\{y_j = i\},$ from the Hoeffding's bound we get that with probability at least $1 - \delta,$ it holds
\begin{eqnarray}
    \|\hat{\pi}_i - \pi_i\|_2 \leq \mathcal{\Tilde{O}}\left(\sqrt{\frac{1}{n}}\right) , \quad \forall i \in [K],\nonumber
\end{eqnarray}
We denote with $\lambda_{min}$ for brevity the minimum over all eigenvalues of the matrices $W_i = \Sigma^{-1}_j - \Sigma^{-1}_{j_*}, \forall j \neq j_*$ and with $\lambda_{\min}^{\Sigma_i}$ the minimum over all eigenvalues of the covariance matrices $\Sigma_i, \forall i \in [K]$ and $\lambda_{\min}^{W_{i}}$. \\
Given the above bounds on the distance of the estimated parameters from the true ones, we bound the learnt certificate of robustness $\hat{\mathcal{R}}(x)$ from the true certificate $\mathcal{R}(x)$. 
From \eqref{eq_cert_radius_case_a_app}, \eqref{eq_cert_radius_case_b_app} in Theorem~\ref{thm: certificate_robustness2}, depending on the sign of $\lambda_{\min}$, there are two cases for the expression of the certified radius, specifically $\mathcal{R}(x) = \frac{c_M - \sqrt{c_M^2 - m(x) \lambda_{\min}}}{\lambda_{\min}}$ or $\mathcal{R}(x) = \frac{m(x)}{2c_M}$. Similarly, based on whether $\hat{\lambda}_{min}$ is positive or negative, there are two cases for the expression of $\hat{\mathcal{R}}(x)$. 

We, thus, partition the input space $\cX$ into four disjoint subspaces $\mathcal{X}_{1}, \mathcal{X}_{2}, \mathcal{X}_{3}, \mathcal{X}_{4}$, where 
\begin{eqnarray}
    \cX_{1} = \{ x \in \cX: \lambda_{\min} > 0, \hat{\lambda}_{\min} > 0\} \nonumber \\
    \cX_{2} = \{ x \in \cX: \lambda_{\min} > 0, \hat{\lambda}_{\min} \leq 0\} \nonumber \\
    \cX_{3} = \{ x \in \cX: \lambda_{\min} \leq 0, \hat{\lambda}_{\min} > 0\} \nonumber \\
    \cX_{4} = \{ x \in \cX: \lambda_{\min} \leq 0, \hat{\lambda}_{\min} \leq 0\} \nonumber
\end{eqnarray}
Based on the above partition of $\cX$, we have that 
\begin{eqnarray}
    \mathop{\mathbb{P}}\limits_{x\sim\mathcal{D}}\left[|\mathcal{R}(x) - \hat{\mathcal{R}}(x)| \geq \epsilon\right] &=& \sum_{i\in[4]} \mathop{\mathbb{P}}\limits_{x\sim\mathcal{D}}\left[x\in\cX_{i}\right] \mathop{\mathbb{P}}\left[|\mathcal{R}(x) - \hat{\mathcal{R}}(x)| \geq \epsilon_{i} \Big| x \in \cX_{i} \right] \nonumber \\
    &\leq& \sum_{i\in[4]} \mathop{\mathbb{P}}\limits_{x\sim\mathcal{D}}\left[x\in\cX_{i}\right] \delta \nonumber \\
    &\leq& \delta \label{eq_0_gen_bound}
\end{eqnarray}
where $\delta = \mathop{\mathbb{P}}\left[|\mathcal{R}(x) - \hat{\mathcal{R}}(x)| \geq \epsilon_{i} \Big| x \in \cX_{i} \right]$. To prove the needed, thus, it suffices to fix a probability $\delta$ and find the errors $\epsilon_{i}$ for each of the four cases, and, finally, let $\epsilon = \max\limits_{i\in[4]}\epsilon_{i}$ in \eqref{eq_0_gen_bound}.
\\
\textbf{Case 1 ($\cX_1$).} We have that $\mathcal{R}(x) = \frac{m(x)}{2c_M}$ and $\hat{\mathcal{R}}(x) = \frac{\hat{m}(x)}{2\hat{c}_M}$ and thus according to Lemma~\ref{lemma: case1} for $\epsilon < c_M,$ it holds that 
\begin{eqnarray}
    |\mathcal{R}(x) - \hat{\mathcal{R}}(x)| &=& \mathcal{\Tilde{O}}\left(\frac{d^{1/2}}{n^{1/2}}\right) \nonumber
\end{eqnarray}
\textbf{Case 2 ($\cX_2$).} We have that $\mathcal{R}(x) = \frac{m(x)}{2c_M}$ and $\hat{\mathcal{R}}(x) = \frac{\hat{c}_M - \sqrt{\hat{c}_M^2 - \hat{m}(x) \hat{\lambda}_{min}}}{\hat{\lambda}_{min}}$ and according to Lemma~\ref{lemma: case2} for $\epsilon < \min\{\lambda_{min}, \lambda_{min}^{\Sigma_{i}}, c_M\},$ it holds that 
\begin{eqnarray}
    |\mathcal{R}(x) - \hat{\mathcal{R}}(x)| &=& \mathcal{\Tilde{O}}\left(\frac{d^{1/2}}{n^{1/2}}\right) \nonumber
\end{eqnarray}
\textbf{Case 3 ($\cX_3$).} We have that $\mathcal{R}(x) = \frac{c_M - \sqrt{c_M^2 - m(x) \lambda_{\min}}}{\lambda_{\min}}$ and $\hat{\mathcal{R}}(x) = \frac{\hat{m}(x)}{2\hat{c}_M}$ and according to Lemma~\ref{lemma: case3} for $\epsilon < \min\{\lambda_{min}, \lambda_{min}^{\Sigma_{i}}, c_M\},$ it holds that 
\begin{eqnarray}
    |\mathcal{R}(x) - \hat{\mathcal{R}}(x)| &=& \mathcal{\Tilde{O}}\left(\frac{d^{1/4}}{n^{1/4}}\right) \nonumber
\end{eqnarray}
\textbf{Case 4 ($\cX_4$).} We have that $\mathcal{R}(x) = \frac{c_M - \sqrt{c_M^2 - m(x) \lambda_{\min}}}{\lambda_{\min}}$ and $\hat{\mathcal{R}}(x) = \frac{\hat{c}_M - \sqrt{\hat{c}_M^2 - \hat{m}(x) \hat{\lambda}_{min}}}{\hat{\lambda}_{min}}$ and thus according to Lemma~\ref{lemma: case4} for $\epsilon < \min\{\lambda_{min}, c_M\},$ it holds that 
\begin{eqnarray}
    |\mathcal{R}(x) - \hat{\mathcal{R}}(x)| &=& \mathcal{\Tilde{O}}\left(\frac{d^{1/4}}{n^{1/4}}\right) \nonumber
\end{eqnarray}
Combining the above results with \eqref{eq_0_gen_bound}, we get that with probability at least $1 - \delta$ it holds that
\begin{eqnarray}
    |\mathcal{R}(x) - \hat{\mathcal{R}}(x)| \leq \mathcal{\Tilde{O}}\left(\frac{d^{1/4}}{n^{1/4}}\right) \nonumber
\end{eqnarray}
For a fixed error $0 < \epsilon < \epsilon_{\min}=\{\lambda_{\min}^{\Sigma_{i}}, \lambda_{\min}, c_M(x)\},$ in order to satisfy that 
\begin{eqnarray}
   |\hat{\mathcal{R}}(x) - \mathcal{R}(x)| < \mathcal{O}\left(\epsilon\right), \nonumber 
\end{eqnarray}
we have that number of samples needed are $n = \mathcal{\Tilde{O}}\left(\frac{d^{1/4}}{\epsilon^{1/4}}\right)$. 
\end{proof}
\newpage
\section{Proof of Theorem~\ref{thm: gen_ELLIPS}}
\label{app: thm: gen_ELLIPS}
\begin{proof}
\textbf{Case 1: $\epsilon = 0$.}
Denote with $x, x'$ the clean and the corresponding adversarially perturbed sample and with $z = f(x), z' = f(x')$ their embeddings in the latent space. Since the encoder network is locally $L_x$-Lipschitz around the sample $x$, we have that 
\begin{eqnarray}
    \|z - z'\|_2 = \|f(x) - f(x')\|_2 \leq L_x \|x - x'\|_2 \label{ineq1_gen_thm}
\end{eqnarray}
Given that $f(x)$ maps the input distribution to a latent distribution that is a mixture of Gaussians, we have from Theorem \ref{thm: certificate_robustness2} that the ELLIPS classifier remains robust as long as 
\begin{eqnarray}
    \|z' - z\|_2 \leq \frac{m(z)}{\lambda_{\min} \left(\sqrt{c_M^2 + (\lambda_{\min})_{+} m(z)}+c_M\right)} \label{ineq2_gen_thm}
\end{eqnarray}
where $\lambda_{\min}$ is the minimum among all eigenvalues of the matrices $W_i = \Sigma^{-1}_i - \Sigma^{-1}_{i_*}, \forall i \neq i_*$, $(-\lambda_{\min})_+=\max(-\lambda_{\min},0)$, and $c_M = \max\limits_{i \neq {i_*}} \|\Sigma_{i_*}^{-1} (f(x) - \mu_{i_*})^T - \Sigma_i^{-1} (f(x)-\mu_i)^T\|_2$.
Thus, in order for the classifier to remain robust, it suffices that the perturbation in the input space satisfies
\begin{eqnarray}
    L_x \|x - x'\|_2 &\leq& \frac{m(z)}{\lambda_{\min} \left(\sqrt{c_M^2 + (\lambda_{\min})_{+} m(z)}+c_M\right)}\nonumber \\
    \iff R(x) &\leq& \frac{m(z)}{\lambda_{\min} L_x\left(\sqrt{c_M^2 + (\lambda_{\min})_{+} m(z)}+c_M\right)} \nonumber
\end{eqnarray}
\textbf{Case 2: $\epsilon > 0$.}
Let $P_z := f_{\#}\mathcal D_x$
denote the true latent distribution induced by the encoder, and let $Q_z$ be the
Gaussian mixture fitted to this latent distribution. Let $\mathcal{E}, \Pi$ denote the
parameters of $Q_z$, and let $\mathcal C_{\mathcal{E}, \Pi}$ be the set of latent points that
are certified by the \textsc{GenELLIPS} certification rule with respect to
$\mathcal{E}, \Pi$
\[
\mathcal C_{\mathcal{E},\Pi}
:=
\left\{
z\in \mathcal Z:
\textsc{GenELLIPS} \text{ is certified at } z
\right\}.
\]
Therefore, by definition,
\[
\mathrm{CertAcc}(P_z)=P_z(\mathcal C_{\mathcal{E},\Pi}),
\qquad
\mathrm{CertAcc}(Q_z)=Q_z(\mathcal C_{\mathcal{E},\Pi}).
\]
Using assumption $\mathrm{KL}(P_z\|Q_z)\leq \epsilon$ and Pinsker's inequality, we have
\begin{eqnarray}
d_{\mathrm{TV}}(P_z,Q_z)
&\leq&
\sqrt{\frac{1}{2}\mathrm{KL}(P_z\|Q_z)}
\nonumber \\
&\leq&
\sqrt{\frac{\epsilon}{2}} .
\end{eqnarray}
By the definition of total variation distance, for every measurable set
$A\subseteq \mathcal Z$, we have
\[
|P_z(A)-Q_z(A)|
\leq
d_{\mathrm{TV}}(P_z,Q_z).
\]
Applying this inequality to the certified set $A=\mathcal C_{\mathcal{E},\Pi}$ gives
\begin{eqnarray}
P_z(\mathcal C_{\mathcal{E},\Pi})
&\geq&
Q_z(\mathcal C_{\mathcal{E},\Pi})
-
d_{\mathrm{TV}}(P_z,Q_z)
\nonumber \\
&\geq&
Q_z(\mathcal C_{\mathcal{E},\Pi})
-
\sqrt{\frac{\epsilon}{2}} .
\end{eqnarray}
Using the definitions of certified accuracy under $P_z$ and $Q_z$, we obtain
\[
\mathrm{CertAcc}(P_z)
\geq
\mathrm{CertAcc}(Q_z)
-
\sqrt{\frac{\epsilon}{2}} .
\]
\end{proof}
\newpage
\section{On Experiments}
\label{app: experiments}
In Appendix~\ref{app: experiments_details}, we provide more details on the experiments presented in the main paper. In Appendix~\ref{app: experiments_additional_exps}, we provide additional experiments showcasing the performance of the proposed classifier in practice. 
\subsection{Experimental Details}
\label{app: experiments_details}
We first describe the experimental setup used and then provide additional synthetic experiments.


\textbf{Experiments in Benchmark Datasets.} We provide the training details - network architecture, datasets, optimization and hyperparameters for the implementation of the \textsc{GenELLIPS} classifier. 

\textbf{Network Architecture.} To construct the proposed classifier we need to apply first an encoder and then the \textsc{ELLIPS} classifier. We take a FARE-4 encoder \citep{schlarmann2024robustclip} pre-trained and finetune it using a loss that promotes the latent distribution to comprise a mixture of Gaussians. Given that the ImageNet dataset appears to have more classes and be more complex than the CIFAR-10, we have utilized a meticulously constructed loss accustomed to each dataset. Specifically, for CIFAR-10 the used loss combines the MCR$^2$ objective with a term promoting the Gaussian marginals to be isotropic, ensuring that the eigenvalues of the covariance matrices are well-behaved
\begin{eqnarray}
   \mathcal{L} = \mathcal{L}_{\mathrm{MCR}^2}(Z, Y) \;+\; \lambda_{\mathrm{iso}} \mathcal{L}_{\text{iso}} \, 
\end{eqnarray}
where $\mathcal{L}_{\text{iso}} = \sum_{k=1}^K \Big\|C_k \;-\;\frac{\mathrm{Tr}(C_k)}{d}\,I_d\Big\|_F^2$ and 
\begin{eqnarray}
    \qquad Z = \bigl[z_1, \dots, z_B\bigr]^\top,\qquad
\mu_k = \frac{1}{n_k}\sum_{i:y_i=k} z_i,\qquad
C_k = \frac{1}{n_k-1}\sum_{i:y_i=k}(z_i-\mu_k)(z_i-\mu_k)^\top \nonumber
\end{eqnarray}

For the ImageNet dataset, given the significantly more complex underlying distribution, we add an additional regularizer, measuring the maximum mean discrepancy of each class conditional from a Gaussian distribution 
\begin{eqnarray}
    \mathcal{L} = \mathcal{L}_{\text{MCR}^2} + \lambda_{\text{iso}} \cdot \mathcal{L}_{\text{iso}} + \lambda_{\text{MMD}} \cdot \sum_k \mathrm{MMD}^2(z_k, \Tilde{z}_k),
\end{eqnarray}
where $\Tilde{z}_k \sim \mathcal{N}(\mu_k, I)$ and $\mathrm{MMD}^2(z_k, \Tilde{z}_k)$ is a kernel-based distance between two discrete distributions defined as follows
\begin{eqnarray}
    \text{MMD}^2(z_k, \Tilde{z}_k) = \frac{1}{n^2} \sum_{i, j} k(z_i, z_j) + \frac{1}{m^2} \sum_{i, j} k(\Tilde{z}_i, \Tilde{z}_j) - \frac{2}{nm} \sum_{i, j} k(z_i, \Tilde{z}_j)
\end{eqnarray}

where \( k(\cdot, \cdot) \) is a positive-definite kernel function.

The choice of the kernel is the Gaussian Radial Basis Function (RBF) kernel
\[
k(x, y) = \exp\left(-\frac{\|x - y\|^2}{2 \sigma^2}\right)
\]

During training, we freeze the FARE-4 backbone and we add, similarly to \cite{chu2023image}, a pre-feature layer composed of Linear-BatchNorm-ReLU-Linear-ReLU. For feature head and cluster head, we utilize a Linear layer that maps the hidden to the feature dimension $d = 128$. 

\textbf{Optimization, Initialization and Hyperparameters.} We initially warmup our pipeline by training the MCR$^2$ loss
and then optimize simultaneously the feature cluster head using the MLC loss. Following \citet{chu2023image}, we use the SGD optimizer for both the
feature head and cluster head with learning rate equal to 0.0001, momentum
set to 0.9 and weight decay set to 0.0001 and 0.005 respectively. All other hyperparameters remain the same to the ones used in \citet{chu2023image}, thus referring the interested reader to the aforementioned related work.

\subsection{Additional Experiments}
\label{app: experiments_additional_exps}

\textbf{Separation of Classes.} We visualize the correlation of the latent embeddings of different classes showing the effectiveness of the $MCR^2$ loss in the CIFAR-10 dataset. As shown in Figure \ref{fig: ZZ^T}, such an encoder trained with the $MCR^2$ objective maps each class of input samples to points near a low-dimensional subspace, as the singular values of the mapped points drop quickly, while the mapped points from different subspaces tend to be orthogonal.

\begin{figure}[ht]
    \centering
    \includegraphics[width=0.3\linewidth]{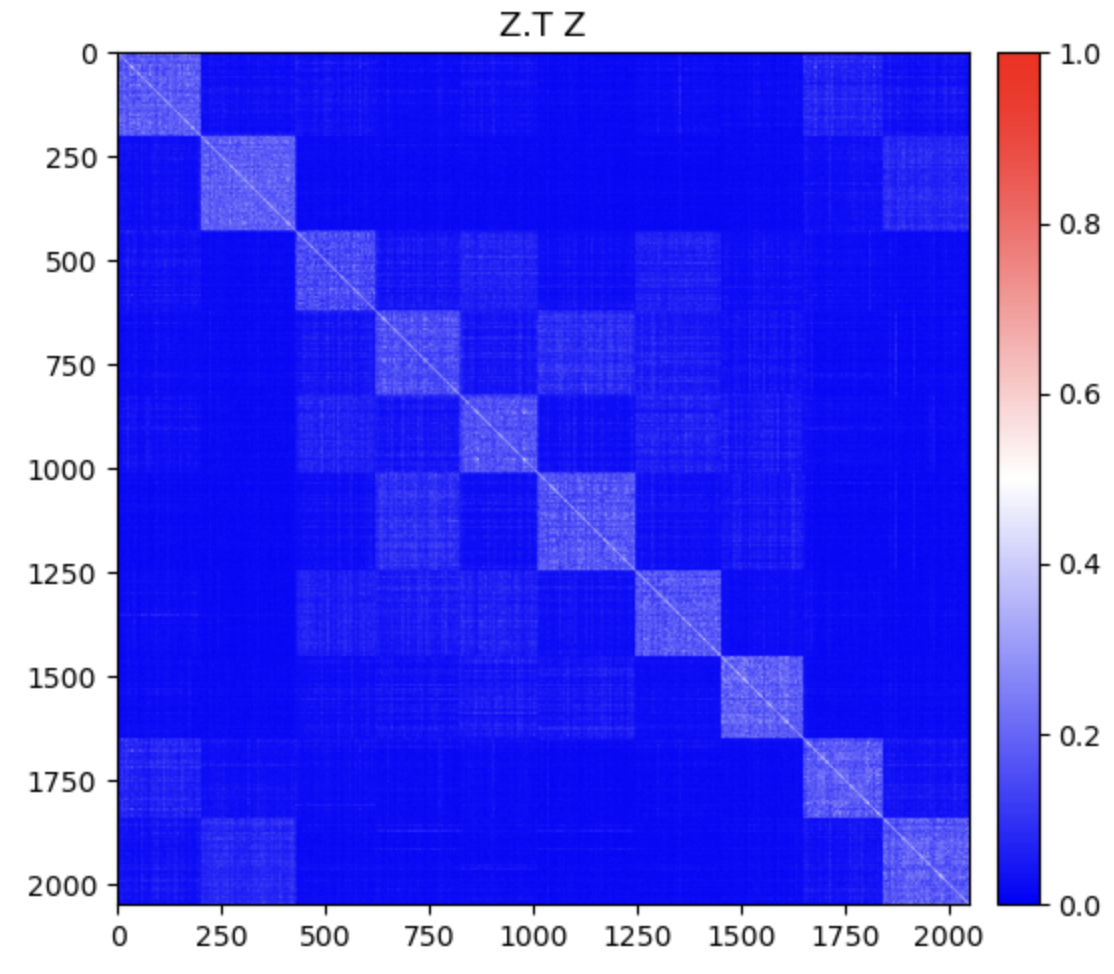}
    \includegraphics[width=0.45\linewidth]{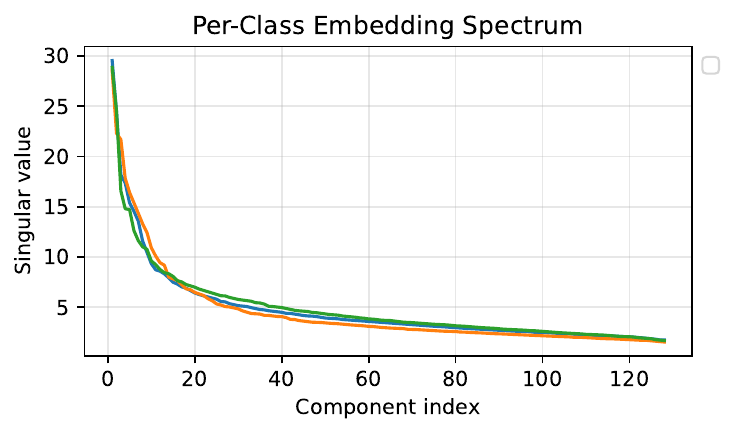}
    \caption{The correlation matrix and minimum eigenvalues of the latent space embeddings for the different classes for CIFAR-10.}
    \label{fig: ZZ^T}
\end{figure}

\textbf{Empirical Validation of Sample Complexity.} In order to empirically validate the result of Theorem \ref{thm: generalization_bound}, we have first expressed the established sample complexity in terms of logarithms, as follows: 
from $n = \Tilde{\mathcal{O}}(\frac{d^{9/2}}{\epsilon^{9/2}})$, we know that there exists a constant $C > 0$ such that 
\begin{eqnarray}
    n &\leq& \frac{Cd^{9/2}}{\epsilon^{9/2}} \nonumber\\
    \Rightarrow \log{\epsilon} &\leq& -\frac{2}{9} \log{n} + \log{d} + \frac{2}{9} \log{C} \label{eq: log_eq_generalization_thm_valid}
\end{eqnarray}
We run multiple experiments for different sample sizes $n \in \{10, 100, 500, 1000\}$ and dimensions $d \in \{2, 10, 100\}$ and estimate in each experiment the distance of the learned certificate from the true certificate of robustness. We, next perform linear regression to estimate the coefficients $\alpha, \beta, \gamma$ in the following equality and compare with the ones of \eqref{eq: log_eq_generalization_thm_valid}. We have found that and , thus validating empirically (1) and hence the sample complexity established in Theorem 4.2.
CIFAR-10 We run multiple experiments for different sample sizes 
 and dimensions 
 and estimate in each experiment the distance 
 of the learned certificate from the true certificate of robustness. We, lastly, performed linear regression to estimate the coefficients 
 in the following equality 
 and compared them with the ones in (1). We have found that 
 
 and 
, thus validating empirically (1) and hence the sample complexity established in Theorem 4.2. For the same example, we have plotted additionally the difference 
 of the learned certified radius from the true one for different sample sizes 
 and have shown how the certified radius 
 scales with respect to the parameters 
 and 
.

\textbf{On Gaussianity of the Latent Distribution.} To empirically validate that the used encoder maps the input distribution to a mixture of Gaussians, we apply Mardia’s statistical normality test \citep{mardia}, a well-known statistical test that evaluates whether a multivariate dataset departs from a Gaussian distribution. More specifically, Mardia’s test computes two statistics:
\begin{enumerate}
    \item multivariate skewness, which accounts for asymmetry
    \item multivariate kurtosis, which evaluates whether the distribution’s tail behavior matches the one of a Gaussian.
\end{enumerate}

Under the null hypothesis, the data follow a multivariate normal distribution. The results show that the embeddings pass this test for all the classes, indicating that the class-conditional distributions are indeed conforming to Gaussians.

\begin{table}[h!]
\centering
\begin{tabular}{|c|c|c|}
\hline
\textbf{Dataset} & \textbf{Mardia’s Average Score} & \textbf{Percentage of Classes Passing Normality Test} \\
\hline
CIFAR-10 & 0.027 & 100\% \\
ImageNet & 0.014 & 100\% \\
\hline
\end{tabular}
\caption{Mardia’s test results validate that the latent distribution of the encoder conforms with a mixture of Gaussian distributions.}
\end{table}

\textbf{Synthetic Experiments.} We conduct experiments in the Gaussian mixture setting, where the input distribution is comprised of $K$ classes and each class is distributed according to $\mathcal{N}(\mu_i, \Sigma_i), \forall i \in [K]$. We run experiments for multiple setups testing for different number of classes $K=\{2, 3, 5, 10\}$ with different distances $R = \{2, 4, 6\}$ between them, as well as isotropic and non-isotropic covariances matrices $\Sigma$. The means are generated to lie in a circle with angle $\frac{2\pi}{K}$ and radius $R = \{2, 4, 6\}$ from the center in order to control the intersection between the classes. The covariance matrices $\Sigma_i$ are selected to be either isotropic or anisotropic. In the case of isotropic covariances, $\Sigma_i = I,$ while in the case of anisotropic covariance matrices the variance in the principal and second principal direction is $1.5, 0.5$ respectively. 

We compare the certified accuracy of the proposed classifier with the method of \citet{paladversarial}. We plot in Figure~\ref{fig: addit_synthetic} the certified accuracy of both methods for the isotropic GMMs and in Figure~\ref{fig: addit_synthetic_anisotropic_covs} for anisotropic covariances. 
As shown in Figure~\ref{fig: addit_synthetic} and Figure~\ref{fig: addit_synthetic_anisotropic_covs}, our method outperforms the one in \citet{paladversarial} and closely approximates the empirical robust accuracy achieved by PGD attack. 

Additionally, we compare the certified radius of Theorem~\ref{thm: certificate_robustness2} with the archetypal technique of randomized smoothing in different settings. As shown in Figure~\ref{fig: addit_rand_smoothing} and Figure~\ref{fig: addit_rand_smoothing_anisotropic}, our method provides higher certified accuracy than randomized smoothing, indicating the tighter certification of the proposed radius of robustness. 

\begin{figure}[ht!]
    \centering
    \includegraphics[width=0.3\linewidth]{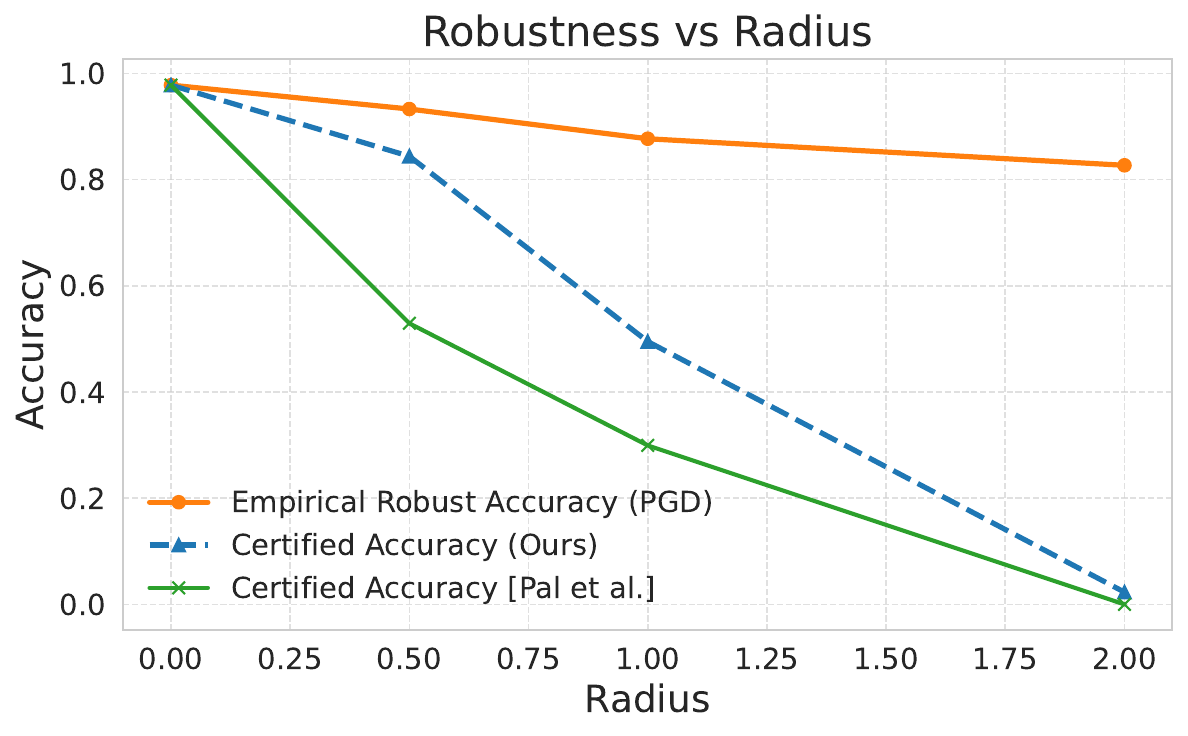}
    \includegraphics[width=0.3\linewidth]{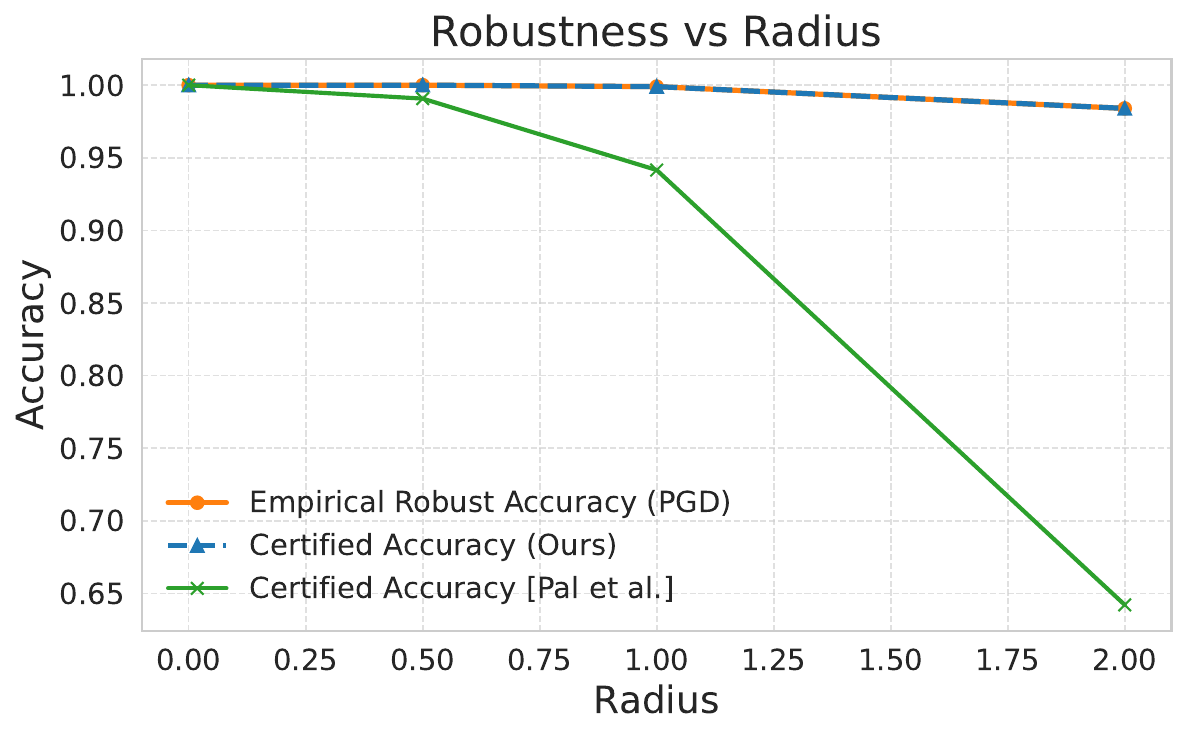}
    \includegraphics[width=0.3\linewidth]{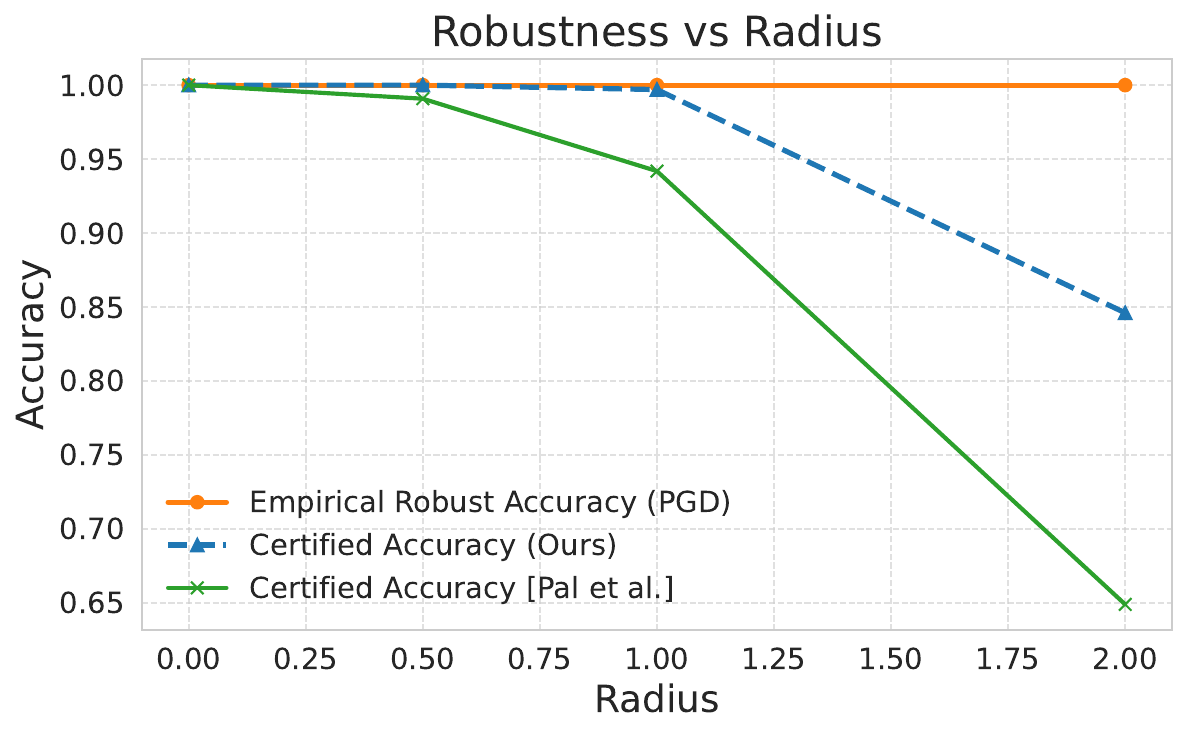}\\
    \includegraphics[width=0.3\linewidth]{Figs/synthetic_new/acc_K3_R2_isotropic_8.pdf}
    \includegraphics[width=0.3\linewidth]{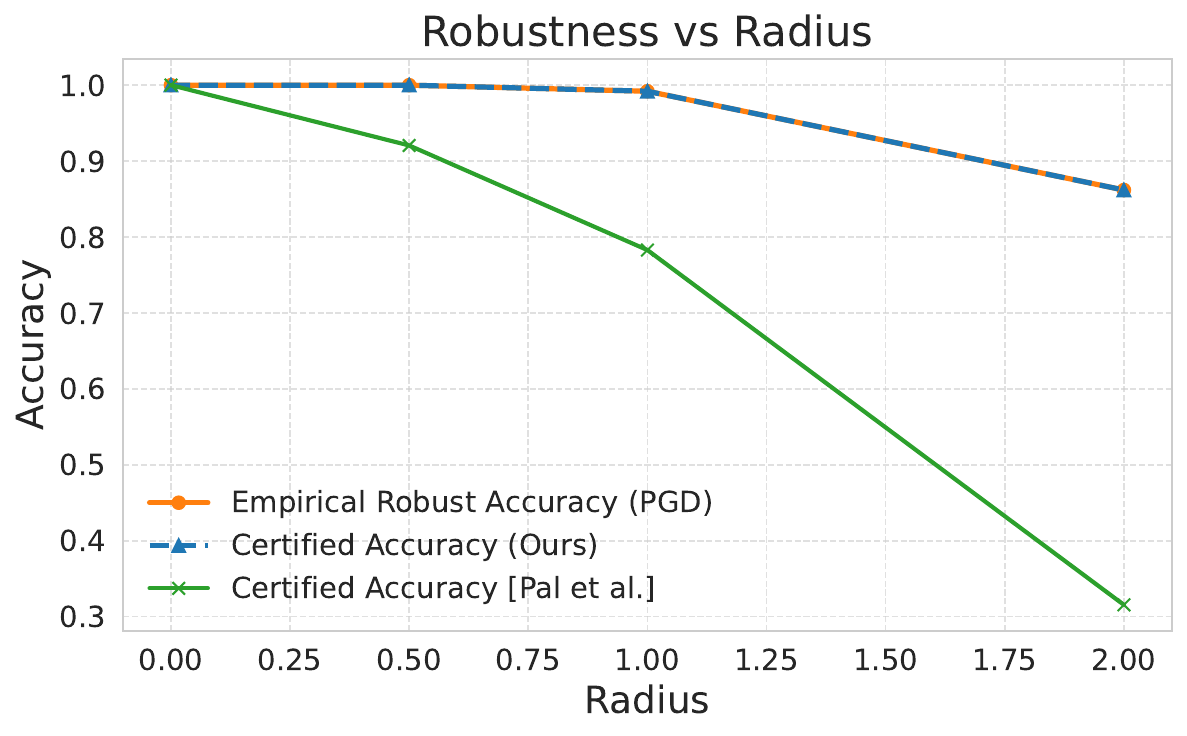}
    \includegraphics[width=0.3\linewidth]{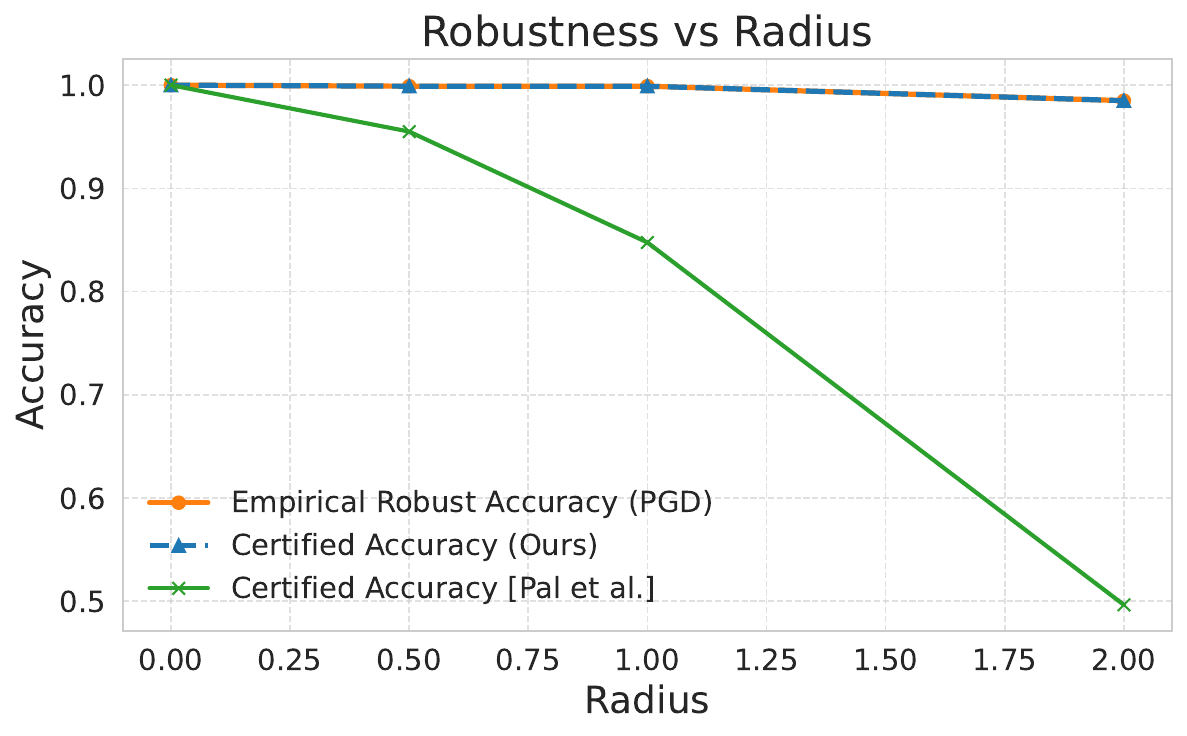}\\
    \includegraphics[width=0.3\linewidth]{Figs/synthetic_new/acc_K2_R2_isotropic_6.pdf}
    \includegraphics[width=0.3\linewidth]{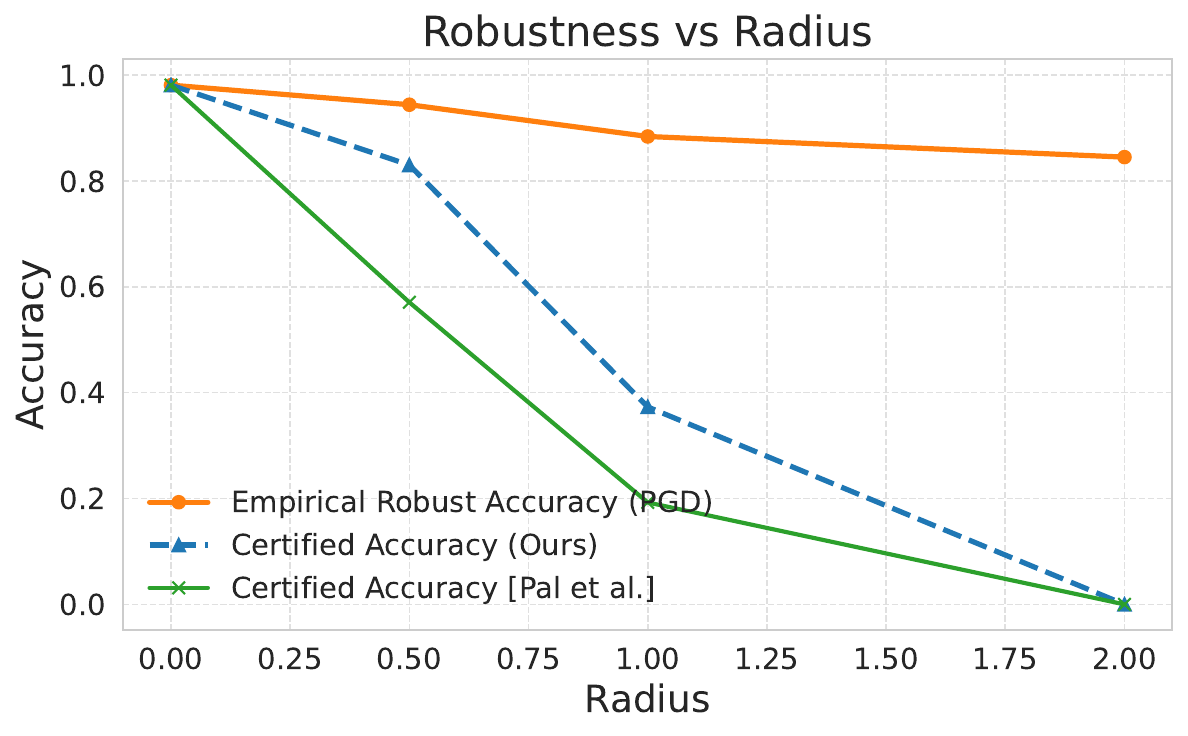}
    \includegraphics[width=0.3\linewidth]{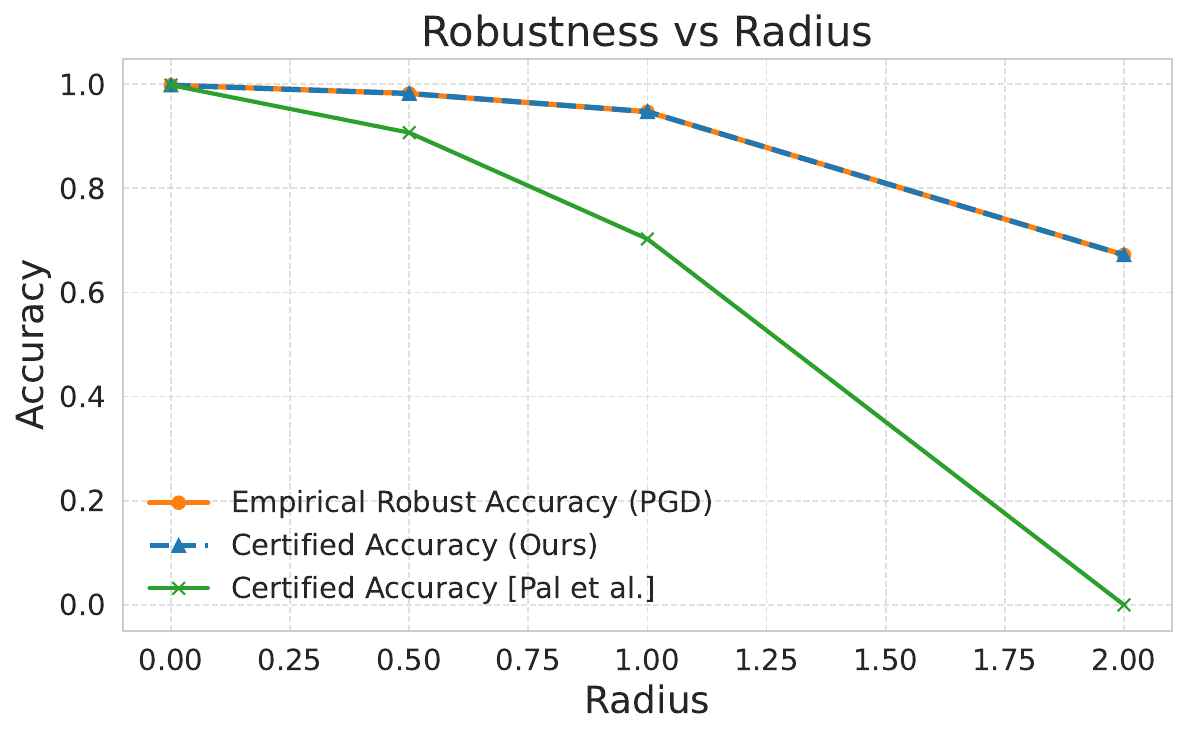}
    \caption{The proposed approach outperforms the method of \citet{paladversarial} in different Gaussian mixture settings. Each row corresponds to a GMM with isotropic covariances and different number of classes $K = \{2, 3, 5\}$, while each column to one with different separation distance $R = \{2, 4, 5\}$.}
    \label{fig: addit_synthetic}
\end{figure}

\begin{figure}[ht]
    \centering
    \includegraphics[width=0.3\linewidth]{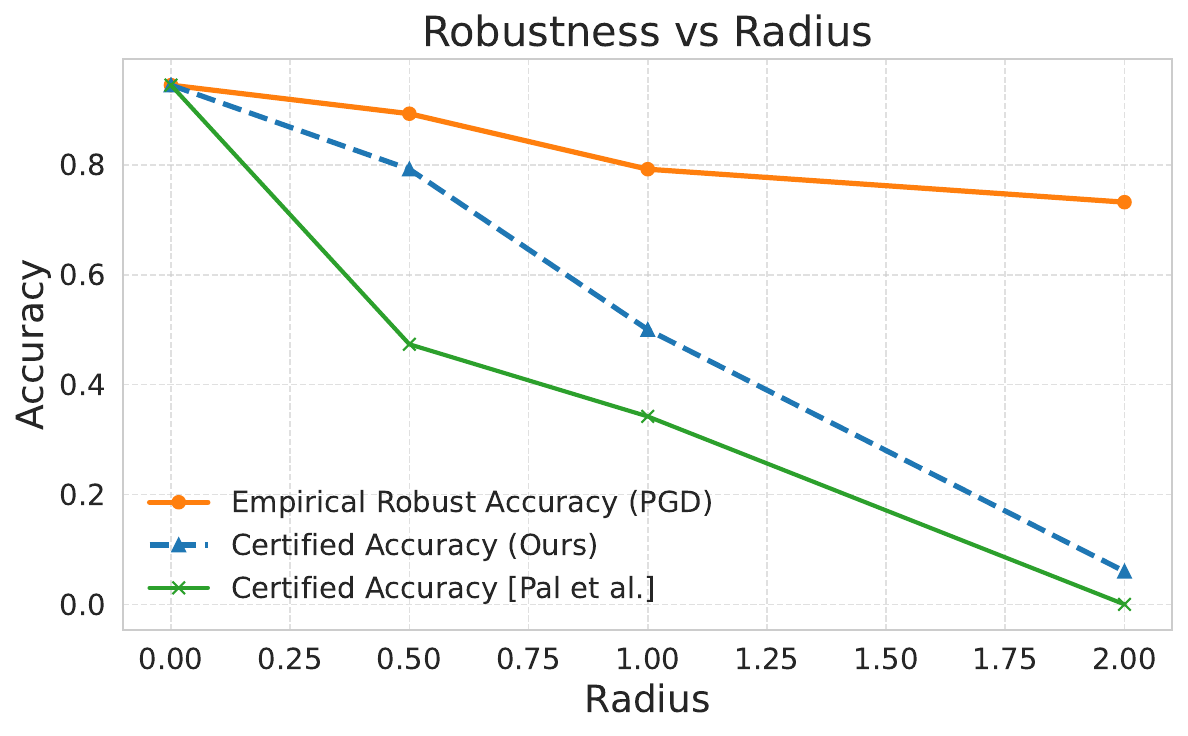}
    \includegraphics[width=0.3\linewidth]{Figs/synthetic_exper_last/acc_K2_R4_isotropic_3.pdf}
    \includegraphics[width=0.3\linewidth]{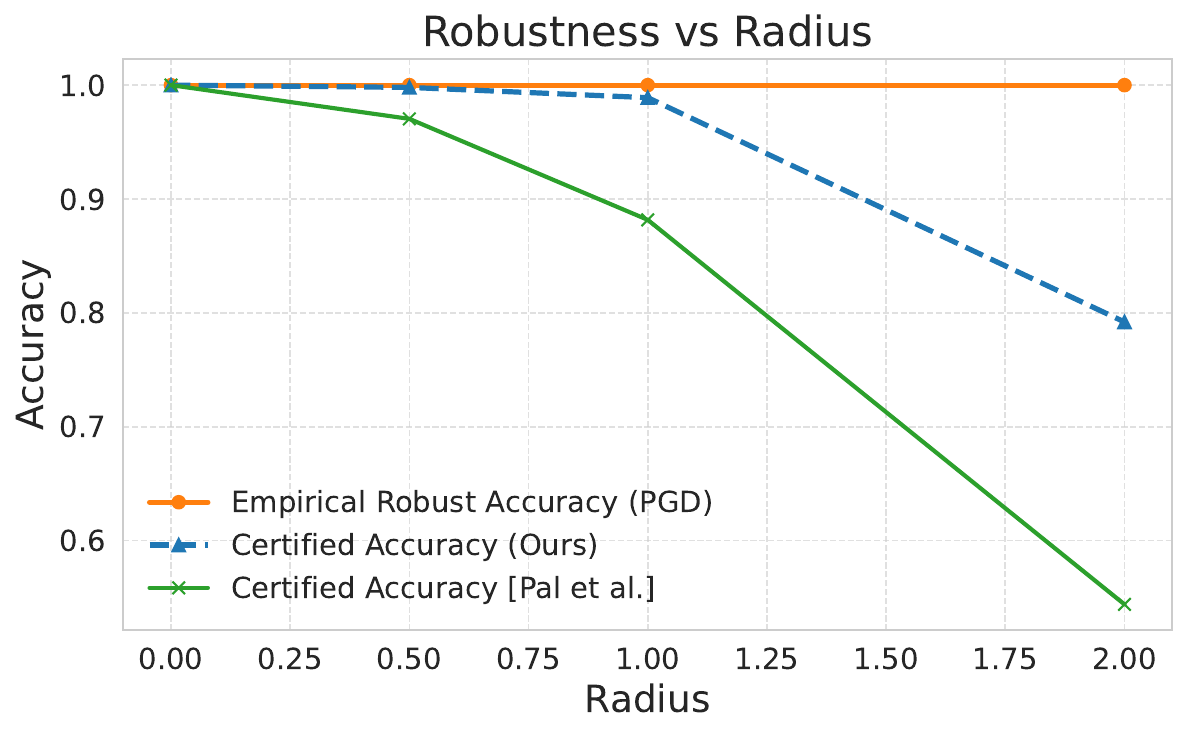}\\
    \includegraphics[width=0.3\linewidth]{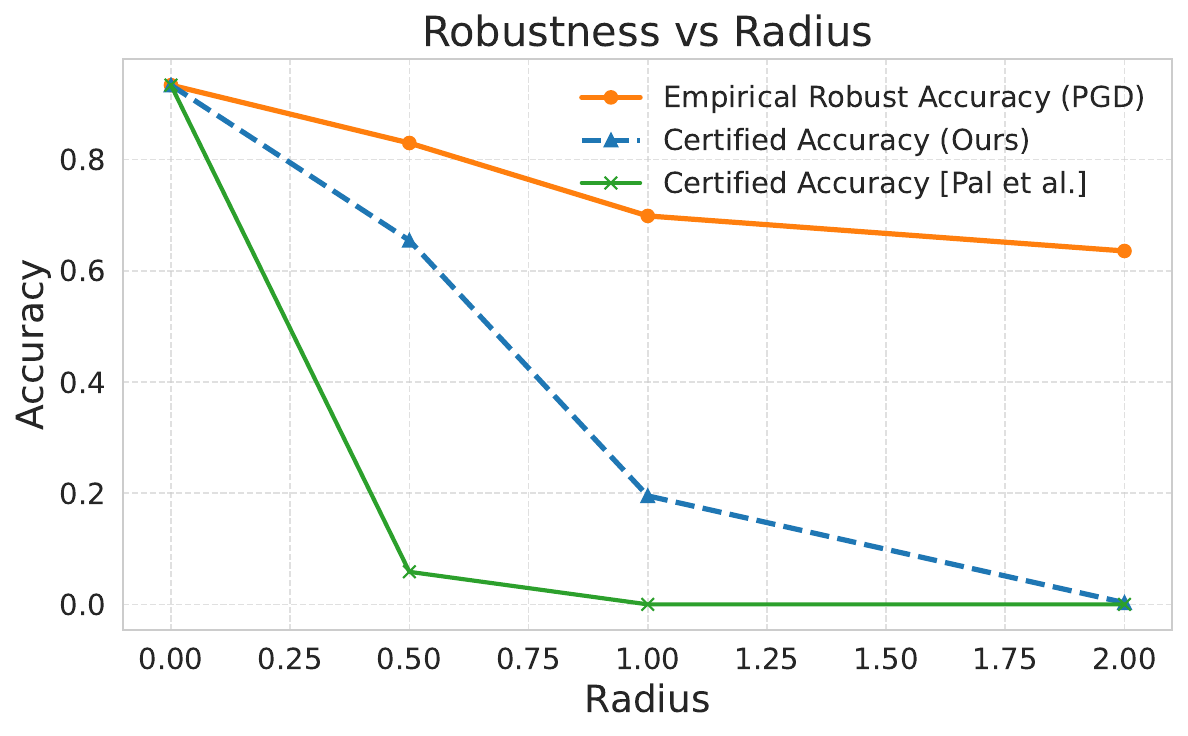}
    \includegraphics[width=0.3\linewidth]{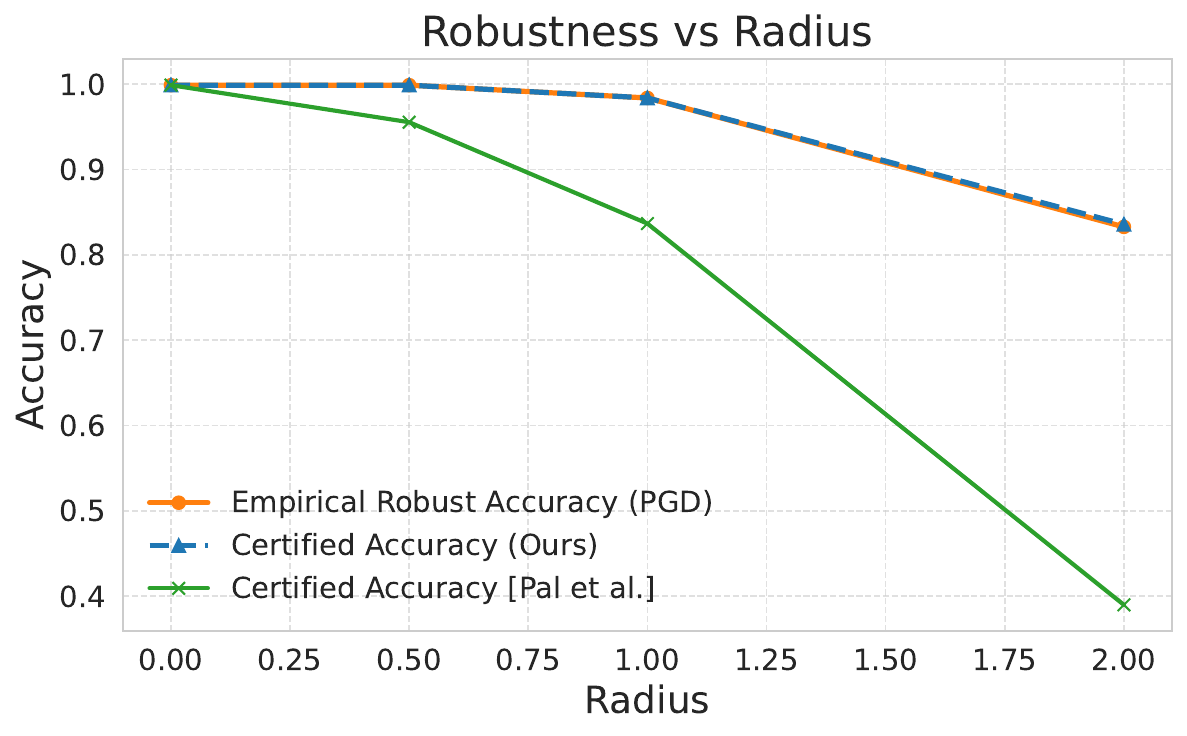}
    \includegraphics[width=0.3\linewidth]{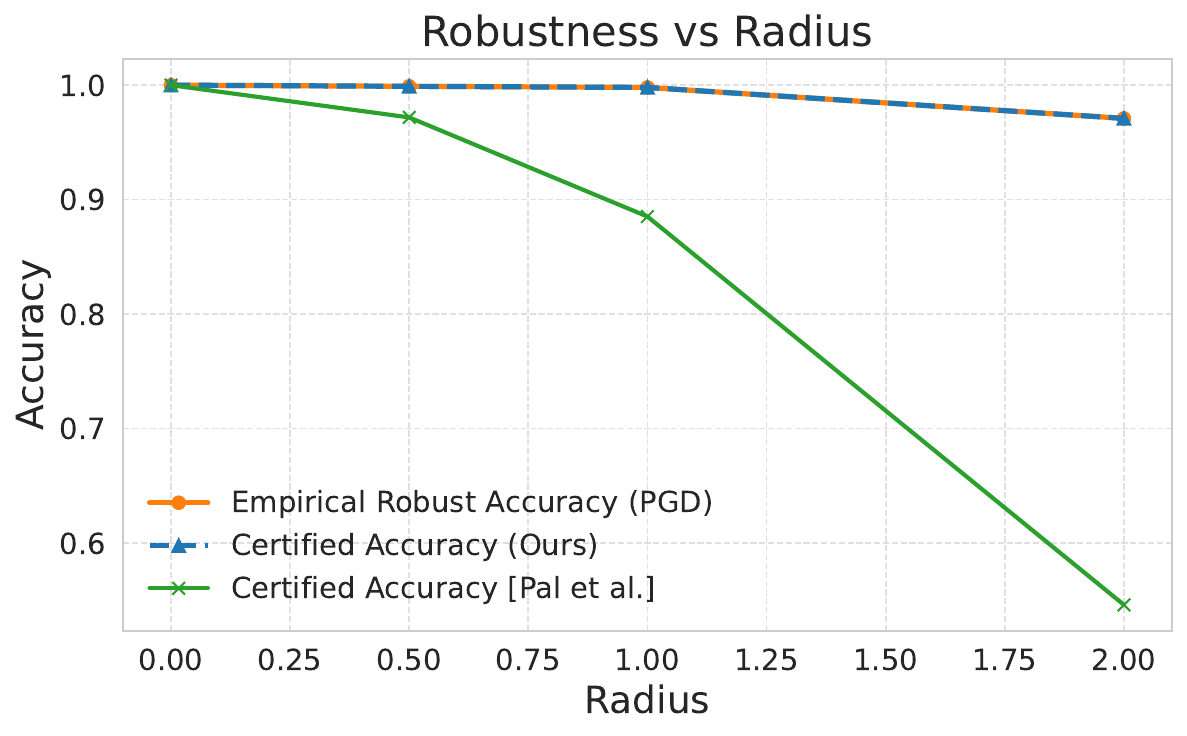}\\
    \includegraphics[width=0.3\linewidth]{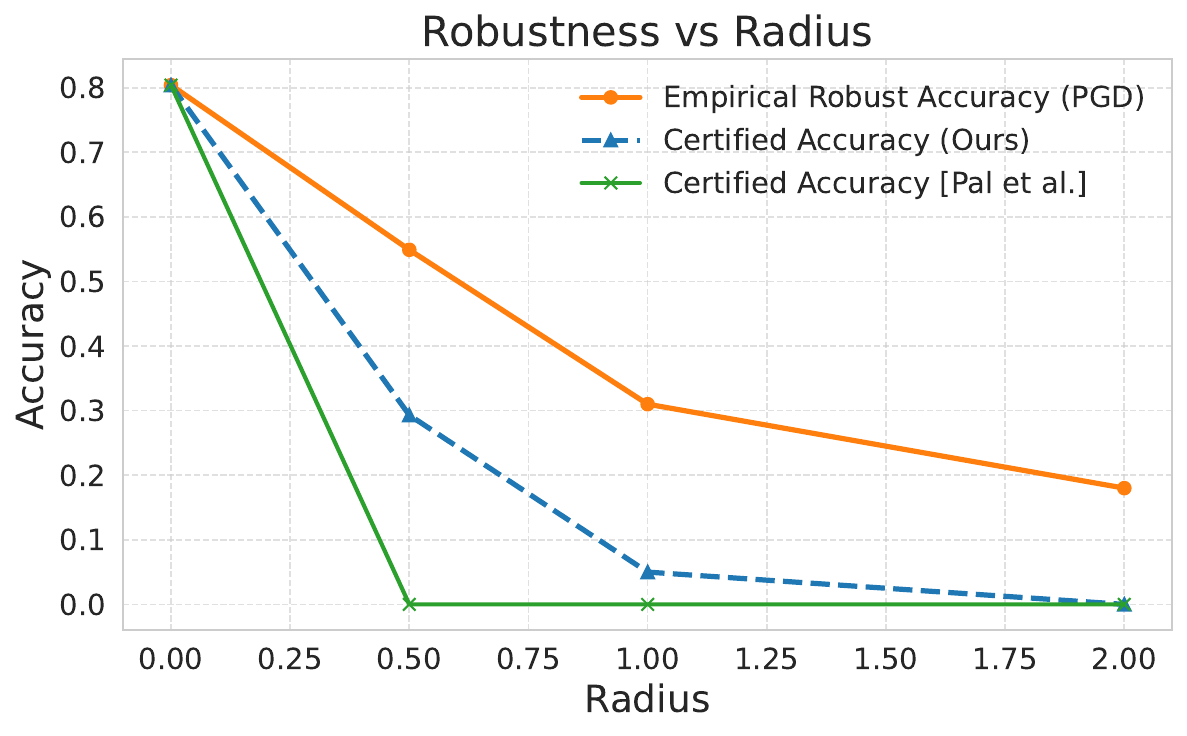}
    \includegraphics[width=0.3\linewidth]{Figs/synthetic_new/acc_K5_R4_anisotropic_1.pdf}
    \includegraphics[width=0.3\linewidth]{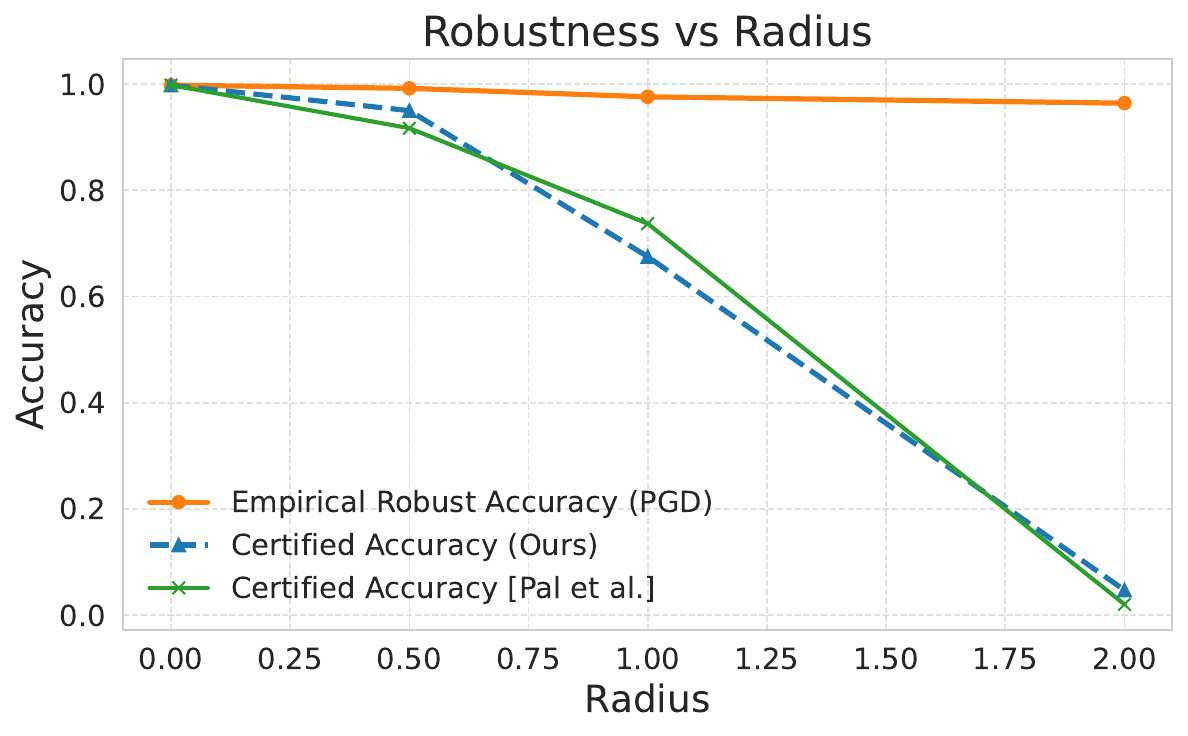}
    \caption{The proposed method outperforms the method of \citet{paladversarial} in different Gaussian mixtures with \textit{anisotropic} covariance matrices. Each row corresponds to a GMM with different number of classes $K = \{2, 3, 5\}$, while each column to one with different separation distance $R = \{2, 4, 5\}$.}
    \label{fig: addit_synthetic_anisotropic_covs}
\end{figure}
\newpage

\begin{figure}[ht!]
    \centering
    \includegraphics[width=0.3\linewidth]{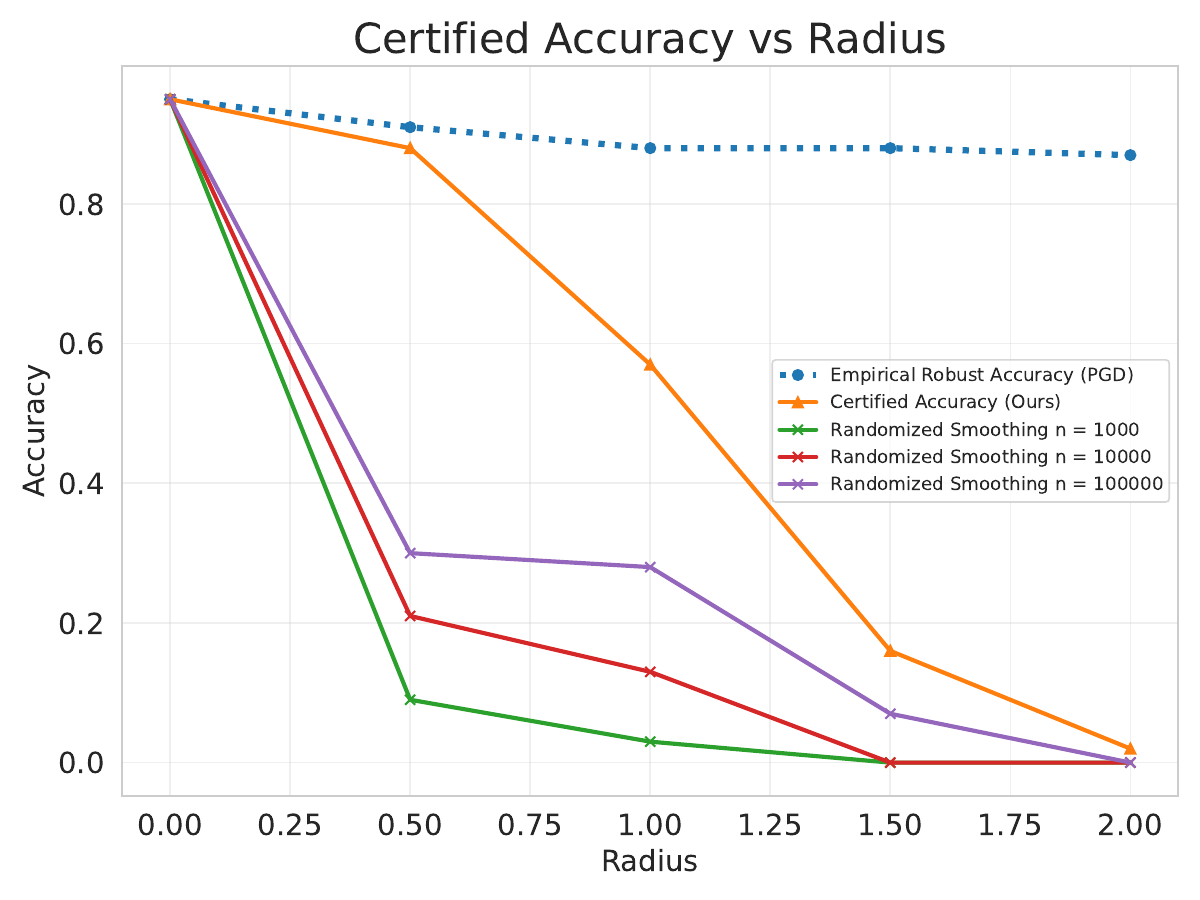}
    \includegraphics[width=0.3\linewidth]{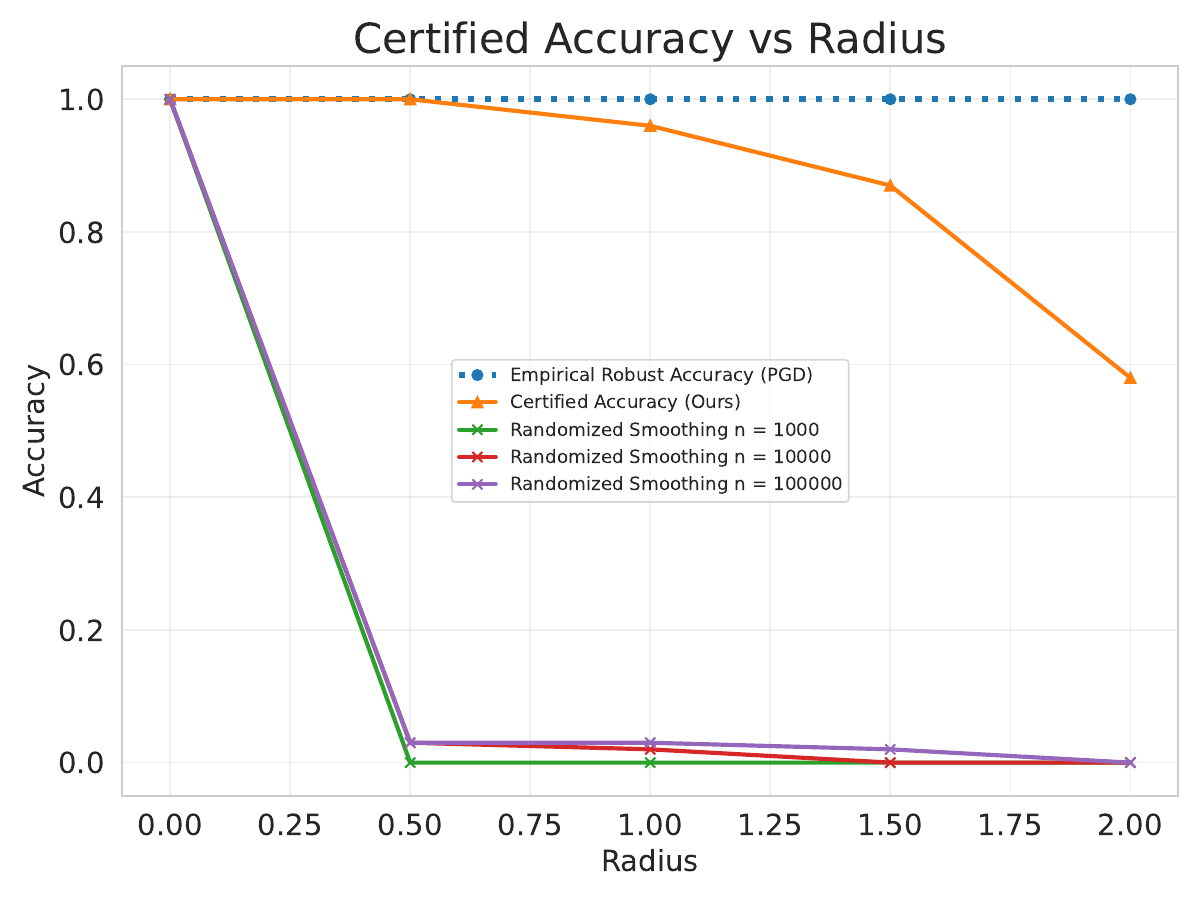}
    \includegraphics[width=0.3\linewidth]{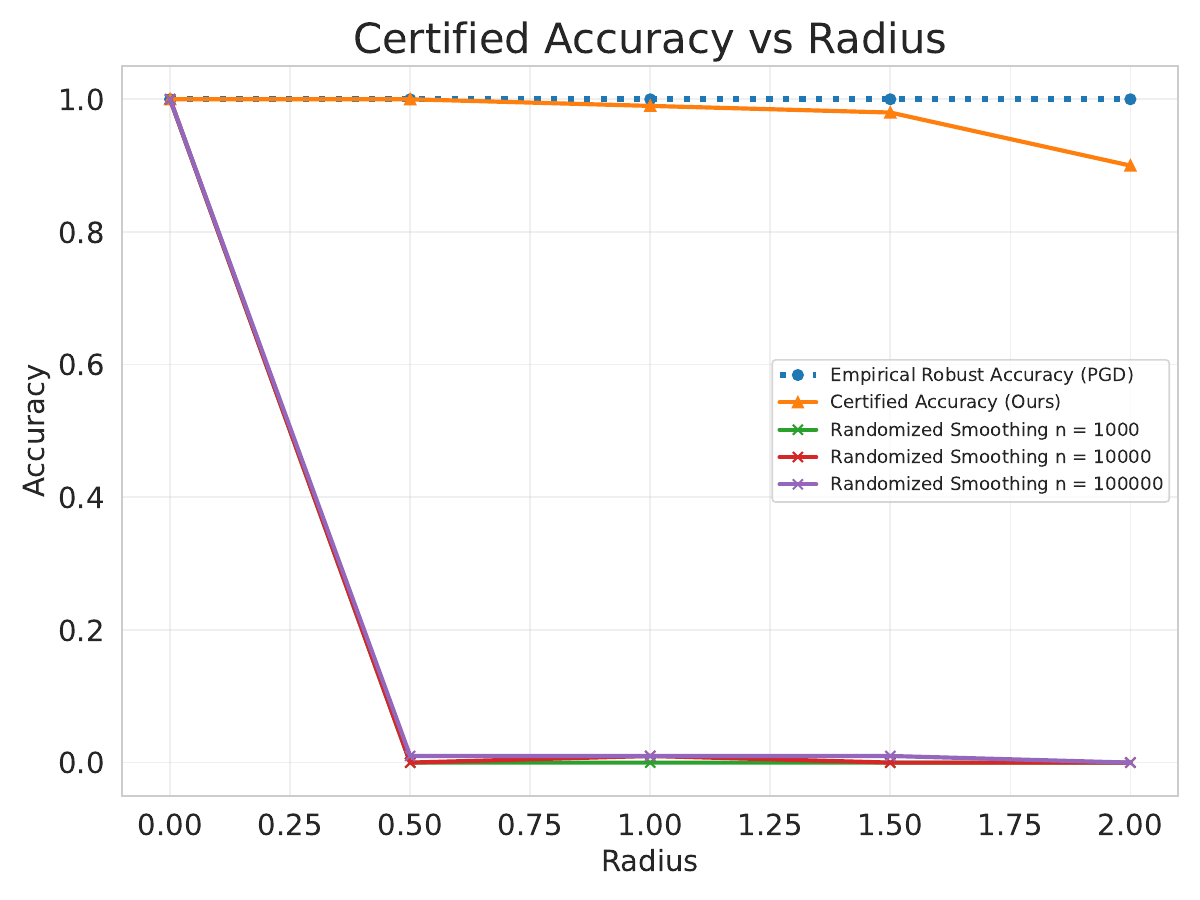}\\
    \includegraphics[width=0.3\linewidth]{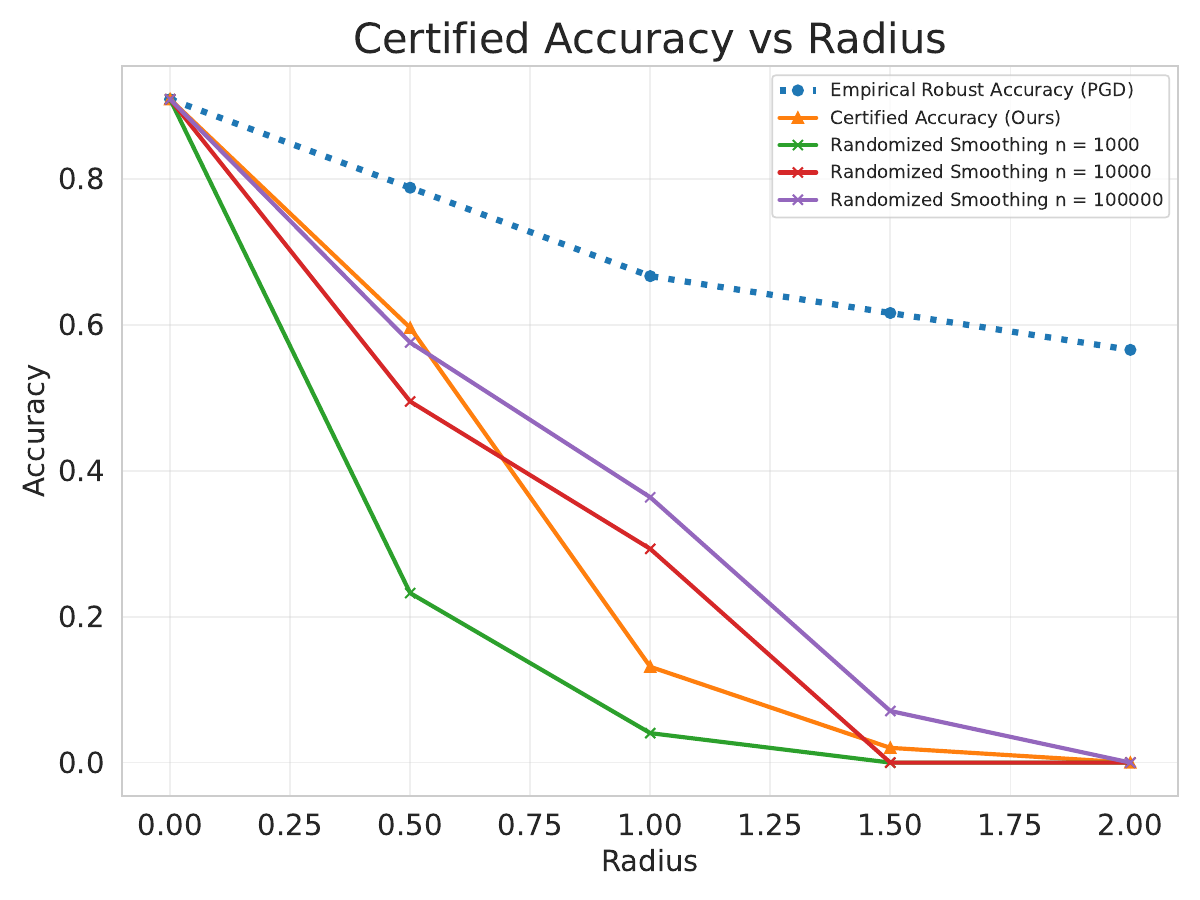}
    \includegraphics[width=0.3\linewidth]{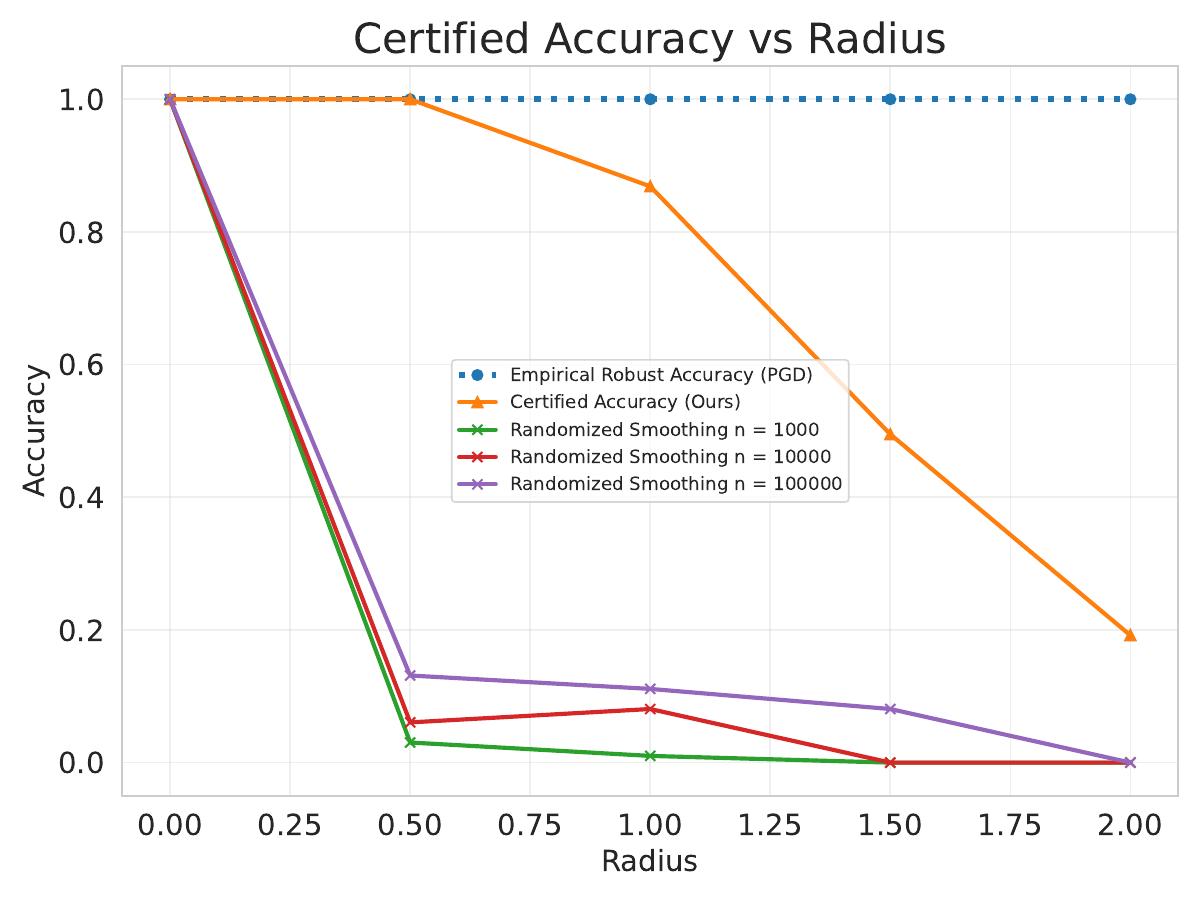}
    \includegraphics[width=0.3\linewidth]{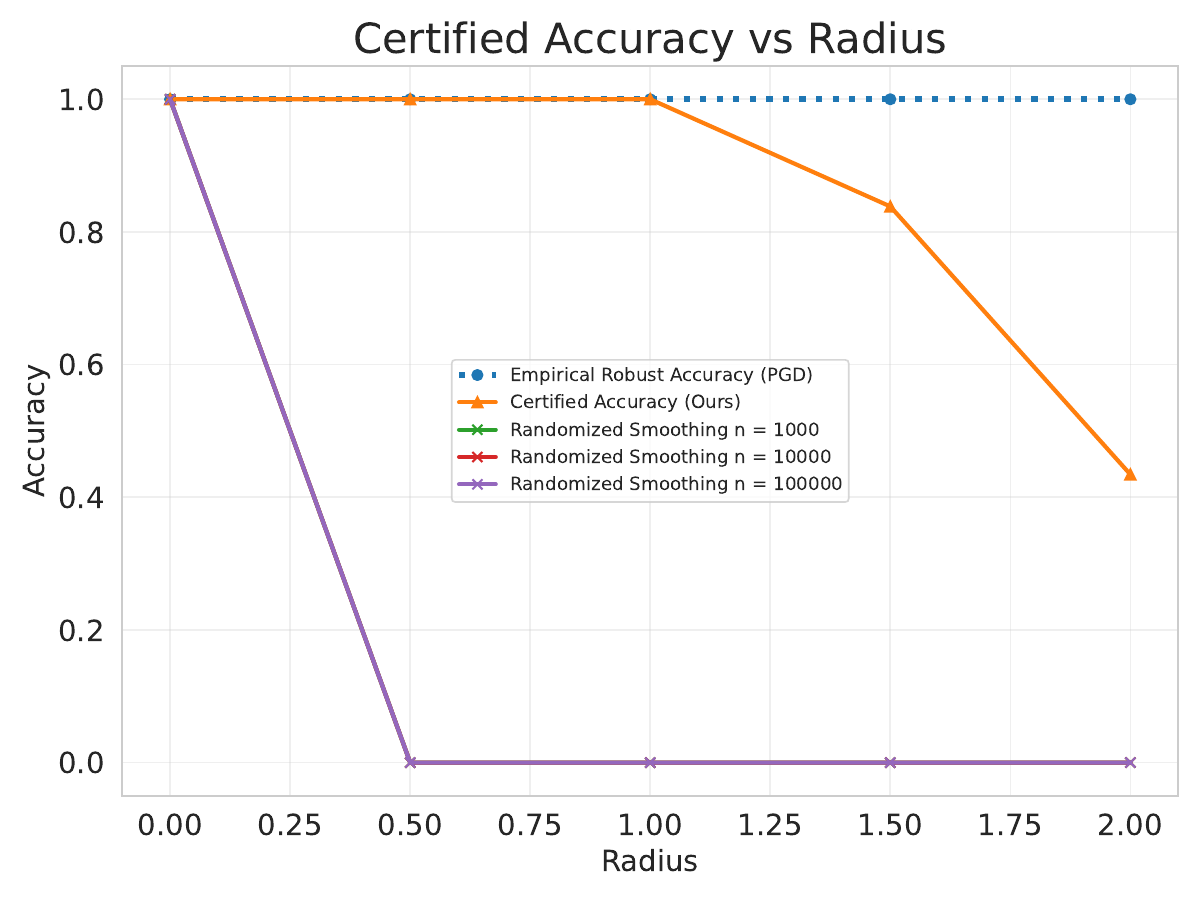}\\
    \includegraphics[width=0.3\linewidth]{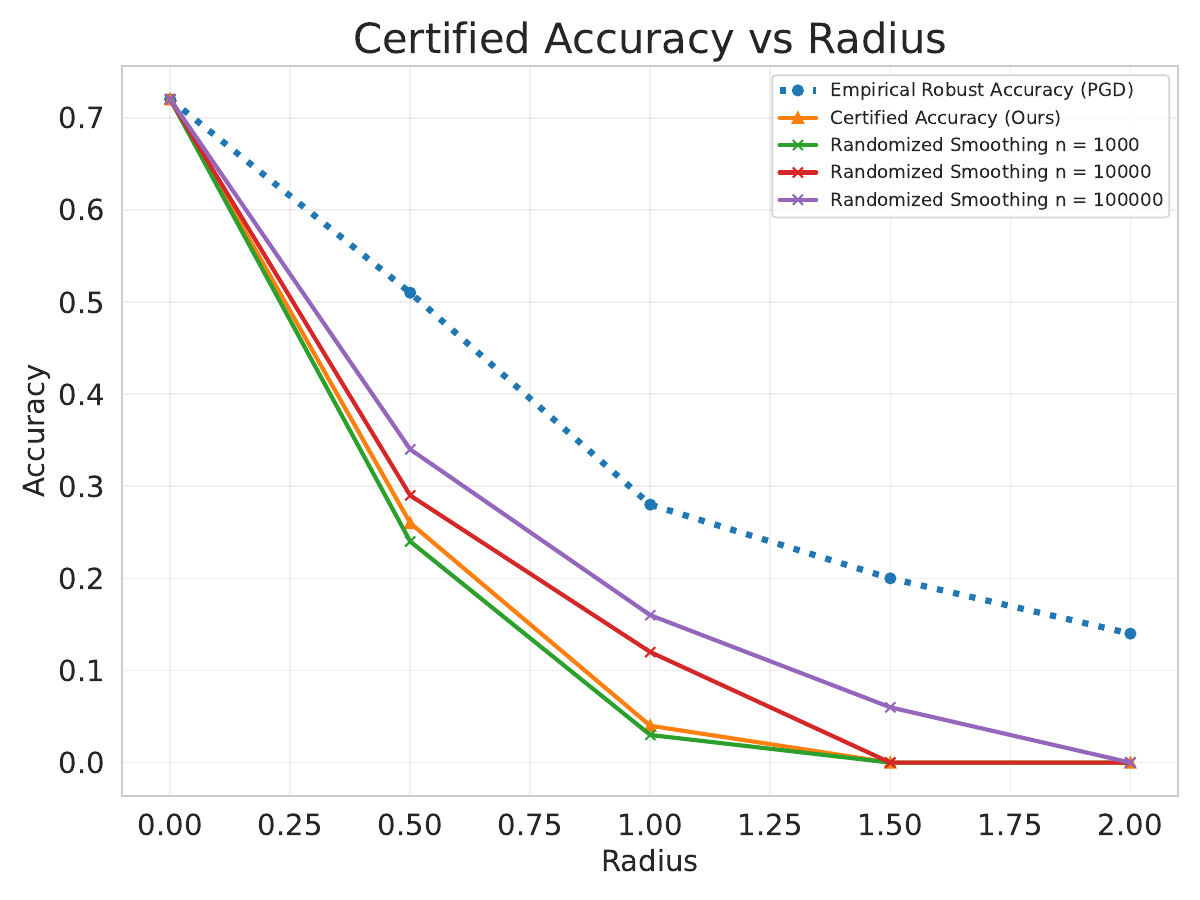}
    \includegraphics[width=0.3\linewidth]{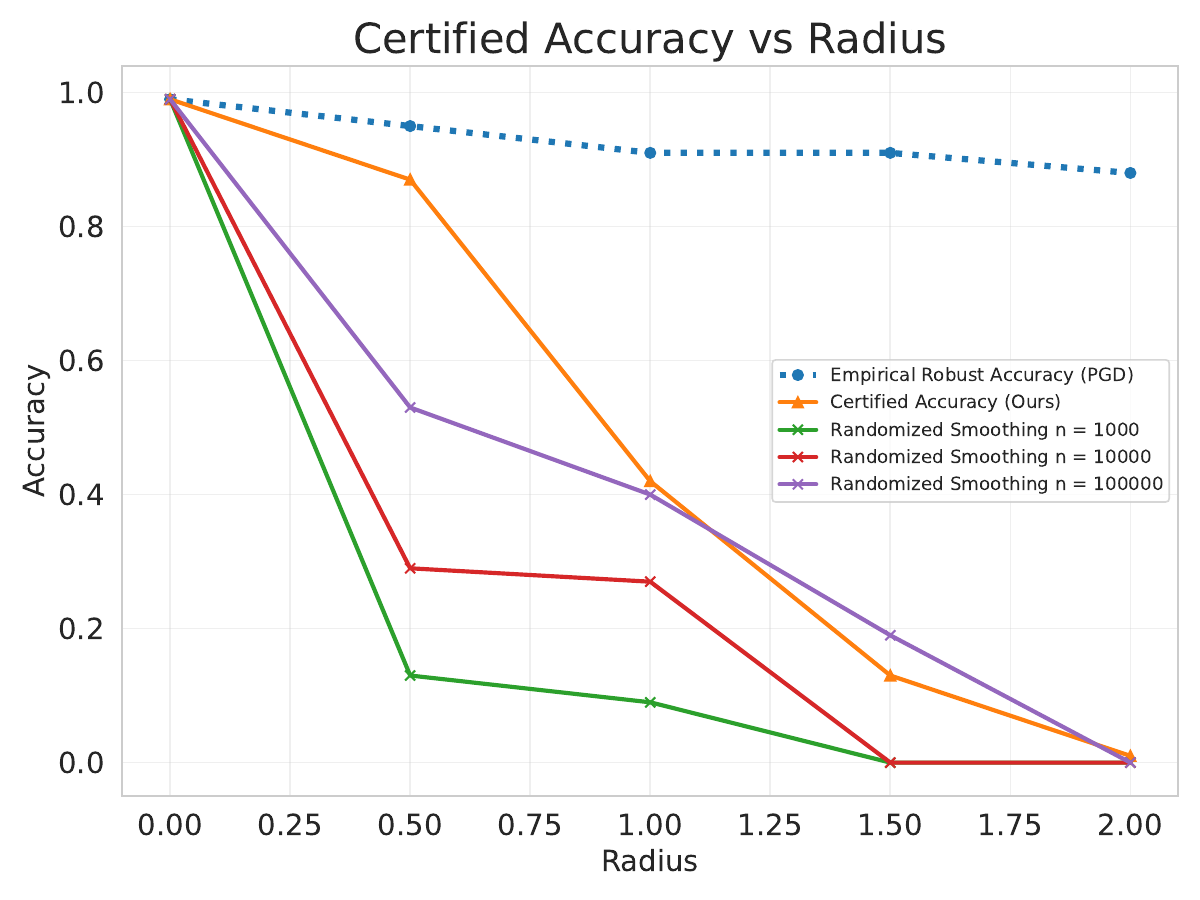}
    \includegraphics[width=0.3\linewidth]{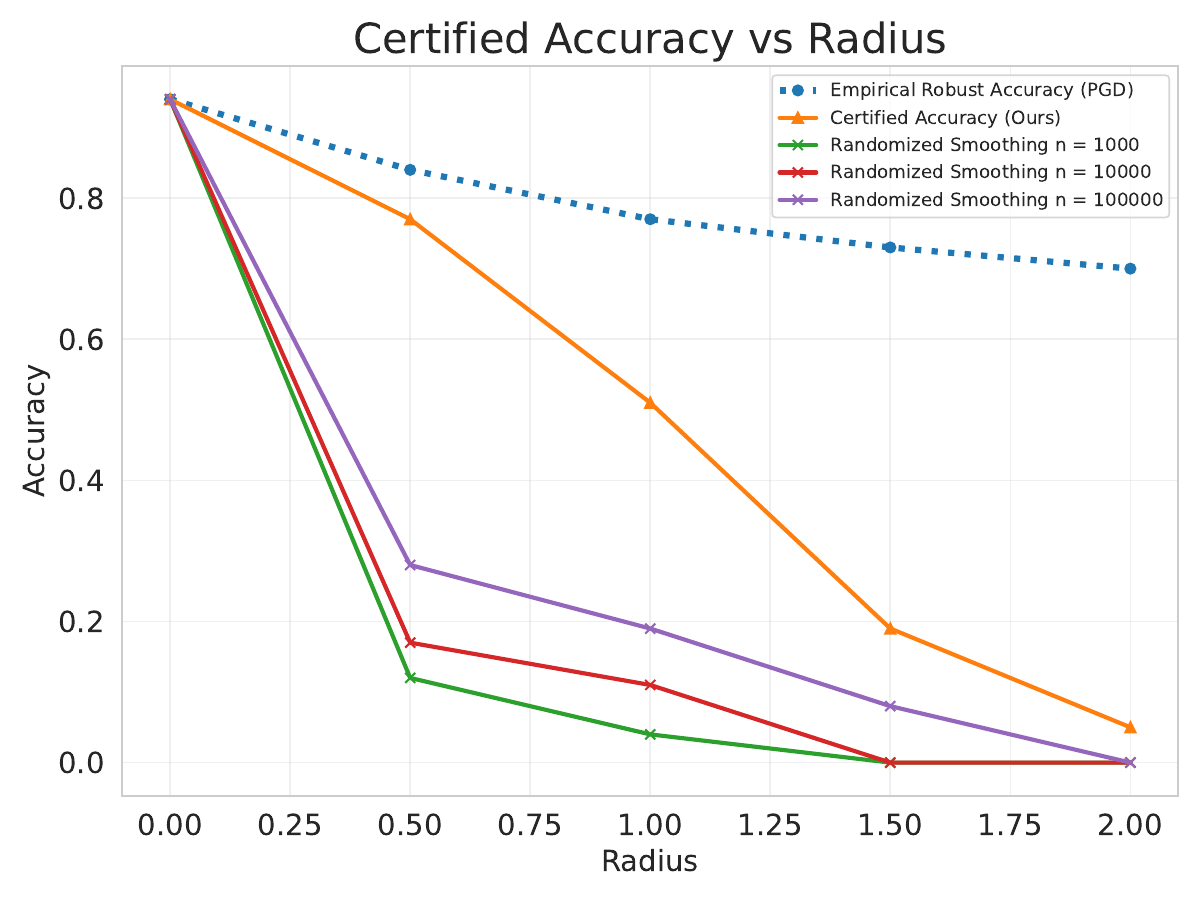}
    \caption{The proposed method achieves competitive robust accuracy in comparison to certified accuracy than randomized smoothing in different Gaussian mixture settings. Each row corresponds to a GMM with isotropic covariances and different number of classes $K = \{2, 3, 5\}$, while each column to one with different seperation distance $R = \{2, 4, 5\}$. }
    \label{fig: addit_rand_smoothing}
\end{figure}

\begin{figure}[ht]
    \centering
    \includegraphics[width=0.3\linewidth]{Figs/randomized_smoothing_new/acc_K2_R2_anisotropic_8.pdf}
    \includegraphics[width=0.3\linewidth]{Figs/randomized_smoothing_new/acc_K2_R4_anisotropic_3.pdf}
    \includegraphics[width=0.3\linewidth]{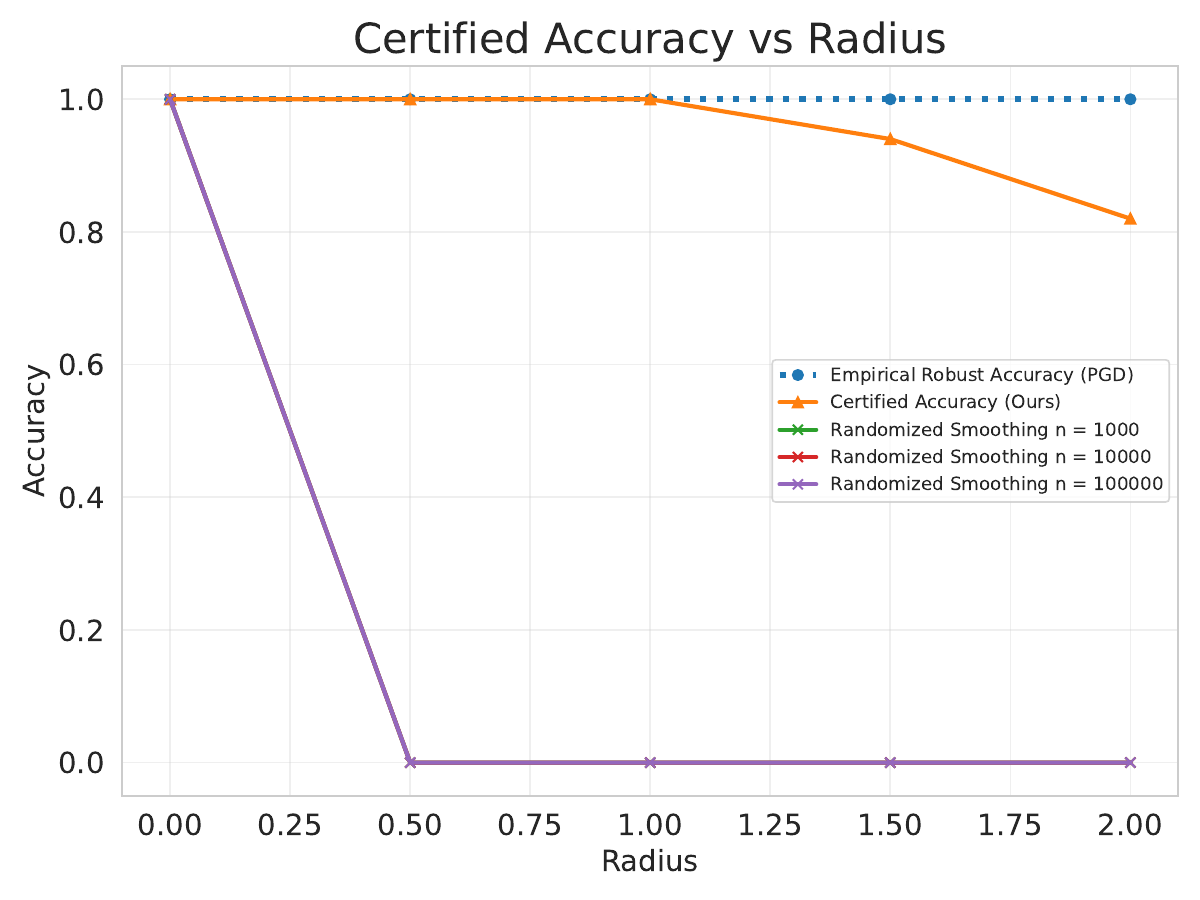}\\
    \includegraphics[width=0.3\linewidth]{Figs/randomized_smoothing_new/acc_K3_R2_anisotropic_8.pdf}
    \includegraphics[width=0.3\linewidth]{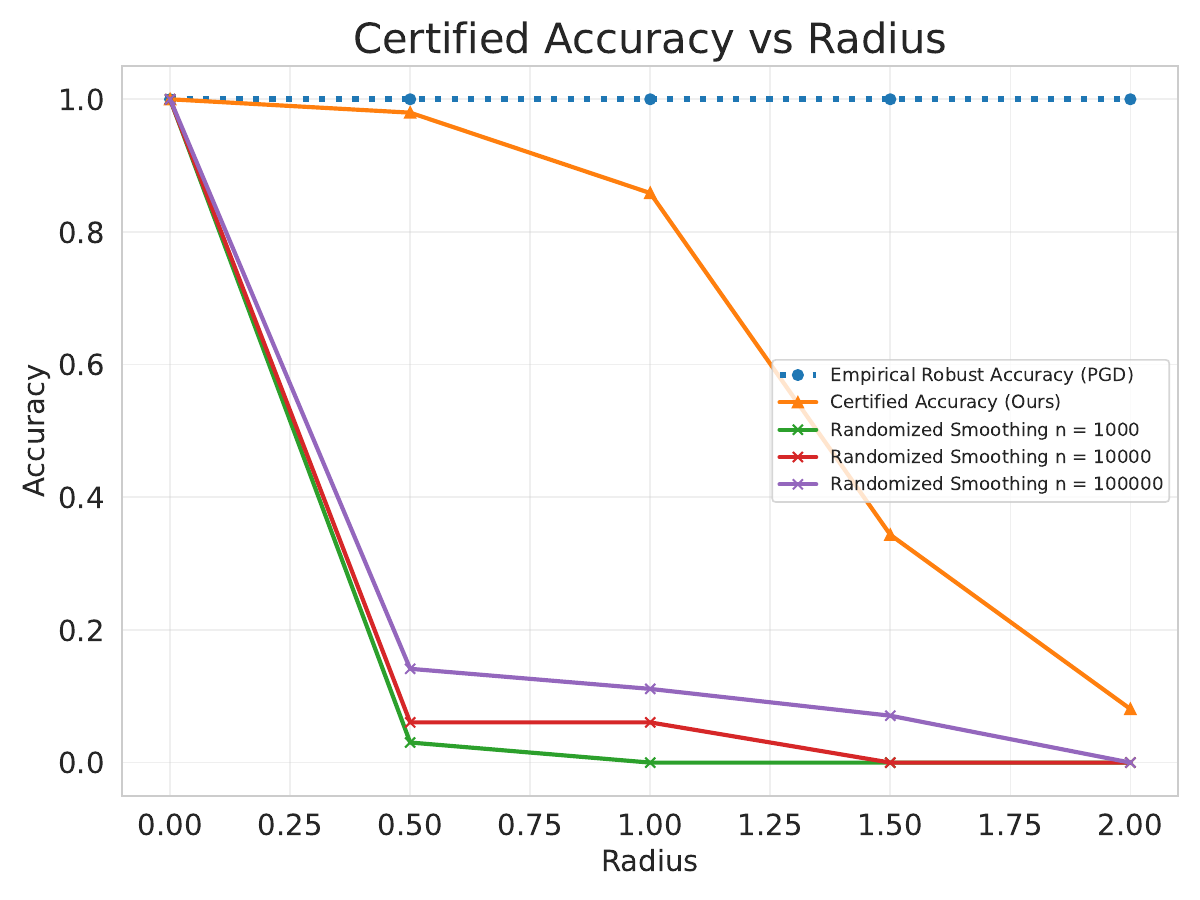}
    \includegraphics[width=0.3\linewidth]{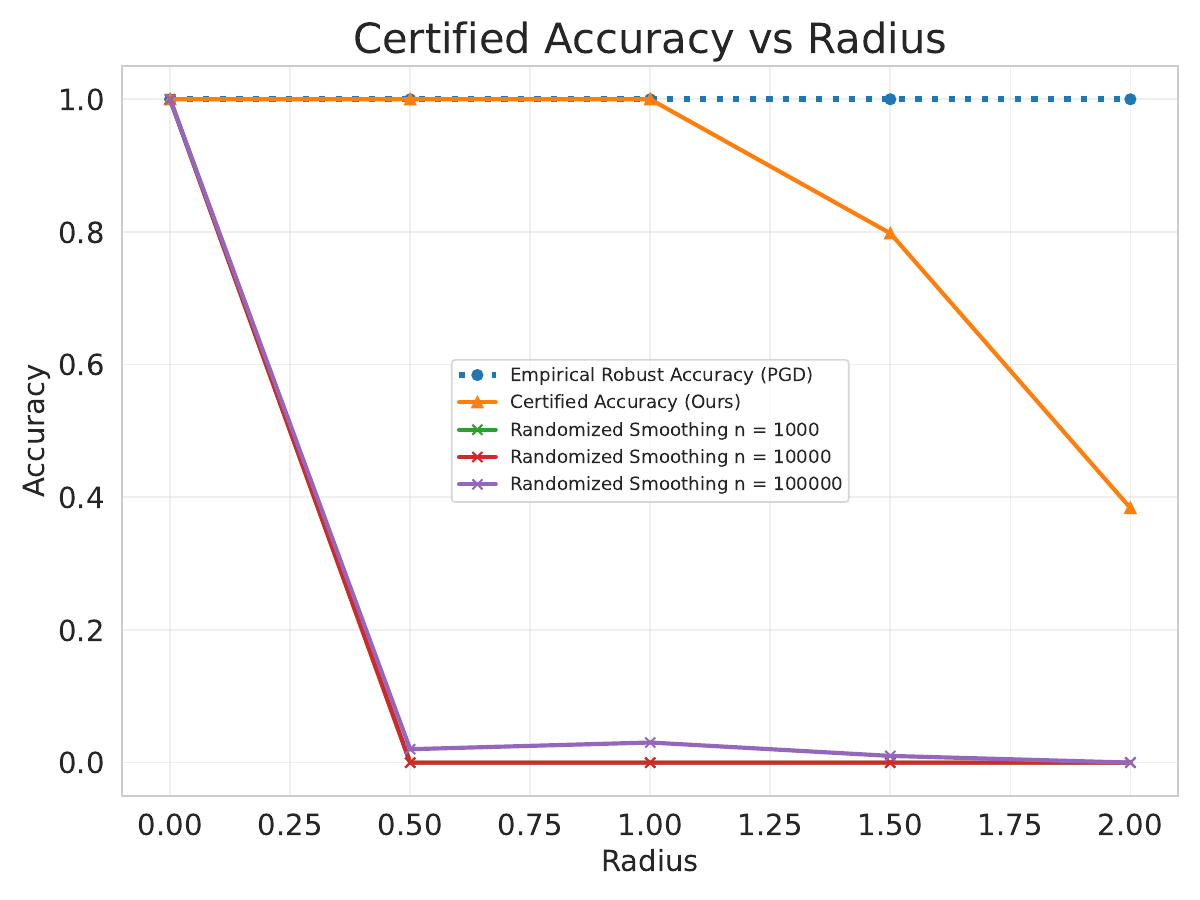}\\
    \includegraphics[width=0.3\linewidth]{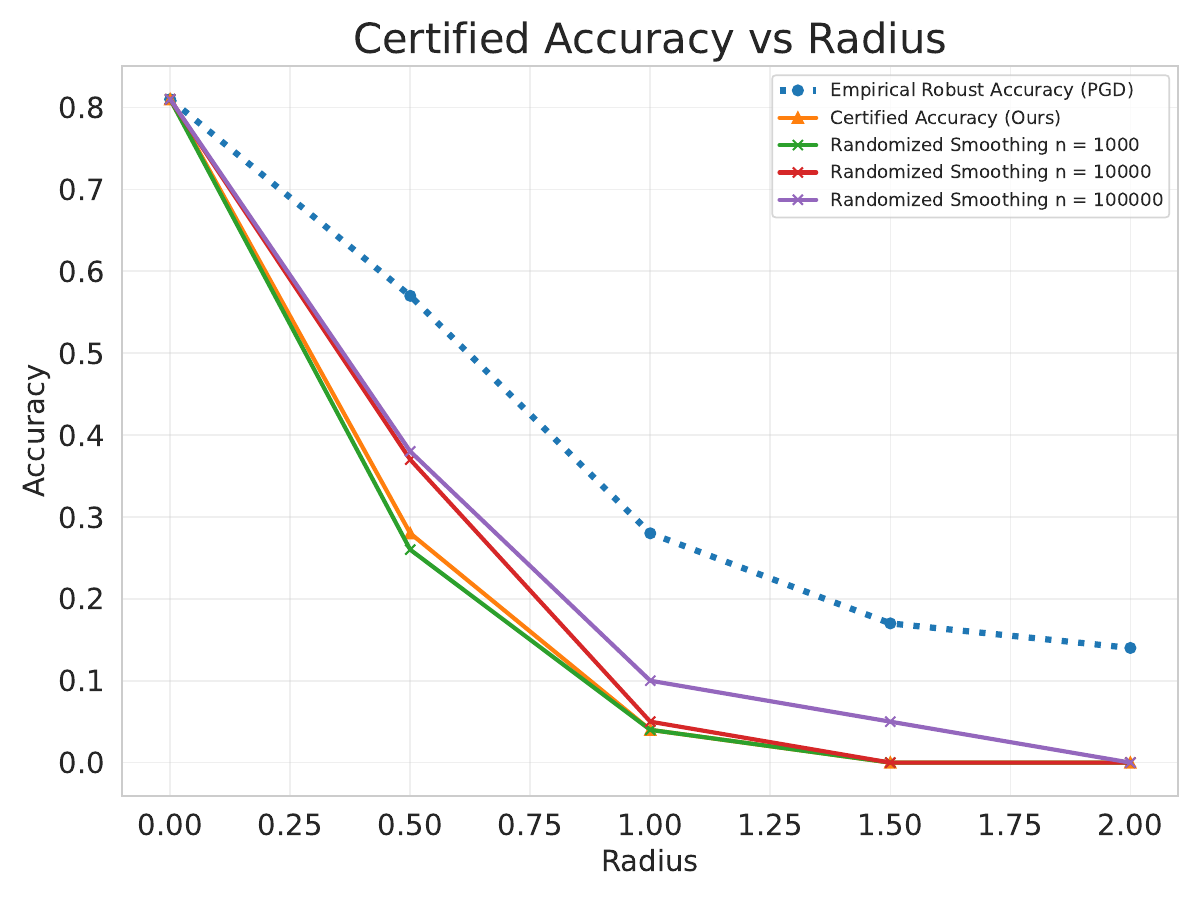}
    \includegraphics[width=0.3\linewidth]{Figs/randomized_smoothing_new/acc_K5_R4_anisotropic_2.pdf}
    \includegraphics[width=0.3\linewidth]{Figs/randomized_smoothing_new/acc_K2_R5_anisotropic_1.pdf}
    \caption{Comparison of our method with randomized smoothing for different mixture of Gaussians with \textit{anisotropic} covariances. The proposed method performs competitively against randomized smoothing even when less number of Monte Carlo samples are used.}
    \label{fig: addit_rand_smoothing_anisotropic}
\end{figure}

\end{document}